\documentclass{article}
\usepackage{iclr2022_conference,times}
\usepackage{natbib}
\usepackage[utf8]{inputenc}
\usepackage[ruled,linesnumbered]{algorithm2e}
\usepackage[noend]{algpseudocode}
\usepackage{amsmath,amsthm,amsfonts,amssymb}
\usepackage{hyperref}
\usepackage{color}
\usepackage{mathrsfs}
\usepackage{enumitem}
\usepackage{bm}
\usepackage{multirow}
\usepackage{booktabs}
\usepackage{makecell}
\usepackage{graphicx}
\usepackage{subfigure}
\usepackage{caption}
\usepackage{thm-restate}
\usepackage{cite}
\usepackage{xcolor}
\definecolor{light-gray}{gray}{0.9}
\usepackage{times}
\usepackage{rotating} 
\usepackage{diagbox}
\usepackage{mathabx}
\usepackage{mathtools}
\usepackage[normalem]{ulem}
\usepackage{url}
\usepackage{comment}

\renewcommand{\epsilon}{\varepsilon}

\newcommand{\trans}{^{\top}}

\newcommand{\cA}{\mathcal{A}}

\newcommand{\cS}{\mathcal{S}}

\newcommand{\cV}{\mathcal{V}}

\newcommand{\cE}{\mathcal{E}}
\newcommand{\cN}{\mathcal{N}}
\newcommand{\cM}{\mathcal{M}}

\newcommand{\EE}{\mathbb{E}}

\newcommand{\tR}{\tilde{R}}

\newcommand{\tSigma}{\tilde{\Sigma}}

\newcommand{\hatcov}{\Lambda}

\newcommand\numberthis{\addtocounter{equation}{1}\tag{\theequation}}

\let\hat\widehat
\let\tilde\widetilde

\newtheorem{theorem}{Theorem}[section]
\newtheorem{lemma}[theorem]{Lemma}
\newtheorem{corollary}[theorem]{Corollary}
\newtheorem{remark}[theorem]{Remark}

\newtheorem{property}[theorem]{Property}
\newtheorem{claim}{Claim}

\theoremstyle{definition}
\newtheorem{definition}[theorem]{Definition}

\newtheorem{assumption}{Assumption}

\newcount\Comments  
\definecolor{olivine}{rgb}{0.6, 0.73, 0.45}

\newcommand{\mat}[1]{\ensuremath{\mathbf{#1}}}
\newcommand{\A}{\mat{A}}
\newcommand{\B}{\mat{B}}
\newcommand{\X}{\mat{X}}
\newcommand{\BoldS}{\mat{S}}
\newcommand{\Y}{\mat{Y}}
\newcommand{\Z}{\mat{Z}}
\newcommand{\I}{\mat{I}}

\newcommand{\tdh}{{\tilde h}}
\newcommand{\h}{{h}}

\newcommand{\cQ}{{\mathcal{Q}}}

\newcommand{\red}[1]{{\color{red} #1}}

\newcommand{\bpi}{{\bar{\pi}}}
\newcommand{\bQ}{\bar{Q}}

\newcommand{\mR}{\mathbb{R}}

\newcommand{\cF}{\mathcal{F}}

\newcommand{\ra}{\rangle}
\newcommand{\la}{\langle}

\newcommand{\name}{{Deployment-Efficient }}

\newcommand{\resolution}{{\epsilon_0}}

\newcommand{\fix}[1]{\red{#1}}

\newcommand{\betap}{{\beta'}}

\iclrfinalcopy
\title{Towards Deployment-Efficient Reinforcement Learning: Lower Bound and Optimality}
\author{Jiawei Huang\thanks{Work done during the internship at Microsoft Research Asia.} $~\dagger$, Jinglin Chen$\dagger$, Li Zhao$\ddagger$, Tao Qin$\ddagger$, Nan Jiang$\dagger$, Tie-Yan Liu$\ddagger$\\
$\dagger$ Department of Computer Science, University of Illinois at Urbana-Champaign \\
\texttt{~\{jiaweih, jinglinc, nanjiang\}@illinois.edu}\\
$\ddagger$ Microsoft Research Asia \\
\texttt{~\{lizo, taoqin, tyliu\}@microsoft.com}
}
\date{\today}

\begin{document}

\maketitle

\begin{abstract}
Deployment efficiency is an important criterion for many real-world applications of reinforcement learning (RL). Despite the community's increasing interest, there lacks a formal theoretical formulation for the problem. In this paper, we propose such a formulation for deployment-efficient RL (DE-RL) from an ``optimization with constraints'' perspective: we are interested in exploring an MDP and obtaining a near-optimal policy within minimal \emph{deployment complexity}, whereas in each deployment the policy can sample a large batch of data. Using finite-horizon linear MDPs as a concrete structural model, we reveal the fundamental limit in achieving deployment efficiency by establishing information-theoretic lower bounds, and provide algorithms that achieve the optimal deployment efficiency. Moreover, our formulation for DE-RL is flexible and can serve as a building block for other practically relevant settings; we give ``Safe DE-RL'' and ``Sample-Efficient DE-RL'' as two examples, which may be worth future investigation.
\end{abstract}

\allowdisplaybreaks
\section{Introduction}

In many real-world applications, deploying a new policy to replace the previous one is costly, while generating a large batch of samples with an already deployed policy can be relatively fast and cheap.
For example, in recommendation systems \citep{afsar2021reinforcement}, education software \citep{bennane2013adaptive}, and healthcare  \citep{yu2019reinforcement}, the new recommendation, teaching, or medical treatment strategy must pass several internal tests to ensure safety and practicality before being deployed, which can be time-consuming. On the other hand, the algorithm may be able to collect a large amount of samples in a short period of time if the system serves a large population of users. 
Besides, in robotics applications \citep{kober2013reinforcement}, deploying a new policy usually involves operations on the hardware level, which requires non-negligible physical labor and long waiting periods, while sampling trajectories is relatively less laborious.
However, deployment efficiency was neglected in most of existing RL literatures. Even for those few works considering this important criterion \citep{bai2020provably, gao2021provably, matsushima2021deploymentefficient}, either their settings or methods have limitations in the scenarios described above, or a formal mathematical formulation is missing. We defer a detailed discussion of these related works to Section \ref{para:discussion_with_prev_DERL}.

In order to close the gap between existing RL settings and real-world applications requiring high deployment efficiency, our first contribution is to provide a formal definition and tractable objective for Deployment-Efficient Reinforcement Learning (DE-RL) via an ``optimization with constraints'' perspective. 
Roughly speaking, 
we are interested in minimizing the number of deployments $K$ under two constraints: (a) after deploying $K$ times, the algorithm can return a near-optimal policy, and (b) the number of \emph{trajectories} collected in each deployment, denoted as $N$, is at the same level across $K$ deployments, and it can be large but should still be polynomial in standard parameters. Similar to the notion of sample complexity in online RL, we will refer to $K$ as \emph{deployment complexity}.

To provide a more quantitative understanding, we instantiate our DE-RL framework in  finite-horizon linear MDPs\footnote{Although we focus on linear MDPs, the core idea can be extended to more general settings such as RL with general function approximation~\citep{Kong2021OnlineSF}.} \citep{jin2019provably} and develop the essential theory. The main questions we address are: 
\begin{center}
\emph{Q1: What is the optimum of the deployment efficiency in our DE-RL setting?}\\
\emph{Q2: Can we achieve the optimal deployment efficiency in our DE-RL setting?}
\end{center}
When answering these questions, 
we separately study algorithms with or without being constrained to deploy deterministic policies each time. 
While deploying more general forms of policies can be practical (e.g., randomized experiments on a population of users can be viewed as deploying a mixture of deterministic policies), 
most previous theoretical works in related settings exclusively focused on upper and lower bounds for algorithms using deterministic policies \citep{jin2019provably, wang2020rewardfree, gao2021provably}. 
As we will show, the origin of the difficulty in optimizing deployment efficiency and the principle in algorithm design to achieve optimal deployment efficiency are generally different in these two settings, and therefore, we believe both of them are of independent interests. 

As our second contribution, in Section \ref{sec:lower_bound}, we answer Q1 by providing information-theoretic lower bounds for the required number of deployments under the constraints of (a) and (b) in Def \ref{def:deployment_efficient_linear_mdp}. We establish $\Omega(dH)$ and $\tilde{\Omega}(H)$ lower bounds for algorithms with and without the constraints of deploying deterministic policies, respectively. 
Contrary to the impression given by previous empirical works \citep{matsushima2021deploymentefficient}, even if we can deploy unrestricted policies, the minimal number of deployments cannot be reduced to a constant 
without additional assumptions, which sheds light on the fundamental limitation in achieving deployment efficiency. 
Besides, in the line of work on ``horizon-free RL'' \citep[e.g.,][]{wang2020long}, it is shown that  RL problem is not significantly harder than bandits  (i.e., when $H=1$) when we consider sample complexity. 
In contrast, the $H$ dependence in our lower bound reveals some fundamental hardness that is specific to long-horizon RL, particularly in the deployment-efficient setting.
\footnote{{Although \citep{wang2020long} considered stationary MDP, as shown in our Corollary \ref{corollary:lower_bounds_stationary_MDP}, the lower bounds of deployment complexity is still related to $H$.}}
Such hardness results were originally conjectured by \citet{jiang2018open}, but no hardness has been shown in sample-complexity settings. 

After identifying the limitation of deployment efficiency, as our third contribution, we address Q2 by proposing novel algorithms whose deployment efficiency match the lower bounds. In Section \ref{sec:DERL_deterministic_policy}, we propose an algorithm deploying deterministic policies, which is based on Least-Square Value Iteration with reward bonus \citep{jin2019provably} and a layer-by-layer exploration strategy, and can return an $\epsilon$-optimal policy within $O(dH)$ deployments. As part of its analysis, we prove Lemma \ref{lem:finite_sample_elliptical_potential_lemma} as a technical contribution, which can be regarded as a batched finite-sample version of the well-known ``Elliptical Potential Lemma''\citep{alex2020elliptical} and may be of independent interest. Moreover, our analysis based on Lemma \ref{lem:finite_sample_elliptical_potential_lemma} can be applied to the reward-free setting \citep{jin2020rewardfree,wang2020rewardfree} and achieve the same optimal deployment efficiency. 
In Section \ref{sec:DERL_arbitrary_policy}, we focus on algorithms which can deploy arbitrary policies. They are much more challenging because it requires us to find a provably exploratory stochastic policy 
without interacting with the environment.
To our knowledge,  \citet{agarwal2020flambe} is the only work  tackling a similar problem, but their algorithm is model-based which relies on a strong assumption about the realizability of the true dynamics and a sampling oracle that allows the agent to sample data from the model, and how to solve the problem in linear MDPs without a model class is still an open problem. 
To overcome this challenge, we propose a model-free layer-by-layer exploration algorithm based on a novel covariance matrix estimation technique, and prove that it requires $\Theta(H)$ deployments to return an $\epsilon$-optimal policy, which only differs  from the lower bound $\tilde{\Omega}(H)$ by a logarithmic factor. 
Although the per-deployment sample complexity of our algorithm has dependence on a ``reachability coefficient'' (see Def. \ref{def:reachability_coefficient}), similar quantities also appear in related works \citep{zanette2020provably, agarwal2020flambe, modi2021model} and we conjecture that it is unavoidable and leave the investigation to future work.

Finally, thanks to the flexibility of our ``optimization with constraints'' perspective, 
our DE-RL setting can serve as a building block for more advanced and practically relevant settings where optimizing the number of deployments is an important consideration. 
In Appendix \ref{appx:discussion_on_Extended_DERL}, we propose two potentially interesting settings: ``Safe DE-RL'' and ``Sample-Efficient DE-RL'', by introducing constraints regarding safety and sample efficiency, respectively.

\subsection{Closely Related Works}\label{sec:related_work}
We defer the detailed discussion of previous literatures about pure online RL and pure offline RL to Appendix \ref{sec:extended_related_work}, and mainly focus on those literatures which considered deployment efficiency and more related to us in this section. 
\label{para:discussion_with_prev_DERL}

To our knowledge, the term ``deployment efficiency'' was first coined by \citet{matsushima2021deploymentefficient}, but they did not provide a concrete mathematical formulation that is amendable to theoretical investigation. In existing theoretical works, low switching cost is a concept closely related to deployment efficiency, and has been studied in both bandit \citep{esfandiari2020regret, han2020sequential,gu2021batched, ruan2021linear} and RL settings \citep{bai2020provably, gao2021provably, Kong2021OnlineSF}. Another related concept is concurrent RL, as proposed by \citet{guo2015concurrent}. 
We highlight the difference with them in two-folds from problem setting and techniques.

As for the problem setting, existing literature on low switching cost mainly focuses on sub-linear regret guarantees, which does not directly implies a near-optimal policy after a number of policy deployments\footnote{Although the conversion from sub-linear regret to polynominal sample complexity is possible (``online-to-batch''), we show in Appendix \ref{sec:extended_related_work} that to achieve accuracy $\epsilon$ 
after conversion, the number of deployments of previous low-switching cost algorithms has dependence on $\epsilon$, whereas our guarantee does not. 
}. Besides, low switching-cost RL algorithms \citep{bai2020provably, gao2021provably, Kong2021OnlineSF} 
rely on adaptive switching strategies (i.e., the interval between policy switching is not fixed), which can be difficult to implement in practical scenarios. 
For example, in recommendation or education systems, once deployed, a policy usually needs to interact with the population of users for a fair amount of time and generate a lot of data.  
Moreover, since policy preparation is time-consuming (which is what motivates our work to begin with), it is practically difficult if not impossible to change the policy immediately once collecting enough data for policy update, and it will be a significant overhead compared to a short policy switch interval. Therefore, in applications we target at, it is more reasonable to assume that the sample size in each deployment (i.e., between policy switching) has the same order of magnitude and is large enough so that the overhead of policy preparation can be ignored.

More importantly, on the technical side, previous theoretical works on low switching cost mostly use deterministic policies in each deployment, which is easier to analyze. This issue also applies to the work of \citet{guo2015concurrent} on concurrent PAC RL. However, if the agent can deploy stochastic (and possibly non-Markov) policies (e.g., a mixture of deterministic policies), then intuitively---and as reflected in our lower bounds---exploration can be done much more deployment-efficiently, and we provide a stochastic policy algorithm that achieves an $\tilde O(H)$ deployment complexity 
and overcomes the $\Omega(dH)$ lower bounds for deterministic policy algorithms \citep{gao2021provably}.


\section{Preliminaries}\label{sec:preliminary}
\paragraph{Notation}
Throughout our paper, for $n\in\mathbb{Z}^+$, we will denote $[n]=\{1,2,...,n\}$. $\lceil\cdot\rceil$ denotes the ceiling function. Unless otherwise specified, for vector $x\in\mR^d$ and matrix $X\in\mR^{d\times d}$, $\|x\|$ denotes the vector $l_2$-norm of $x$ and $\|X\|$ denotes the largest singular value of $X$. We will use standard big-oh notations $O(\cdot), \Omega(\cdot), \Theta(\cdot)$, and notations such as $\tilde{O}(\cdot)$ to suppress logarithmic factors.

\subsection{Episodic Reinforcement Learning}
We consider an episodic Markov Decision Process denoted by $M(\cS, \cA, H, P, r)$, where $\cS$ is the state space, $\cA$ is the finite action space, $H$ is the horizon length, and $P=\{P_h\}_{h=1}^H$ and $r=\{r_h\}_{h=1}^H$ denote the transition and the reward functions. At the beginning of each episode, the environment will sample an initial state $s_1$ from the initial state distribution $d_1$. Then, for each time step $h\in[H]$, the agent selects an action $a_h \in \cA$, interacts with the environment, receives a reward $r_h(s_h,a_h)$, and transitions to the next state $s_{h+1}$. The episode will terminate once $s_{H+1}$ is reached.

A (Markov) policy $\pi_h(\cdot)$ at step $h$ is a function mapping from $\cS\to \Delta(\cA)$, where $\Delta(\cA)$ denotes the probability simplex over the action space. With a slight abuse of notation, when $\pi_h(\cdot)$ is a deterministic policy, we will assume $\pi_h(\cdot):\cS\to \cA$. A full (Markov) policy $\pi=\{\pi_1,\pi_2,...,\pi_H\}$ specifies such a mapping for each time step. We use $V^\pi_h(s)$ and $Q^\pi_h(s,a)$ to denote the value function and Q-function at step $h\in[H]$, which are defined as:
\begin{align*}
    V^\pi_h(s)=\EE[\sum_{h'=h}^H r_{h'}(s_{h'},a_{h'})|s_h=s,\pi],\quad Q^\pi_h(s,a)=\EE[\sum_{h'=h}^H r_{h'}(s_{h'},a_{h'})|s_h=s,a_h=a,\pi]
\end{align*}
We also use $V^*_h(\cdot)$ and $Q^*_h(\cdot,\cdot)$ to denote the optimal value functions and use $\pi^*$ to denote the optimal policy that maximizes the expected return  $J(\pi) := \EE[\sum_{h=1}^H r(s_h,a_h)|\pi]$. 
In some occasions, we use $V^\pi_h(s;r)$ and $Q^\pi_h(s,a;r)$ to denote the value functions with respect to $r$ as the reward function for disambiguation purposes. The optimal value functions and the optimal policy will be denoted by $V^*(s;r),Q^*(s,a;r),\pi_r^*$, respectively.

\paragraph{Non-Markov Policies} 
While we focus on Markov policies in the above definition, some of our results apply to or require more general forms of policies. For example, our lower bounds apply to non-Markov policies that can depend on the history (e.g., $\cS_1\times\cA_1\times\mR...\times\cS_{h-1}\times\cA_{h-1}\times\mR\times\cS_h\rightarrow \cA$ for deterministic policies); our algorithm for arbitrary policies deploys a mixture of deterministic Markov policies, which corresponds to choosing a deterministic policy from a given set at the initial state, and following that policy for the entire trajectory. This can be viewed as a non-Markov stochastic policy.

\subsection{Linear MDP Setting}
We mainly focus on the linear MDP \citep{jin2019provably} satisfying the following assumptions:
\begin{assumption}[Linear MDP Assumptions]\label{assump:linear_MDP}
    An MDP $\cM=(\cS,\cA,H,P,r)$ is said to be a linear MDP with a feature map $\phi:\cS\times\cA\rightarrow \mR^d$ if the following hold for any $h\in[H]$: \vspace{-.5em}
    \begin{itemize}[leftmargin=*, itemsep=0pt]
        \item There are $d$ unknown signed measures $\mu_h=(\mu_h^{(1)},\mu_h^{(2)},...,\mu_h^{(d)})$ over $\cS$ such that for any $(s,a,s')\in\cS\times\cA\times\cS, P_h(s'|s,a)=\la\mu_h(s'),\phi(s,a)\ra$.
        \item There exists an unknown vector $\theta_h\in\mR^d$ such that for any $(s,a)\in\cS\times\cA$, $r_h(s,a)=\la\phi(s,a),\theta_h\ra$.
    \end{itemize}
\end{assumption}
Similar to \citet{jin2019provably} and \citet{wang2020rewardfree}, without loss of generality, we assume for all $(s,a)\in\cS\times\cA$ and $h\in[H]$, $\|\phi(s,a)\|\leq 1$, $\|\mu_h\|\leq\sqrt{d}$, and $\|\theta_h\|\leq \sqrt{d}$. 
In Section~\ref{sec:lower_bound} we will refer to linear MDPs with stationary dynamics, which is a special case when $\mu_1=\mu_2=\ldots \mu_H$ and $\theta_1 = \theta_2 = \ldots = \theta_H$.

\subsection{A Concrete Definition of DE-RL}\label{sec:definition_difference}
In the following, we introduce our formulation for DE-RL in linear MDPs. 
For discussions of comparison to existing works, please refer to  Section \ref{sec:related_work}.
\begin{definition}[Deployment Complexity in Linear MDPs]\label{def:deployment_efficient_linear_mdp}
We say that an algorithm has a deployment complexity $K$ in linear MDPs if the following holds: given an arbitrary linear MDP under Assumption \ref{assump:linear_MDP}, for arbitrary $\epsilon$ and $0<\delta<1$, the algorithm will return a policy $\pi_K$ after $K$  deployments and collecting at most $N$ trajectories in each deployment, under the following constraints:
\begin{enumerate}[leftmargin=*] 
    \item[(a)] With probability $1-\delta$, $\pi_K$ is $\epsilon$-optimal, i.e. $J(\pi_K) \geq \max_\pi J(\pi) - \epsilon$.
    \item[(b)] The sample size $N$ is polynominal, i.e. $N = {\rm poly}(d, H, \frac{1}{\epsilon}, \log \frac{1}{\delta})$. Moreover, $N$ should be fixed a priori and cannot change adaptively from deployment to deployment.  
\end{enumerate} 
\end{definition} 
Under this definition, the goal of \name RL is to design algorithms with provable guarantees of low deployment complexity. \vspace{-.5em}

\paragraph{Polynomial Size of $N$} We emphasize that the restriction of polynomially large $N$ is crucial to our formulation, and not including it can result in degenerate solutions. For example, if $N$ is allowed to be exponentially large, we can finish exploration in $1$ deployment in the arbitrary policy setting, by deploying a mixture of exponentially many policies that form an $\epsilon$-net of the policy space. Alternatively, 
we can sample actions uniformly, and use importance sampling \citep{precup2000eligibility} to evaluate all of them in an off-policy manner. None of these solutions are practically feasible and are excluded by our restriction on $N$.

\section{Lower Bound for Deployment Complexity in RL}\label{sec:lower_bound}

In this section, we provide information-theoretic lower bounds of the deployment complexity in our DE-RL setting. We defer the lower bound construction and the proofs to Appendix \ref{appx:lower_bound}. As mentioned in Section \ref{sec:preliminary}, we consider non-Markov policies when we refer to deterministic and stochastic policies in this section, which strengthens our lower bounds as they apply to very general forms of policies.

We first study the algorithms which can only deploy deterministic policy at each deployment. 
\begin{restatable}{theorem}{ThmLBDP}[Lower bound for deterministic policies, informal]\label{thm:lower_bound_deterministic_policy}
    For any $d\geq 4, H$ and any algorithm $\psi$ that can only deploy a deterministic policy at each deployment, there exists a linear MDP $M$ satisfying Assumption \ref{assump:linear_MDP}, such that the deployment complexity of $\psi$ in $M$ is $K=\Omega(dH)$.
\end{restatable}
The basic idea of our construction and the proof is that, intuitively, a linear MDP with dimension $d$ and horizon length $H$ has $\Omega(dH)$ ``independent directions'', while deterministic policies have limited exploration capacity and only reach $\Theta(1)$ direction in each deployment, which result in $\Omega(dH)$ deployments in the worst case.

In the next theorem, we will show that, even if the algorithm can use arbitrary exploration strategy (e.g. maximizing entropy, adding reward bonus), without additional assumptions, the number of deployments $K$ still has to depend on $H$ and may not be reduced to a constant when $H$ is large.
\begin{restatable}{theorem}{ThmLBAPP}[Lower bound for arbitrary policies, informal]\label{thm:lower_bound_all_possible_policy}
    For any $d\ge 4, H, N$ and any algorithm $\psi$ which can deploy arbitrary policies, 
    there exists a linear MDP $M$ satisfying Assumption \ref{assump:linear_MDP}, such that
    the deployment complexity of $\psi$ in $M$ is 
    $K = \Omega(H/\lceil\log_d (NH)\rceil)=\tilde{\Omega}(H)$.
\end{restatable} 
The origin of the difficulty can be illustrated by a recursive dilemma: in the worst case, if the agent does not have enough information at layer $h$, then it cannot identify a good policy to explore till layer 
$h+\Omega(\log_d(NH))$ in 1 deployment, 
and so on and so forth. Given that we enforce $N$ to be polynomial, the agent can only push the ``information boundary'' forward by $\Omega(\log_d (NH))=\tilde\Omega(1)$ layers per deployment.
In many real-world applications, such difficulty can indeed exist. For example, in healthcare, the entire treatment is often divided into multiple stages. If the treatment in stage $h$ is not effective, the patient may refuse to continue. This can result in insufficient samples for identifying a policy that performs well in stage $h+1$.

\paragraph{Stationary vs.~non-stationary dynamics}
Since we consider non-stationary dynamics in Assump.~\ref{assump:linear_MDP}, one may  suspect that the $H$-dependence in the lower bound is mainly due to such non-stationarity. 
We show that this is not quite the case, and the $H$-dependence still exists for stationary dynamics. In fact, our lower bound for non-stationary dynamics directly imply one for stationary dynamics: given a finite horizon non-stationary MDP $\tilde M=(\tilde\cS, \cA, H, \tilde P, \tilde r)$, we can construct a stationary MDP $M = (\cS, \cA, H, P, r)$ by expanding the state space to $\cS = \tilde\cS\times[H]$ so that the new transition function $P$ and reward function $r$ are stationary across time steps. 
As a result, given arbitrary $d\geq 4$ and $H\geq 2$, we can construct a hard non-stationary MDP instance $\tilde M$ with dimension $\tilde d = \max\{4, d/H\}$ and horizon $\tilde h = d / \tilde d = \min\{H, d/4\}$, and convert it to a stationary MDP $M$ with dimension $d$ and horizon $h=\tilde h = \min\{H, d/4\} \leq H$. If there exists an algorithm which can solve $M$ in $K$ deployments, then it can be used to solve $\tilde M$ in no more than $K$ deployments. Therefore, the lower bounds for stationary MDPs can be extended from Theorems~\ref{thm:lower_bound_deterministic_policy} and \ref{thm:lower_bound_all_possible_policy}, as shown in the following corollary:
\begin{corollary}[Extension to Stationary MDPs]\label{corollary:lower_bounds_stationary_MDP}
    For stationary linear MDP with $d\geq 4$ and $H\geq 2$, suppose $N=\mathrm{poly}(d, H, \frac{1}{\epsilon}, \log\frac{1}{\delta})$, the lower bound of deployment complexity would be $\Omega(d)$ for deterministic policy algorithms, and $\Omega(\frac{\min\{d/4, H\}}{\lceil\log_{\max\{d/H, 4\}}NH\rceil})=\tilde\Omega(\min\{d, H\})$ for algorithms which can deploy arbitrary policies.
\end{corollary}
As we can see, the dependence on dimension and horizon will not be eliminated even if we make a stronger assumption that the MDP is stationary. 
The intuition is that, although the transition function is stationary, some states may not be reachable from the initial state distribution within a small number of times, so the stationary MDP can effectively have a ``layered'' structure. 
For example, in Atari games \citep{bellemare2013arcade} (where many algorithms like DQN \citep{mnih2013playing} model the environments as infinite-horizon discounted MDPs) such as Breakout, the agent cannot observe states where most of the bricks are knocked out at the initial stage of the trajectory. Therefore, the agent still can only push forward the ``information frontier'' a few steps per deployment. 
That said, it is possible reduce the deployment complexity lower bound in stationary MDPs by adding more assumptions, such as the initial state distribution providing good coverage over the entire state space, or all the states are reachable in the first few time steps. However, because these assumptions do not always hold and may overly trivialize the exploration problem, we will not consider them in our algorithm design. 
Besides, although our algorithms in the next section are designed for non-stationary MDPs, they can be extended to stationary MDPs by sharing covariance matrices, and we believe the analyses can also be extended to match the lower bound in Corollary \ref{corollary:lower_bounds_stationary_MDP}.
\section{Towards Optimal Deployment Efficiency}\label{sec:towards_optimal_DERL}
In this section we provide algorithms with deployment-efficiency guarantees that nearly match the lower bounds established in Section~\ref{sec:lower_bound}. 
Although our lower bound results in Section \ref{sec:lower_bound} consider non-Markov policies, our algorithms in this section only use Markov policies (or a mixture of Markov policies, in the arbitrary policy setting), which are simpler to implement and compute and are already near-optimal in deployment efficiency. 

\noindent\textbf{Inspiration from Lower Bounds: a Layer-by-Layer Exploration Strategy}~
The linear dependence on $H$ in the lower bounds implies a possibly deployment-efficient manner to explore, which we call  a layer-by-layer strategy: 
conditioning on sufficient exploration in previous $h-1$ time steps, we can use $\mathrm{poly}(d)$ deployments to sufficiently explore the $h$-th time step, then we only need $H\cdot\mathrm{poly}(d)$ deployments to explore the entire MDP. 
If we can reduce the deployment cost in each layer from $\mathrm{poly}(d)$ to $\Theta(d)$ or even $\Theta(1)$, then we can achieve the optimal deployment efficiency.
Besides, as another motivation, in Appendix \ref{appx:additional_safety_guarantee}, we will briefly discuss the additional benefits of the layer-by-layer strategy, which will be useful especially in ``Safe DE-RL''.
In Sections \ref{sec:DERL_deterministic_policy} and  \ref{sec:DERL_arbitrary_policy}, we will introduce algorithms based on this idea and provide theoretical guarantees.

\subsection{Deployment-Efficient RL with Deterministic Policies}\label{sec:DERL_deterministic_policy}

\begin{algorithm}[ht]
    \textbf{Input}: Failure probability $\delta>0$, and target accuracy $\epsilon > 0$, $\beta\gets c_\beta \cdot dH\sqrt{\log(dH\delta^{-1}\epsilon^{-1})}$ for some $c_\beta>0$, total number of deployments $K$, batch size $N$, \\
    $h_1\gets 1\quad\quad$    // $h_k$ denotes the layer to explore in iteration $k$, for all $k\in[K]$\\
    \For{$k=1,2,...,K$}{
        $Q^k_{h_k+1}(\cdot,\cdot)\gets 0$ and $V^k_{h_k+1}(\cdot)=0$\\
        \For{$h=h_k,h_k-1,...,1$}{
            $\Lambda^k_h \gets I+\sum_{\tau=1}^{k-1}\sum_{n=1}^N \phi_h^{\tau n}(\phi_h^{\tau n})\trans,\quad\quad u_h^k(\cdot,\cdot)\gets \min\{\beta\cdot \sqrt{\phi(\cdot,\cdot)\trans(\Lambda^k_h)^{-1}\phi(\cdot,\cdot)}, H\}$\\
            $w^k_h\gets (\Lambda^k_h)^{-1}\sum_{\tau=1}^{k-1}\sum_{n=1}^N\phi_h^{\tau n}\cdot V^k_{h+1}(s^{\tau n}_{h+1})$\\
            $Q^k_h(\cdot,\cdot)\gets \min\{(w^k_h)\trans\phi(\cdot,\cdot)+r_h(\cdot,\cdot)+u^k_h(\cdot,\cdot), H\}$ and $V^k_h(\cdot)=\max_{a\in\cA}Q^k_h(\cdot,a)$\\
            $\pi^k_h(\cdot)\gets\arg\max_{a\in\cA}Q^k_h(\cdot,a)$
        }
        Define $\pi^k = \pi^k_1 \circ \pi^k_2 ... \circ \pi^k_{h_k} \circ \textrm{unif}_{[h_k+1:H]}$\\
        \For{$n=1,...,N$}{
            Receive initial state $s_1^{kn}\sim d_1$\\
            \lFor{$h=1,2,...,H$}{
                Take action $a^{kn}_h\gets \pi^k_h(s_h^{kn})$ and observe $s_{h+1}^{kn}\sim P_h(s^k_h, a^k_h)$
            }
        }
        Compute $\Delta_k \gets \frac{2\beta}{N}\sum_{n=1}^N\sum_{h=1}^{h_k} \sqrt{\phi(s_h^{kn},a_h^{kn})\trans (\Lambda_h^k)^{-1} \phi(s_h^{kn},a_h^{kn})}$.\\
        \lIf{$\Delta_k \geq \frac{\epsilon h_k}{2H}$}{
            $h_{k+1} \gets h_k$
        }
        \lElseIf{$h_k = H$}{
                \Return $\pi^k$
        }
        \lElse{
            $h_{k+1} \gets h_k + 1$
        }
    }
    \caption{Layer-by-Layer Batch Exploration Strategy for Linear MDPs Given Reward Function}\label{alg:DERL_Deterministic_Policy_Algorithm}
\end{algorithm}

In this sub-section, we focus on the setting where each deployed policy is deterministic.
In Alg \ref{alg:DERL_Deterministic_Policy_Algorithm}, we propose a provably deployment-efficient algorithm built on Least-Square Value Iteration with UCB \citep{jin2019provably}\footnote{In order to align with the algorithm in reward-free setting, slightly different from \citep{jin2019provably} but similar to \citep{wang2020rewardfree}, we run linear regression on $P_hV_h$ instead of $Q_h$.} and the ``layer-by-layer'' strategy.
Briefly speaking, at deployment $k$, we focus on exploration in previous $h_k$ layers, and compute $\pi_1^k,\pi_2^k,...,\pi^k_{h_k}$ by running LSVI-UCB in an MDP truncated at step $h_k$. After that, we deploy $\pi^k$ to collect $N$ trajectories, and complete the trajectory after time step $h_k$ with an arbitrary policy. (In the pseudocode we choose uniform, but the choice is inconsequential.) 
In line 19, we compute $\Delta_k$ with samples and use it to judge whether we should move on to the next layer till all $H$ layers have been explored. The theoretical guarantee is listed below, and the missing proofs are deferred to Appendix \ref{appx:RB_BE_with_Deter_Policies}.
\begin{theorem}[Deployment Complexity]\label{thm:deployment_complexity_given_reward}
    For arbitrary $\epsilon,\delta > 0$, and arbitrary $c_K \geq 2$, as long as $N \geq c \Big(c_K\frac{H^{4c_K+1}d^{3c_K}}{\epsilon^{2c_K}}\log^{2c_K}(\frac{Hd}{\delta\epsilon})\Big)^{\frac{1}{c_K-1}}$, 
    where $c$ is an absolute constant, 
    by choosing
    \begin{align}
        K=c_KdH+1. \label{eq:lower_bound_on_N}
    \end{align} 
    Algorithm \ref{alg:DERL_Deterministic_Policy_Algorithm} will terminate at iteration $k \leq K$ and return us a policy $\pi^{k}$, and with probability $1-\delta$, $\EE_{s_1\sim d_1}[V_1^*(s_1)-V_1^{\pi^{k}}(s_1)] \leq \epsilon$. 
\end{theorem}
As an interesting observation, Eq~\eqref{eq:lower_bound_on_N} reflects the trade-off between the magnitude of $K$ and $N$ when $K$ is small. To see this, when we increase $c_K$ and keep it at the constant level, $K$ definitely increases while $N$
will be lower because its dependence on $d,H,\epsilon,\delta$ decreases. Moreover, the benefit of increasing $c_K$ is only remarkable when $c_K$ is small (e.g. we have $N=O(H^9d^6\epsilon^{-4})$ if $c_K=2$, while $N=O(H^5d^{3.6}\epsilon^{-2.4})$ if $c_K=6$), and even for moderately large $c_K$, the value of $N$ quickly  approaches the limit $\lim_{c_K\rightarrow \infty}N=c\frac{H^4d^3}{\epsilon^2}\log^2(\frac{Hd}{\delta\epsilon})$. It is still an open problem that whether the trade-off in Eq.\ref{eq:lower_bound_on_N} is exact or not, and we leave it for the future work.

Another key step in proving the deployment efficiency of Alg. \ref{alg:DERL_Deterministic_Policy_Algorithm} is  Lem. \ref{lem:finite_sample_elliptical_potential_lemma} below. 
In fact, by directly applying Lem. \ref{lem:finite_sample_elliptical_potential_lemma} to LSVI-UCB \citep{jin2019provably} with large batch sizes, we can achieve $O(dH)$ deployment complexity in deterministic policy setting without exploring in a layer-by-layer manner.
We defer the discussion and the additional benefit of layer-by-layer strategy to Appx. \ref{appx:additional_safety_guarantee}.
\begin{restatable}{lemma}{ThmFiniteViolation}[Batched Finite Sample Elliptical Potential Lemma]\label{lem:finite_sample_elliptical_potential_lemma}
    Consider a sequence of matrices $\A_0,\A_{N},...,\A_{(K-1)N}\in \mR^{d\times d}$ with $\A_0=I_{d\times d}$ and $\A_{kN}=\A_{(k-1)N}+\Phi_{k-1}$, where $\Phi_{k-1}=\sum_{t=(k-1)N+1}^{kN}\phi_t\phi_t\trans$ and $\max_{t\leq KN} \|\phi_t\| \leq 1$. We define:
        $
        \mathcal{K}^+:=\Big\{k\in[K]\Big| Tr(\A_{(k-1)N}^{-1}\Phi_{k-1}) \geq N\epsilon\Big\}
        $.
    For arbitrary $\epsilon < 1$, and arbitrary $c_K\geq 2$, if $K = c_K dH+1$, by choosing $
    N\geq c\Big(c_K\frac{Hd^{c_K}}{\epsilon^{c_K}}\log^{c_K}(\frac{Hd}{\epsilon})\Big)^{\frac{1}{c_K-1}}
    $, where $c$ is an absolute constant independent with $c_K, d, H, \epsilon$, 
    we have $|\mathcal{K}^+| \leq c_K d< K/H.$
\end{restatable}
\paragraph{Extension to Reward-free setting}
Based on the similar methodology, we can design algorithms for reward-free setting \citep{wang2020rewardfree} and obtain $O(dH)$ deployment complexity. 
We defer the algorithms and proofs to Appx. \ref{appx:RF_BE_with_Deter_Policies}, and summarize the main result in Thm. \ref{thm:deployment_complexity_reward_free_detailed_version}.

\subsection{Deployment-Efficient RL with Arbitrary Policies}
\label{sec:DERL_arbitrary_policy}

\begin{algorithm}[ht]
    \textbf{Input}: Accuracy level $\epsilon$; Iteration number $i_{\max}$; Resolution $\resolution$; Reward $r$; Bonus coefficient $\beta$.\\
    \For{$h=1,2,...,H$}{
        Initialize $\pi_{h,1}$ with an arbitrary deterministic policy \label{line:policy_collection_start}; $\tilde\Sigma_{h,1}=2I$, $\Pi_h = \{\}$. \\
        \For{$i=1,2,...,i_{\max}$}{
            $\hat\hatcov^{\pi_{h,i}}_h \gets {\rm EstimateCovMatrix}(h, D_{[1:h-1]}, \Sigma_{[1:h-1]}, \pi_{h,i})$ \quad {\color{blue} \# Alg \ref{alg:approx_random_feature_matrix}, Appx \ref{appx:DERL_with_arbitrary_policies}} \label{line:est_cov_matrix} \\
            $\tilde\Sigma_{h,i+1}=\tilde\Sigma_{h,i} + \hat\hatcov^{\pi_{h,i}}_h$ \\
            $V_{h,i+1},~ \bar\pi_{h,i+1} \gets {\rm SolveOptQ}(h, D_{[1:h-1]}, \Sigma_{[1:h-1]}, \beta, \tSigma_{h,i+1},\resolution)$  {\color{blue} \# Alg \ref{alg:optQ_solution}, Appx \ref{appx:DERL_with_arbitrary_policies}}\\
            \lIf{$V_{h,i+1} \leq  3\nu^2_{\min} / 8$}{
                break
            }\label{line:break_criterion}
            $\Pi_h = \Pi_h \bigcup \{\bpi_{h,i+1}\}$
        }\label{line:policy_collection_end}
        $\Sigma_h=I$, $D_h = \{\}$, 
        $\pi_{h,\mathrm{mix}}:= 
        \mathrm{unif}(\Pi_h)$ \label{line:sample_start} \\
        \For{$n=1,2,...,N$}{
            Sample trajectories with $\pi_{h,\mathrm{mix}}$ \\
            $\Sigma_h = \Sigma_h + \phi(s_{h,n}, a_{h,n})\phi(s_{h,n},a_{h,n})\trans,\quad D_h = D_h \bigcup \{s_{h,n},a_{h,n},r_{h,n},s_{h+1,n}\}$
        }\label{line:sample_end}
    }
    \Return $\hat\pi_r \gets \text{Alg \ref{alg:DE_rl_layer_by_layer_reward_free_planning}}(H, \{D_1,...,D_H\}, r)$\label{line:returned_policy}
    \caption{Deployment-Efficient RL with Covariance Matrix Estimation}\label{alg:DERL_with_arbitrary_policies}
\end{algorithm}

From the discussion of lower bounds in Section \ref{sec:lower_bound}, we know that in order to reduce the deployment complexity from $\Omega(dH)$ to $\tilde\Omega(H)$, we have to utilize stochastic (and possibly non-Markov) policies and try to explore as many different directions as possible in each deployment (as opposed to $1$ direction in Algorithm~\ref{alg:DERL_Deterministic_Policy_Algorithm}). 
The key challenge is to find a stochastic policy---before the deployment starts---which can sufficiently explore $d$ independent directions.

In Alg. \ref{alg:DERL_with_arbitrary_policies}, we overcome this difficulty by a new covariance matrix estimation method (Alg. \ref{alg:approx_random_feature_matrix} in Appx. \ref{appx:DERL_with_arbitrary_policies}). The basic idea is that, for arbitrary policy $\pi$ \footnote{Here we mainly focus on evaluating deterministic policy or stochastic policy mixed from a finite number of deterministic policies, because for the other stochastic policies, exactly computing the expectation over policy distribution may be intractable.}, the covariance matrix $\Lambda^\pi_h:=\EE_\pi[\phi\phi\trans]$ can be estimated element-wise by running policy evaluation for $\pi$ with $\phi_i\phi_j$ as a reward function, where $i,j\in[d]$ and $\phi_i$ denotes the $i$-th component of vector $\phi$. 

However, a new challenge emerging is that, because the transition is stochastic, in order to guarantee low evaluation error for all possible policies $\bpi_{h,i+1}$, we need an union bound over all policies to be evaluated, which is challenging if the policy class is infinite. To overcome this issue, we discretize the value functions in Algorithm \ref{alg:optQ_solution} (see Appendix \ref{appx:DERL_with_arbitrary_policies}) to allow for a union bound over the policy space: after computing the Q-function by LSVI-UCB, before converting it to a greedy policy, we first project it to an $\resolution$-net of the entire Q-function class. In this way, the number of policy candidates is finite and the projection error can be controlled as long as $\resolution$ is small enough.

Using the above techniques, in Lines \ref{line:policy_collection_start}-\ref{line:policy_collection_end}, we repeatedly use Alg \ref{alg:approx_random_feature_matrix} to estimate the accumulative covariance matrix $\tSigma_{h,i+1}$ and further eliminate uncertainty by calling Alg \ref{alg:optQ_solution} to find a policy (approximately) maximizing uncertainty-based reward function $\tR:=\|\phi\|_{\tSigma^{-1}_{h,i+1}}$.
For each $h\in[H]$, inductively conditioning on sufficient exploration in previous $h-1$ layers, the errors of Alg \ref{alg:approx_random_feature_matrix} and Alg \ref{alg:optQ_solution} will be small, and we will find a finite set of policies $\Pi_h$ to cover all dimensions in layer $h$. (This is similar to the notion of ``policy cover'' in \citet{du2019provably, agarwal2020pc}.)
Then, layer $h$ can be explored sufficiently by deploying a uniform mixture of $\Pi$ and choosing $N$ large enough (Lines \ref{line:sample_start}-\ref{line:sample_end}).
Also note that the algorithm does not use the reward information, and is essentially a reward-free exploration algorithm. After exploring all $H$ layers, we obtain a dataset $\{D_1,...,D_H\}$ and can use Alg \ref{alg:DE_rl_layer_by_layer_reward_free_planning} for planning with any given reward function $r$ satisfying Assump.~\ref{assump:linear_MDP} to obtain a near-optimal policy.




\paragraph{Deployment complexity guarantees}
We first introduce a quantity denoted as $\nu_{\min}$, which measures the reachability to each dimension in the linear MDP. 
In Appendix \ref{sec:reachability_coefficient}, we will show that the $\nu_{\min}$ is no less than the ``explorability'' coefficient in Definition 2 of \citet{zanette2020provably} and $\nu_{\min}^2$ is also lower bounded by the maximum of the smallest singular value of matrix $\EE_\pi[\phi\phi\trans]$.
\begin{definition}[Reachability Coefficient]\label{def:reachability_coefficient}
    $$
    \nu_h:=\min_{\|\theta\|=1}\max_\pi \sqrt{\EE_\pi[(\phi\trans_h\theta)^2]}~~;\quad\quad \nu_{\min}=\min_{h\in[H]}\nu_h~.
    $$
\end{definition}

Now, we are ready to state the main theorem of this section, and defer the formal version and its proofs to Appendix \ref{appx:DERL_with_arbitrary_policies}. Our algorithm is effectively running reward-free exploration and therefore our results hold for  arbitrary linear reward functions.
\begin{restatable}{theorem}{ThmBEArbPolicy}[Informal]\label{thm:arbitary_policy_informal}
    For arbitrary $0<\epsilon,\delta < 1$, with proper choices of $i_{\max}, \epsilon_0, \beta$, we can choose $N=\mathrm{poly}(d, H, \frac{1}{\epsilon}, \log \frac{1}{\delta}, \frac{1}{\nu_{\min}})$, such that,
    after $K=H$ deployments, with probability $1-\delta$, Algorithm \ref{alg:DERL_with_arbitrary_policies} will collect a dataset $D=\{D_1,...,D_H\}$, and if we run Alg \ref{alg:DE_rl_layer_by_layer_reward_free_planning} with $D$ and arbitrary reward function satisfying Assump.~\ref{assump:linear_MDP}, we will obtain $\hat\pi_r$ such that $V_1^{\hat\pi_r}(s_1;r) \geq V_1^{*}(s_1;r) - \epsilon$.
\end{restatable}
\paragraph{Proof Sketch}
Next, we briefly discuss the key steps of the proof. Since $\epsilon_0$ can be chosen to be very small, we will ignore the bias induced by $\epsilon_0$ when providing intuitions. 
Our proof is based on the induction condition below. We first assume it holds after $h-1$ deployments (which is true when $h=1$), and then we try to prove at the $h$-th deployment we can explore layer $h$ well enough so that the condition holds for $h$. 
\begin{restatable}{condition}{CondInductionError}[Induction Condition]\label{cond:induction_condition}
    Suppose after $h-1$ deployments, we have the following induction condition for some $\xi < 1/d$, which will be determined later:
    \begin{align} \textstyle
        \max_\pi \EE_\pi [\sum_{\tdh=1}^{\h-1}\sqrt{\phi(s_\tdh,a_\tdh)\trans\Sigma_\tdh^{-1}\phi(s_\tdh,a_\tdh)}] \leq \frac{h-1}{H}\xi. \label{eq:induction_condition}
    \end{align}
\end{restatable}
The l.h.s.~of Eq.\eqref{eq:induction_condition} measures the uncertainty in previous $h-1$ layers after exploration. As a result, with high probability, the following estimations will be accurate:
\begin{align}  \label{eq:proof_sketch_1}
    \|\hat\hatcov^{\pi_{h,i}}-\EE_{\pi_{h,i}}[\phi(s_h,a_h)\phi(s_h,a_h)\trans]\|_{\infty,\infty} \leq O(\xi),
\end{align}
where $\|\cdot\|_{\infty,\infty}$ denotes the entry-wise maximum norm. This directly implies that:
\begin{align*}
    \|\tilde{\Sigma}_{h,i+1}-\Sigma_{h,i+1}]\|_{\infty,\infty} \leq i\cdot O(\xi).
\end{align*}
where $\Sigma_{h,i+1}:=2I+\sum_{i'=1}^i\EE_{\pi_{h,i'}}[\phi(s_h,a_h)\phi(s_h,a_h)\trans]$ is the target value for $\tilde\Sigma_{h,i+1}$ to approximate. 
Besides, recall that in Algorithm \ref{alg:optQ_solution}, we use $\sqrt{\phi\trans\tilde\Sigma^{-1}_{h,i+1}\phi}$ as the reward function, and the induction condition also implies that:
\begin{align*}
    |V_{h,i+1}-\max_\pi \EE_\pi[\|\phi(s_h,a_h)\|_{\tilde\Sigma^{-1}_{h,i+1}}]| \leq &O(\xi).
\end{align*}
As a result, if $\xi$ and the resolution $\epsilon_0$ are small enough, $\bar\pi_{h,i+1}$ would gradually reduce the uncertainty and $V_{h,i+1}$ (also $\max_\pi \EE_\pi[\|\phi(s_h,a_h)\|_{\tilde\Sigma^{-1}_{h,i+1}}]$) will decrease. 
However, 
the bias is at the level $O(\xi)$, and therefore, no matter how small $\xi$ is, as long as $\xi > 0$, it is still possible that the policies in $\Pi_h$ do not cover all directions if some directions are very difficult to reach, and the error due to such a bias will be at the same level of the required accuracy in induction condition, i.e. $O(\xi)$. 
This is exactly where the ``reachability coefficient'' $\nu_{\min}$ definition helps. The introduction of $\nu_{\min}$ provides a threshold, and as long as $\xi$ is small enough so that the bias is lower than the threshold, each dimension will be reached with substantial probability when the breaking criterion in Line 9 is satisfied. As a result, by deploying $\mathrm{unif}(\Pi_h)$ and collecting a sufficiently large dataset, the induction condition will 
hold till layer $H$. Finally, combining the guarantee of Alg \ref{alg:DE_rl_layer_by_layer_reward_free_planning}, we complete the proof.

\section{Conclusion and Future Work}
In this paper, we propose a concrete theoretical formulation for DE-RL to fill the gap between existing RL literatures and real-world applications with deployment constraints. Based on our framework, we establish lower bounds for deployment complexity 
in linear MDPs,  and provide novel algorithms and techniques to achieve optimal deployment efficiency. 
Besides, our formulation is flexible and can serve as building blocks for other practically relevant settings related to DE-RL. 
We conclude the paper with two such examples,  defer a more detailed discussion to Appendix \ref{appx:discussion_on_Extended_DERL}, and leave the investigation to future work. 
\paragraph{Sample-Efficient DE-RL}
In our basic formulation in Definition \ref{def:deployment_efficient_linear_mdp}, we focus on minimizing the deployment complexity $K$ and put very mild constraints on the per-deployment sample complexity $N$. In practice, however, the latter is also an important consideration, and we may face additional constraints on how large $N$ can be, as they can be upper bounded by e.g. the number of customers or patients our system is serving. 

\paragraph{Safe DE-RL}
In real-world applications, safety is also an important criterion. The definition for safety criterion in Safe DE-RL is still an open problem, but we believe it is an interesting setting since it implies a trade-off between exploration and exploitation in deployment-efficient setting.




\section*{Acknowledgements}
JH's research activities on this work were completed by December 2021 during his internship at MSRA. NJ acknowledges funding support from ARL Cooperative Agreement W911NF-17-2-0196, NSF IIS-2112471, and Adobe Data Science Research Award.

\bibliography{references}
\bibliographystyle{iclr2022_conference}

\clearpage
\appendix
\section{Extended Related Work}\label{sec:extended_related_work}
\paragraph{Online RL} 
Online RL is a paradigm focusing on the challenge of strategic exploration. 
On the theoretical side, based on the ``Optimism in Face of Uncertainty''(OFU) principle or posterior sampling techniques, many provable algorithms have been developed for tabular MDPs \citep{jin2018qlearning, Azar2017MinimaxRB, zanette2019tighter, agrawal2017posterior,agrawal2021improved}, linear MDPs \citep{jin2019provably, agarwal2020pc}, general function approximation \citep{wang2020reinforcement, NIPS2013_41bfd20a}, or MDPs with structural assumptions \citep{pmlr-v70-jiang17c, du2019provably}. 
Moreover, there is another stream of work studying how to guide exploration by utilizing state occupancy \citep{hazan2019provably, zhang2021made}.
Beyond the learning in MDPs with pre-specified reward function, recently, \citet{jin2020rewardfree, wang2020rewardfree, zanette2020provably} provide algorithms for exploration in the scenarios where multiple reward functions are of interest.
On the practical side, there are empirical algorithms such as intrinsically-motivated exploration \citep{bellemare2016unifying, campero2020learning}, exploration with hand-crafted reward bonus (RND) \citep{burda2018exploration}, and other more sophisticated strategies \citep{ecoffet2019go}.
However, all of these exploration methods do not take deployment efficiency into consideration, and will fail to sufficiently explore the MDP and learn near-optimal policies in DE-RL setting where the number of deployments is very limited.

\paragraph{Offline RL}
Different from the online setting, where the agents are encouraged to explore rarely visited states to identify the optimal policy, 
the pure offline RL setting serves as a framework for utilizing historical data to learn a good policy without further interacting with the environment.
Therefore, the core problem of offline RL is the performance guarantee of the deployed policy, which motivated multiple importance-sampling based off-policy policy evaluation and optimization methods \citep{jiang2016doubly, liu2018breaking, uehara2020minimax, yang2020off, nachum2019algaedice, lee2021optidice}, and the ``Pessimism in Face of Uncertainty'' framework \citep{liu2020provably, kumar2020conservative, fujimoto2021minimalist, yu2020mopo,jin2021pessimism, xie2021bellman} in contrast with OFU in online exploration setting.
However, as suggested in \citet{matsushima2021deploymentefficient}, pure offline RL can be regarded as constraining the total number of deployments to be $1$.

\paragraph{Bridging Online and Offline RL; Trade-off between Pessimism and Optimism}
As pointed out by \citet{xie2021policy, matsushima2021deploymentefficient}, there is a clear gap between existing online and offline RL settings, and some efforts have been made towards bridging them. For example, \citet{nair2021awac, pertsch2020accelerating, bai2020provably} studied how to leverage pre-collected offline datasets to learn a good prior to accelerate the online learning process. \citet{moskovitz2021tactical} proposed a learning framework which can switch between optimism and pessimism by modeling the selection as a bandit problem. None of these works give provable guarantees in our deployment-efficiency setting. 

\paragraph{Conversion from Linear Regret in \citet{gao2021provably} to Sample Complexity}
\citet{gao2021provably} proposed an algorithm with the following guarantee: after interacting with the environments for $\tilde K$ times (we use $\tilde K$ to distinguish with $K$ in our setting), there exists a constant $c$ such that the algorithm's regret is
\begin{align*}
    \sum_{k=1}^{\tilde{K}} V^*(s_1) - V^{\pi_k}(s_1) = c\cdot \sqrt{d^3H^4\tilde{K}}\cdot\iota = c\cdot \sqrt{d^3H^3T}\cdot\iota
\end{align*}
where we denote $T=\tilde K H$, and use $\iota$ to refer to the log terms. Besides, in $\pi_1,...,\pi_K$ there are only $O(dH\log \tilde K)$ policy switching.

As discussed in Section 3.1 by \citet{jin2018qlearning}, such a result can be convert to a PAC guarantee that, by uniformly randomly select a policy $\pi$ from $\pi_1,...,\pi_K$, with probability at least 2/3, we should have:
\begin{align*}
    V^*(s_1)-V^\pi(s_1)=\tilde{O}(\sqrt{\frac{d^3H^5}{T}})=\tilde{O}(\sqrt{\frac{d^3H^4}{\tilde K}})
\end{align*}
In order to make sure the upper bound in the r.h.s. will be $\epsilon$, we need:
\begin{align*}
    \tilde K = \frac{d^3H^4}{\epsilon^2}
\end{align*}
and the required policy switching would be:
\begin{align*}
    O(dH\log\tilde K) = O(dH\log\frac{dH}{\epsilon})
\end{align*}
In contrast with our results in Section \ref{sec:DERL_deterministic_policy}, there is an additional logarithmic dependence on $d,H$ and $\epsilon$. Moreover, since their algorithm only deploys deterministic policies, their deployment complexity has to depend on $d$, which is much higher than our stochastic policy algorithms in Section \ref{sec:DERL_arbitrary_policy} when $d$ is large.

\paragraph{Investigation on Trade-off between Sample Complexity and Deployment Complexity}
\begin{table}[h!]
    \begin{center}
    \begin{tabular}{c| c | c |c|c} 
    \hline\hline    
    Methods & R-F? & Deployed Policy & \begin{tabular}[x]{@{}c@{}}\# Trajectories\end{tabular} &  \begin{tabular}[x]{@{}c@{}}Deployment \\Complexity\end{tabular}   \\
    \hline\hline
    \begin{tabular}[x]{@{}c@{}}LSVI-UCB\\ \citep{jin2019provably}\end{tabular} & $\times$ & Deterministic & \multicolumn{2}{c}{$\tilde{O}(\frac{d^3H^4}{\epsilon^2})$}
    \\[6pt]   
    \hline
    \begin{tabular}[x]{@{}c@{}}Reward-free LSVI-UCB \\ \citep{wang2020rewardfree}\end{tabular} & \checkmark & Deterministic &  \multicolumn{2}{c}{$\tilde{O}(\frac{d^3H^6}{\epsilon^2})$}
    \\[6pt]
    \hline
    \begin{tabular}[x]{@{}c@{}}FRANCIS \\ \citep{zanette2020provably}\end{tabular} & \checkmark & Deterministic &  \multicolumn{2}{c}{$\tilde{O}(\frac{d^3H^5}{\epsilon^2})$}
    \\[6pt]
    \hline
    \begin{tabular}[x]{@{}c@{}}    \citet{gao2021provably} \end{tabular} & $\times$ & Deterministic  & $\tilde{O}(\frac{d^3H^4}{\epsilon^2})$ & $O(dH\log\frac{dH}{\epsilon})$
    \\[6pt]
    \hline
    \begin{tabular}[x]{@{}c@{}}Q-type OLIVE \\ \citep{jiang2017contextual}\\
    \citep{jin2021bellman}\end{tabular} & $\times$ & Deterministic  & $\tilde{O}(\frac{d^3H^6}{\epsilon^2})$ & $O(dH\log\frac{1}{\epsilon})$
    \\[6pt]
    \hline
    \begin{tabular}[x]{@{}c@{}}Simplified MOFFLE \\ \citep{modi2021model}\end{tabular} & \checkmark & Stochastic  & $\tilde{O}(\frac{d^8H^7|\cA|^{13}}{\min(\epsilon^2\eta_{\min},\eta_{\min}^5)})$ & $\tilde{O}(\frac{Hd^3|\cA|^4}{\eta_{\min}^2})$
    \\[6pt]
    \hline
    Alg. \ref{alg:DERL_Deterministic_Policy_Algorithm} [\red{Ours}] & $\times$ & Deterministic & $\tilde{O}(\frac{d^4H^5}{\epsilon^2})$ & $O(dH)$ \\[6pt]
    \hline
    Alg. \ref{alg:DE_rl_layer_by_layer_reward_free_exploration} + \ref{alg:DE_rl_layer_by_layer_reward_free_planning}[\red{Ours}] & \checkmark & Deterministic & $\tilde{O}(\frac{d^4H^7}{\epsilon^2})$ & $O(dH)$ \\[6pt]
    \hline
    Alg. \ref{alg:DERL_with_arbitrary_policies}[\red{Ours}] & \checkmark & Stochastic & 
    $\tilde{O}(\frac{d^3H^4}{\epsilon^2 \nu^2_{\min}}+\frac{d^7H^2}{\nu^{14}_{\min}})$
    & $H$\\[6pt]
    \hline\hline    
    \end{tabular}
\end{center}
\caption{Comparison between our algorithms and online RL methods without considering deployment costraints in our setting defined in Def. \ref{def:deployment_efficient_linear_mdp}, where R-F is the short note for Reward-Free. The total number of trajectories cost by our methods is computed by $K\cdot N$. We omit log terms in $\tilde{O}$. For algorithm \citep{jin2019provably}, we report the sample complexity after the conversion from regret. For our deterministic policy algorithms, we report the asymptotic results when $c_K\rightarrow +\infty$, which can be achieved approximately when $c_K$ is a large constant (e.g. $c_K=100$).}\label{table:DE_SE_trade_off_comparison}
\end{table}

In Table \ref{table:DE_SE_trade_off_comparison}, we compare our algorithms and previous online RL works which did not consider deployment efficiency to shed light on the trade-off between sample and deployment complexities. 
Besides algorithms that are specialized to linear MDPs, we also include results such as \citet{zanette2020provably}, which studied a more general linear approximation setting and can be adapted to our setting. As stated in Def.~2 of \citet{zanette2020provably}, they also rely on some reachability assumption. To avoid ambiguity, we use $\tilde{\nu}_{\min}$ to refer to their reachability coefficient (as discussed in Appx \ref{sec:reachability_coefficient}, $\tilde{\nu}_{\min}$ is no larger than and can be much smaller than our $\nu_{\min}$). Because they also assume that $\epsilon\leq \tilde{O}(\tilde{\nu}_{\min}/\sqrt{d})$ (see Thm 4.1 in their paper), their results have an implicit dependence on 
$\tilde{\nu}^{-2}_{\min}$. 
In addition, by using the class of linear functions w.r.t.~$\phi$, Q-type OLIVE  \citep{jiang2017contextual,jin2021bellman} has $\tilde{O}(\frac{d^3H^6}{\epsilon^2})$ sample complexity and $O(dH\log(1/\epsilon))$ deployment complexity. Its deployment complexity is close to our deterministic algorithm, but with additional dependence on $\epsilon$. We also want to highlight that OLIVE is known to be computationally intractable \citep{dann2018oracle}, while our algorithms are computationally efficient. With the given feature $\phi$ in linear MDPs and additional reachability assumption (not comparable to us), we can use a simplified version of MOFFLE \citep{modi2021model} by skipping their LearnRep subroutine. Though this version of MOFFLE is computationally efficient and its deployment complexity does not depend on $\epsilon$, it has much worse sample complexity ($\eta_{\min}$ is their reachability coefficient) and deployment complexity. On the other hand, PCID \citep{du2019provably} and HOMER \citep{misra2020kinematic} achieve $H$ deployment complexity in block MDPs. However, block MDPs are more restricted than linear MDPs and these algorithms have worse sample and computational complexities.

It is worth to note that all our algorithms achieve the optimal dependence on $\epsilon$ (i.e., $\epsilon^{-2}$) in sample complexity. For algorithms that deploy deterministic policies, we can see that our algorithm has higher dependence on $d$ and $H$ in the sample complexity in both reward-known and reward-free setting, while our deployment complexity is much lower.
Our stochastic policy algorithm (last row) is naturally a reward-free algorithm. Comparing with \citet{wang2020rewardfree} and \citet{zanette2020provably}, our sample complexity has higher dependence on $d$ and the reachability coefficient $\nu_{\min}$, while our algorithm achieves the optimal deployment complexity and better dependence on $H$.

\section{On the Lower Bound of Deployment-Efficient Reinforcement Learning}\label{appx:lower_bound}
\subsection{A Hard MDP Instance Template}
\begin{figure}[h]
    \begin{center}
    \caption{Lower Bound Instance Template. The states in each layer can be divided into three groups. {\color{green} \textbf{Group 1}}: absorbing states, marked with green; {\color{red} \textbf{Group 2}}: core state, marked with red; {\color{blue}\textbf{Group 3}}: normal states, marked with blue.}\label{fig:lower_bound_instance}
    \includegraphics[width=1.0\textwidth]{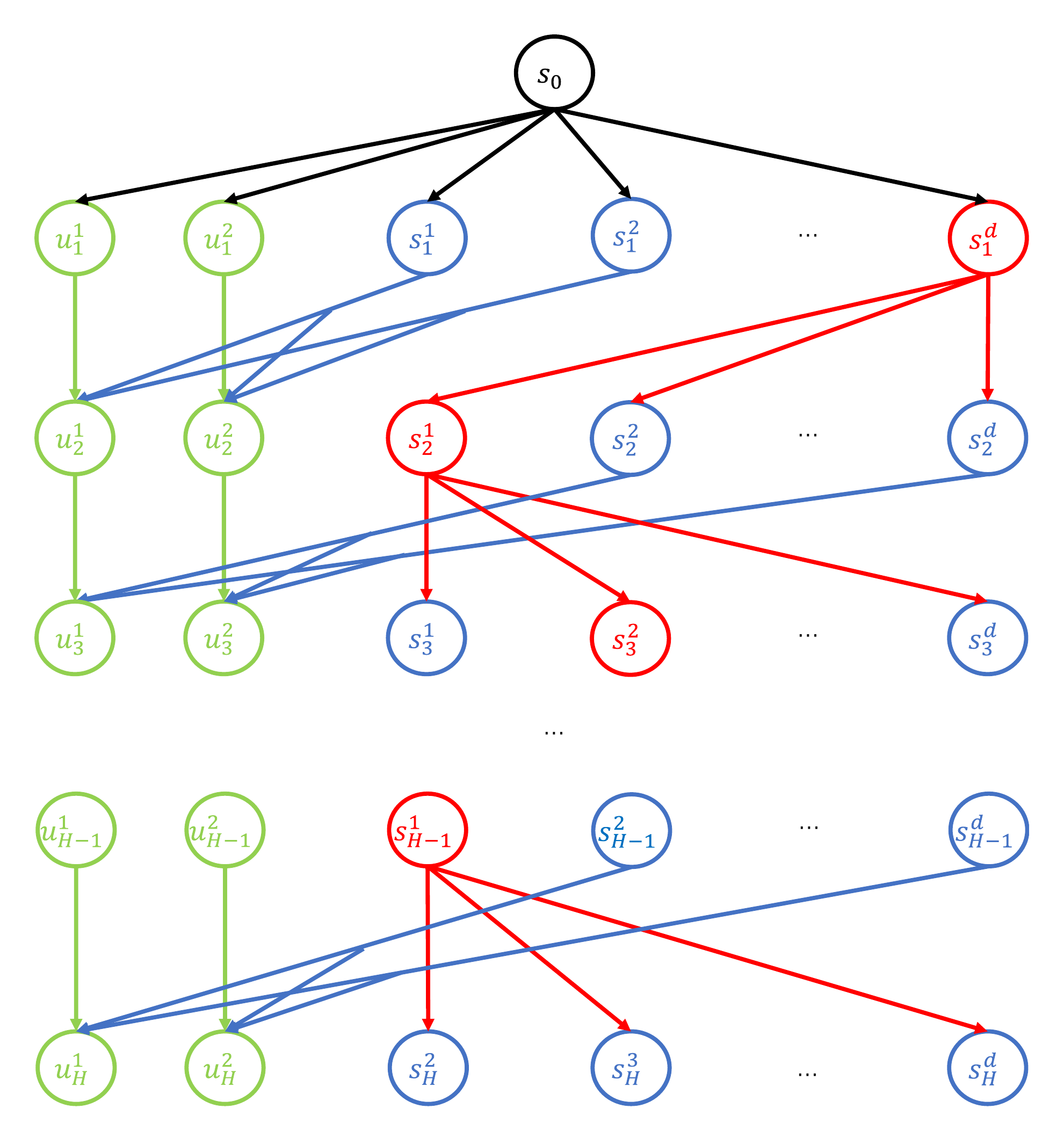}
    \end{center}
\end{figure}
In this sub-section, we first introduce a hard MDP template that is used in further proofs.
As shown in Figure \ref{fig:lower_bound_instance}, we construct a tabular MDP (which is a special case of linear MDP) where the horizon length is $H+1$ and in each layer except the first one, there are $d+2$ states and $2d+1$ different state action pairs. The initial state is fixed as $s_0$ and there are $d+2$ different actions. It is easy to see that we can represent the MDP by linear features with at most $2d+1$ dimensions, and construct reward and transition function satisfying Assumption \ref{assump:linear_MDP}.
As a result, it is a linear MDP with dimension $2d+1$ and horizon length $H+1$.
Since there is only a constant-level blow up of dimension, the dimension of these MDPs is still $\Theta(d)$, and we will directly use $d$ instead of $\Theta(d)$ in the rest of the proof.
The states in each layer $h \geq 1$ can be divided into three groups and we introduce them one-by-one in the following.

\paragraph{Group 1: Absorbing States (Green Color)}
The first group $G_h^1=\{u_h^1,u_h^2\}$ consists of two absorbing states $u_h^1$ and $u_h^2$, which can only take one action at each state $\bar a_h^1$ and $\bar a_h^2$ and transit to $u_{h+1}^1$ and $u_{h+1}^2$ with probability 1, respectively. The reward function is defined as $r_h(u_h^1, \bar a_h^1)=r_h(u_h^2, \bar a_h^2) = 0.5$ for all $h\leq H-2$ and $r_{H}(u_{H}^1, \bar a_{H}^1)=0.0$, $r_{H}(u_{H}^2, \bar a_{H}^2) = 1.0$.

\paragraph{Group 2: Core States (Red Color)}
The second group $G_h^2=\{s_h^*\}$ only contains one state, which we call it core state and denote it as $s_h^*$. For example, in Figure \ref{fig:lower_bound_instance}, we have $s_1^* = s_1^d, s_2^*=s_2^1$ and $s_3^*=s_3^2$. In the ``core state'', the agent can take $d$ actions $\tilde a^1_h,\tilde a^2_h,\ldots,\tilde a^d_h$ and transit deterministically to $s_{h+1}^1,s_{h+1}^2,\ldots,s_{h+1}^d$. Besides, the reward function is $r_h(s_h^*,\tilde a_h^i)=0.5$ for all $i\in[d]$.

\paragraph{Group 3: Normal States (Blue Color)}
The third group $G_h^3=\{s_h^i|i\in[d], s_h^i \neq s_h^*\}$ is what we call ``normal states'', and each state $s_h^i \in G_h^3$ can only take one action $\tilde a_h^i$ and will transit randomly to one of the absorbing states in the next layer, i.e. $G_{h+1}^1$. Besides, the reward function is $r_h(s_h^i,\tilde{a}_h^i)=0.5$ for arbitrary $s_h^i\in G_h^3$, and the transition function is $P(u_{h+1}^1|s_h^i,\tilde{a}_h^i)=P(u_{h+1}^1|s_h^i,\tilde{a}_h^i)=0.5$, except for a state action pair $s^\#:=s^{i^\#}_{h^\#}, a^\#:= a^{i^\#}_{h^\#}$ at layer $h^\#\in[H-1]$ with index $i^\#$, such that $s^\# \not\in G^2_{h^\#}$ and $P(u_{h+1}^1|s^{i^\#}_{h^\#},a^{i^\#}_{h^\#})=0.5-\epsilon$ and $P(u_{h+1}^2|s^{i^\#}_{h^\#},a^{i^\#}_{h^\#})=0.5 + \epsilon$. In the following, we will call $s^\#, a^\#$ the ``optimal state'' and ``optimal action'' in this MDP. Note that the ``optimal state'' can not be the core state.

We will use $M(h^\#, i^\#, I_{core}=\{i_1,i_2,...,i_H\})$ with $i_{h^\#} \neq i^\#$ to denote the MDP whose optimal state is at layer $h^\#$ and indexed by $i^\#$, and the core states in each layer are $s_h^{i_1}, s_h^{i_2},...,s_H^{i_H}$. 
As we can see, the only optimal policy should be the one which can generate the following sequence of states before transiting to absorb states at layer $h^\#+1$:
\begin{align*}
    s_0, s_1^{i_1}, s_2^{i_2}, ..., s_{h^\#-1}^{i_{h^\#-1}}, s_{h^\#}^{i^\#}
\end{align*}
and the optimal value function would be $\frac{1}{2}H + \epsilon$. 
In order to achieve $\epsilon$-optimal policy, the algorithm should identify $s^{i^\#}_{h^\#}$, which is the main origin of the difficulty in exploration.

\begin{remark}[Markov v.s. Non-Markov Policies]
    As we can see, the core states in each layer are the only states with \#actions $>$ 1, and for each core state, there exists and only exists one deterministic path (a sequence of states, actions and rewards) from initial state to it, which implies that for arbitrary non-Markov policy, there exists an equivalent Markov policy. Therefore, in the rest of the proofs in this section, we only focus on Markov policies.
\end{remark}

\subsection{Lower bound for algorithms which can deploy deterministic policies only}

In the following, we will state the formal version of the lower bound theorem for deterministic policy setting and its proof. The intuition of the proof is that we can construct a hard instance, which can be regarded as a $\Omega(dH)$ multi-arm bandit problem, and we will show that in expectation the algorithm need to ``pull $\Omega(dH)$ arms" before identifying the optimal one.


\begin{theorem}[Lower bound for number of deployments in deterministic policy setting]
    For the linear MDP problem with dimension $d$ and horizon $H$, given arbitrary algorithm $\psi$ ($\psi$ can be deterministic or stochastic), which can only deploy a deterministic policy but can collect arbitrary number of samples in each deployment, there exists a MDP problem where the optimal deterministic policy $\pi^*$ is $\epsilon$ better than all the other deterministic policies, but the estimate policy $\hat\pi$ (which is also a deterministic policy) of the best policy output by $\psi$ after $K$ deployments must have $P(\pi^* \neq \hat\pi) \geq 1/10$ unless the number of deployments $K>(d-1)(H-1)/2=\Omega(dH)$.
\end{theorem}
\begin{proof}
First of all, we introduce how we construct hard instances.
\paragraph{Construction of Hard Instances}
We consider a set of MDPs $\bar\cM$, where for each MDP in that set, the core states (red color) in each layer are fixed to be $s_1^1, s_2^1,...,s_{H-1}^1$ and the only optimal states which has different probability to transit to absorbing states are randomly selected from $(d-1)(H-1)$ normal states (blue color). Easy to see that, $|\bar\cM|=(d-1)(H-1)$.

Because of the different position of optimal states, the optimal policies for each MDP in $\bar\cM$ (i.e. the policy which can transit from $s_0$ to optimal state) is different. We will use $\pi_1,\pi_2,...,\pi_{(d-1)(H-1)}$ to refer to those different policies and use $M_{\pi_i}$ with $1\leq i\leq (d-1)(H-1)$ to denote the MDP in $\bar{\cM}$ where $\pi_i$ is the optimal policy. For convenience, we will use $M_{0}$ to denote the MDP where all the normal states have equal probability to transit to different absorbing states, i.e., all states are optimal states.
Based on the introduction above, we define $\cM:=\bar \cM \bigcup \{M_0\}$ and use the MDPs in $\cM$ as hard instances.

\paragraph{Lower Bound for Average Failure Probability}
Next, we try to lower bound the average failure probability, which works as a lower bound for the maximal failure probability among MDPs in $\cM$. Since any randomized algorithm is just a distribution over deterministic
ones, and it therefore suffices to only consider deterministic algorithms $\psi$ \citep{krishnamurthy2016pac}.

Given an arbitrary algorithm $\psi$ and $k\in[K]$, we use $\psi(k)$ to denote the policy taken by $\psi$ at the $k$-th deployment (which is a random variable). Besides, we denote $\psi(K+1)$ as the output policy.

For arbitrary $k\in[K]$, we use $P_{M_{\pi_i}, \psi}(\psi(k)=\pi_j)$ with $1\leq i,j\leq (d-1)(H-1)$ to denote the probability that $\psi$ takes policy $\pi_j$ at deployment $k$ when running $\psi$ on $M_{\pi_i}$, and use $P_{M_{\pi_i}, \psi}(\psi(K+1)=\pi_i)$ to denote the probability that the algorithm $\psi$ returns policy $\pi_i$ as optimal arm after running with $K$ deployments under MDP $M_{\pi_i}$. We are interested in providing an upper bound for the expected success rate:
\begin{align*}
    P_{\psi, M\sim \cM}(\text{success}) :=&  \frac{1}{|\cM|} P(\text{success~in}~M_0)+\frac{1}{|\cM|}\sum_{i=1}^{|\bar\cM|} P_{M_{\pi_i}, \psi}(\psi(K+1)= \pi_i) \\
    =&\frac{1}{|\cM|}+\frac{1}{|\cM|}\sum_{i=1}^{|\bar\cM|} P_{M_{\pi_i}, \psi}(\psi(K+1)=\pi_i),
\end{align*}
where we assume that all the policies in $M_0$ are optimal policies.

In the following, we use $E_{k,\pi_i}$ to denote the event that the policy $\pi_i$ has been deployed at least once in the first $k$ deployments and $P_{M_{\pi_j}, \psi}(\cdot)$ to denote the probability of an event when running algorithm $\psi$ under MDP $M_{\pi_j}$. 

Next, we prove that, for arbitrary $M_{\pi_i}$,
\begin{align}
    P_{M_{\pi_i}, \psi}(E_{k,\pi_i}^\complement)=P_{M_{0}, \psi}(E_{ k,\pi_i}^\complement) \quad \forall k\in[K+1]\label{eq:induction_same_behavior}.
\end{align}

First of all, it holds for $k=1$, because at the beginning $\psi$ has't observe any data, and all its possible behavior should be the same in both $M_0$ and $M_i$, and therefore $P_{M_{\pi_i}, \psi}(E_{1,\pi_i}^\complement)=P_{M_{0}, \psi}(E_{1,\pi_i}^\complement)$. Next, we do induction. Suppose we already know it holds for $1,2,...,k$, then consider the case for $k+1$. Because $\psi$ behave the same if the pre-collected episodes are the same, which is the only information it will use for decision, we should have:
\begin{align*}
    P_{M_{\pi_i}, \psi}(\psi(k+1)=\pi_i\cap E_{k,\pi_i}^\complement)=&\sum_{\tau \in \psi(k+1)=\pi_i\cap E_{k,\pi_i}^\complement}P_{M_{\pi_i}, \psi}(\tau)\\
    =&\sum_{\tau \in \psi(k+1)=\pi_i\cap E_{k,\pi_i}^\complement}P_{M_0, \psi}(\tau)\\
    =&P_{M_{0}, \psi}(\psi(k+1)=\pi_i \cap E_{k,\pi_i}^\complement).\numberthis\label{eq:induction_same_policy_choice}
\end{align*}
The second equality is due to each trajectory $\tau\in \psi(k+1)=\pi_i\cap E_{k,\pi_i}^\complement$ has the same probability under $M_0$ and $M_i$ by the construction. Notice that in this induction step, we only consider the trajectory with first $(k+1)N$ episodes because define the whole sample space and event only based on the first $(k+1)N$ episodes.

This implies that,
\begin{align*}
    P_{M_{\pi_i}, \psi}(E_{k+1,\pi_i}^\complement)=&P_{M_{\pi_i}, \psi}(E_{k,\pi_i}^\complement)-P_{M_{\pi_i}, \psi}(\psi(k+1)=\pi_i\cap E_{k,\pi_i}^\complement)\\
    =&P_{M_{0}, \psi}(E_{k,\pi_i}^\complement)-P_{M_{0}, \psi}(\psi(k+1)=\pi_i\cap E_{k,\pi_i}^\complement)\\
    =&P_{M_{0}, \psi}(E_{k+1,\pi_i}^\complement).
\end{align*}

Now we are ready to bound the failure rate. Suppose $K < (d-1)(H-1)/2 < |\cM|/2$, we have:
\begin{align*}
    &\frac{1}{|\cM|}\sum_{i=1}^{|\bar\cM|} \left(P_{M_{\pi_i}, \psi}(\psi(K+1)=\pi_i) - P_{M_{0}, \psi}(\psi(K+1)=\pi_i)\right)\\
    =& \frac{1}{|\cM|}\sum_{i=1}^{|\bar\cM|} \Big(P_{M_{\pi_i}, \psi}(\psi(K+1)=\pi_i\cap E_{K+1,\pi_i})- P_{M_{0}, \psi}(\psi(K+1)=\pi_i\cap E_{K+1,\pi_i})\Big)\\
    &+\frac{1}{|\cM|}\sum_{i=1}^{|\bar\cM|}\Big(P_{M_{\pi_i}, \psi}(\psi(K+1)=\pi_i\cap E_{K+1,\pi_i}^\complement)  - P_{M_{0}, \psi}(\psi(K+1)=\pi_i\cap E_{K+1,\pi_i}^\complement)\Big)\\
    =&\frac{1}{|\cM|}\sum_{i=1}^{|\bar\cM|}P_{M_{\pi_i}, \psi}(E_{K+1,\pi_i})\Big(P_{M_{\pi_i}, \psi}(\psi(K+1)=\pi_i|E_{K+1,\pi_i})  - P_{M_{0}, \psi}(\psi(K+1)=\pi_i|E_{K+1,\pi_i})\Big) \tag{Eq.\eqref{eq:induction_same_policy_choice} and $P(A\cap B)=P(B)P(A|B)$} \\
    \leq& \frac{1}{|\cM|}\sum_{i=1}^{|\bar\cM|} P_{M_{\pi_i}, \psi}(E_{K+1,\pi_i}) \tag{$P_{M_{\pi_i}, \psi}(\psi(K+1)=\pi_i|E_{K+1,\pi_i})- P_{M_{0}, \psi}(\psi(K+1)=\pi_i|E_{K+1,\pi_i })\leq 1$}\\
    =& \frac{1}{|\cM|}\sum_{i=1}^{|\bar\cM|} P_{M_{0}, \psi}(E_{K+1,\pi_i}) \tag{ Eq.~\eqref{eq:induction_same_behavior}}\\
    \leq& \frac{K}{|\cM|}
\end{align*}
where the last step is because:
\begin{align*}
    &\frac{1}{|\cM|}\sum_{i=1}^{|\bar\cM|} P_{M_{0}, \psi}(E_{K+1,\pi_i})\\
    =& \frac{1}{|\cM|}\sum_{i=1}^{|\bar\cM|} \EE_{M_0,\psi}[\mathbf{1}\{\pi_i \text{ is selected}\}]\\
    \leq& \frac{1}{|\cM|}\sum_{i=1}^{|\bar\cM|} \EE_{M_0,\psi}[\sum_{k=1}^{K+1} \mathbf{1}\{\pi_i \text{ is selected at deployment } k\}]\\
    =& \frac{1}{|\cM|}\EE_{M_0,\psi}[\sum_{k=1}^{K+1} \sum_{i=1}^{|\bar\cM|}\mathbf{1}\{\pi_i \text{ is selected at deployment } k\}]\\
    =& \frac{K+1}{|\cM|} \tag{$\psi$ deploy deterministic policy each time}\\
    \leq& \frac{1}{2} \tag{Deployment time $K<|\cM|/2$}
\end{align*}
As a result,
\begin{align*}
    \frac{1}{|\cM|}+\frac{1}{|\cM|}\sum_{i=1}^{|\bar\cM|} P_{M_{\pi_i}, \psi}(\psi(K+1)=\pi_i) \leq \frac{1}{|\cM|}+\frac{1}{|\cM|}\sum_{i=1}^{|\bar\cM|} P_{M_{0}, \psi}(\psi(K+1)=\pi_i) + \frac{1}{2} \leq \frac{2}{|\cM|}+\frac{1}{2}.
\end{align*}
As long as $d, H \geq 3$, we have $|\cM|=(d-1)(H-1)+1 \geq 5$ the failure rate will be higher than $1/10$, which finishes the proof.
\end{proof}


\subsection{Proof for Lower Bound in Arbitrary Setting}


In the following, we provide a formal statement of the lower bound theorem for the arbitrary policy setting and its proof.

\begin{theorem}[Lower bound for number of deployments in arbitrary setting]
For the linear MDP problem with given dimension $d\ge 2$ and horizon $H\ge 3$, $N={\rm poly}(d, H, \frac{1}{\epsilon}, \log \frac{1}{\delta})$, and arbitrary given  algorithm $\psi$. Unless the number of deployments $K> \frac{H-2}{2 \lceil \log_d NH\rceil} = \Omega(H/\log_d (NH)) =\tilde \Omega(H)$, for any $\epsilon$, there exists an MDP such that the output policy is not $\epsilon$-optimal with probability at least $\frac{1}{2e}$. Here $\psi$ can be deterministic or stochastic. The algorithm can deploy arbitrary policy but can only collect $N={\rm poly}(d, H, \frac{1}{\epsilon}, \log \frac{1}{\delta})$ samples in each deployment.
\end{theorem}

\begin{proof}
Since any randomized algorithm is just a distribution over deterministic ones, it suffices to only consider deterministic algorithms $\psi$ in the following proof \citep{krishnamurthy2016pac}. The crucial part here is notice that a deployment means we have a fixed distribution (occupancy) over the state space and such distribution only depends on the prior information.

\paragraph{Construction of Hard Instances}
We have $d^{H-2}\times d$ instances by enumerating the location of core state from level 1 to $H-2$ and the optimal normal state at level $H-1$. We assign $s_{H-1}^{(i+1)\% d}$ as the core state at level $H-1$ if $s_{H-1}^i$ is the optimal state. Notice that for this hard instance class, we only consider the case that the optimal state is in level $H-1$. We use $\cM$ to denote this hard instance class. 

We make a few claims and later prove these claims and the theorem. We will use the notation $E_i(j)$ to denote the event that at least one state at level $j$ is reached by the $i$-th deployment. Also we notice that in all the discussion an event is just a set of trajectories. For all related discussion, the state at level $L$ does not include the state in the absorbing chain. In addition, we will use $P_{\cM,\psi}$ to denote the distribution of trajectories when executing algorithm $\psi$ and uniformly taking an instance from the hard instance class. 

\begin{claim}\label{claim:claim_1}
Assume $L\le H-2$. Then for any deterministic algorithm $\psi$, we have
$$P_{\cM,\psi}(E_1(L))\le \frac{N}{d^{L-1}}.$$
\end{claim}

\begin{claim}\label{claim:claim_2}
Assume $L+L'\le H-2$. We have that for any deterministic algorithm $\psi$,
$$P_{\cM,\psi}(E_k^\complement(L'+L))\ge(1-\frac{N}{d^{L-1}})P(E^\complement_{k-1}(L')).$$
\end{claim}


\paragraph{Proof of Claim \ref{claim:claim_1}}
By the nature of the deterministic algorithm, we know that for any deterministic algorithm $\psi$, the deployment is the same at the first time for all instances. The reason is that the agent hasn't observed anything, so the deployed policy has to be the same.

Let $p(i,j,h)$ denote the probability of the first deployment policy to choose action $j$ at node $i$ at level $h$ under the first deployment policy. We know that $p(i,j,h)$ for $\psi$ is the same under all instances.

Note that there is a one to one correspondence between an MDP in the hard instance class and the specified locations of core states in the layer $1,\ldots,H-2$ and the optimal state at level $H-1$. Therefore, we can use $(i_1,\ldots, i_{H-2},s_{H-1})$ to denote any instance in the hard instance class, where $i_1,\ldots,i_{H-2}$ refers to the location of the core states and $s_{H-1}$ refers to the location of the optimal state. From the construction, we know that for instance $(i_1,\ldots, i_{H-2},s_{H-1})$, to arrive at level $s_L$ at level $L$, the path as to be $s_0,i_1,\ldots,i_{L-1},s_L$. Therefore the probability of a trajectory sampled from $\psi$ to reach state $s_L$ is $$p(s_0,i_1,0)p(i_1,i_2,1)\ldots p(i_{L-2},i_{L-1},L-2)p(i_{L-1},s_L,L-1).$$
Here we use $p(i_{L-1},s_L,L-1)$ denotes the probability of taking action at $i_{L-1}$ to transit to $s_L$ and similarly for others.
In the deployment, $\psi$ draws $N$ episodes, so the probability of executing $\psi$ to reach any state $s_L$ at level $L$ during the first deployment is no more than $N p(s_0,i_1,0)p(i_1,i_2,1)\ldots p(i_{L-2},i_{L-1},L-2)p(i_{L-1},s_L,L-1)$.

Calculating the sum over $i_1,\ldots,i_{L-1}$ and $s_L$ gives us
\begin{align}
&\sum_{i_1,i_2,\ldots,i_{L-1},s_L} Np(s_0,i_1,0)p(i_1,i_2,1)\ldots p(i_{L-2},i_{L-1},L-2)p(i_{L-1},s_L,L-1)\notag\\
=&N \sum_{i_1}p(s_0,i_1,0)\sum_{i_2}p(i_1,i_2,1)\ldots\sum_{i_{L-2}} p(i_{L-3},i_{L-2},L-3)\sum_{i_{L-1}} p(i_{L-2},i_{L-1},L-2)\sum_{s_L} p(i_{L-1},s_L,L-1)\notag\\
=&N \sum_{i_1}p(s_0,i_1,0)\sum_{i_2}p(i_1,i_2,1)\ldots\sum_{i_{L-2}} p(i_{L-3},i_{L-2},L-3)\sum_{i_{L-1}} p(i_{L-2},i_{L-1},L-2)\notag\\
=&N \sum_{i_1}p(s_0,i_1,0)\sum_{i_2}p(i_1,i_2,1)\ldots\sum_{i_{L-2}} p(i_{L-3},i_{L-2},L-3)\notag\\
=&\ldots\notag\\
=&N\label{eq:sum_n}.
\end{align}

Therefore we have the following equation about $P_{\cM,\psi}(E_1(L))$
\begin{align*}
&P_{\cM,\psi}(E_1(L))\\
=&\frac{1}{|\cM|}\sum_{i_{L},\ldots,i_{H-2},s_{H-1}}\sum_{i_1,\ldots,i_{L-1}}P_{I=(i_1,\ldots,i_{H-1},s_{H-1}),\psi}(E_1(L))\\
\le &\frac{1}{d^{H-1}}\sum_{i_{L},\ldots,i_{H-2},s_{H-1}}\sum_{i_1,\ldots,i_{L-1},s_L}Np(s_0,i_1,0)p(i_1,i_2,1)\ldots p(i_{L-2},i_{L-1},L-2)p(i_{L-1},s_L,L-1)\\
=&\frac{1}{d^{H-1}}\sum_{i_{L},\ldots,i_{H-2},s_{H-1}}N\\
=&\frac{N}{d^{L-1}}.
\end{align*}

\paragraph{Proof of Claim \ref{claim:claim_2}}

Let $\tau$ denote any possible concatenation of the first $kN$ episodes we get in the first $k$ deployments. In this claim, it suffices to consider the $kN$ episodes because the event $E_k(L'+L) \cap E^\complement_{k-1}(L')$ only depends on the first $kN$ episodes. Therefore the sample space and the event will be defined on any trajectory with $kN$ episodes. For any $\tau$, we know that $\psi$ will output the $k$-th deployment policy solely based on the $\tau[0,k-1]$ and this map is deterministic (we use $\tau[i,j]$ to denote the $iN +1$ to $jN$ episodes in $\tau$). In other words, $\psi$ will map $\tau[0,k-1]$ to a fixed policy $\psi(\tau[0,k-1])$ to deploy at the $k$-th time. 

We have the following equation for any $I\in \cM$
\begin{align*}
    &P_{I,\psi}(E_k(L'+L)\cap E_{k-1}^\complement(L'))\\
    =&\sum_{\tau \in E_k(L'+L)\cap E_{k-1}^\complement(L')}P_{I,\psi}(\tau)\\
    =&\sum_{\tau \in E_k(L'+L)\cap E_{k-1}^\complement(L')}P_{I,\psi}(\tau[0,k-1])P_{I,\psi(\tau[0,k-1])}(\tau[k-1,k])\\
    =&\sum_{\tau:\tau[0,k-1]\in E_{k-1}^\complement(L'),\tau':\tau'[0,k-1]=\tau[0,k-1] \text{ and }\tau'[k-1,k] \text{ hit level $L'+L$} }P_{I,\psi}(\tau[0,k-1])P_{I,\psi(\tau[0,k-1])}(\tau'[k-1,k])\\
    =&\sum_{\tau:\tau[0,k-1] \in E_{k-1}^\complement(L')}P_{I,\psi}(\tau[0,k-1])\sum_{\tau':\tau'[0,k-1]=\tau[0,k-1] \text{ and }\tau'[k-1,k] \text{ hit level $L'+L$}}P_{I,\psi(\tau[0,k-1])}(\tau'[k-1,k])
\end{align*}

Notice that this equality does not generally hold for probability distribution $P_{\cM,\psi}$.

Then we fix $\tau[0,k-1]$, such that $\tau[0,k-1]\in E_{k-1}^\complement(L')$. We also fix $(i_1,\ldots, i_{L'-1})$, $(i_{L'+L},\ldots, i_{H-2},s_{H-1})$ and consider two instances $I_1=(i_1,\ldots, i_{L'-1},i_{L'}^1,\ldots,i_{L'+L-1}^1,i_{L'+L},\ldots,i_{H-2},s_{H-1})$ and $I_2=(i_1,\ldots, i_{L'-1},i_{L'}^2,\ldots,i_{L'+L -1}^2,i_{L'+L},\ldots,i_{H-2},s_{H-1})$. Therefore, we have that $P_{I_1,\psi}(\tau[0,k-1])=P_{I_2,\psi}(\tau[0,k-1])$ (from the construction of $I_1$, $I_2$ and the property of deterministic algorithm $\psi$). We use $I(i_1,\ldots, i_{L'-1},i_{L'+L},\ldots,i_{H-2},s_{H-1})$ to denote the instance class that has fixed $(i_1,\ldots, i_{L'-1})$, $(i_{L'+L},\ldots, i_{H-2},s_{H-1})$, but different $(i_{L'},\ldots i_{L'+L-1})$. In addition, we use $I(i_1,\ldots, i_{L'-1})$ to denote the instance class that has fixed $(i_1,\ldots, i_{L'-1})$, but different $(i_{L'},\ldots i_{L'+L-1})$ and $(i_{L'+L},\ldots, i_{H-2},s_{H-1})$.

Since we have already fixed $\tau[0,k-1]\in E_{k-1}^\complement(L')$ here, $\psi(\tau[0,k-1])$ is also fixed (for all $I\in I(i_1,\ldots, i_{L'-1},i_{L'+L},\ldots,i_{H-2},s_{H-1})$). Also notice that we are considering the probability of $N$ episodes $\tau'[k-1:k]$. Therefore, we can follow Claim \ref{claim:claim_1} and define $p(i,j,h)$ for $0\le h\le L'+L-1$, which represents the probability of choosing action $j$ at node $i$ at level $h$ under the $k$-th deployment policy. In the $k$-th deployment, $\psi$ draws $N$ episodes, so the probability of executing $\psi$ to reach any state $s_{L'+L}$ at level $L'+L$ under instance $I=(i_1,\ldots,i_{H-2},s_{H-1})$ is 
\begin{align*}
&N p(s_0,i_1,0)p(i_1,i_2,1)\ldots p(i_{L'+L-2},i_{L'+L-1},L'+L-2)p(i_{L'+L-1},s_{L'+L},L'+L-1)\\
\le&N p(i_{L'-1},i_{L'},L'-1)\ldots p(i_{L'+L-2},i_{L'+L-1},L'+L-2)p(i_{L'+L-1},s_{L'+L},L'+L-1).
\end{align*}
Following the same step in Eq \eqref{eq:sum_n} by summing over $i_{L'},\ldots,i_{L'+L-1}$ and $s_{L'+L}$ gives us 
\begin{align*}
    &\sum_{I\in I(i_1,\ldots, i_{L'-1},i_{L'+L},\ldots,i_{H-2},s_{H-1})}\quad \sum_{\tau':\tau'[0,k-1]=\tau[0,k-1] \text{ and }\tau'[k-1,k] \text{ hit level $L'+L$}}P_{I,\psi(\tau[0,k-1])}(\tau'[k-1,k])\\
    \le& N.
\end{align*}


Now, we sum over all possible $(i_{L'+L},\ldots,i_{H-2},s_{H-1})$ and take the average. For any fixed $\tau[0,k-1] \in E_{k-1}^\complement(L')$ we have
\begin{align*}
    &\frac{1}{|I(i_1,\ldots, i_{L'-1})|}\sum_{I\in I(i_1,\ldots, i_{L'-1})}\quad \sum_{\tau':\tau'[0,k-1]=\tau[0,k-1] \text{ and }\tau'[k-1,k] \text{ hit level $L'+L$}}P_{I,\psi(\tau[0,k-1])}(\tau[k-1,k])\\
    =&\frac{1}{|I(i_1,\ldots, i_{L'-1})|}\sum_{i_{L'+L},\ldots, i_{H-2},s_{H-1}}\quad \sum_{I\in I(i_1,\ldots, i_{L'-1},i_{L'+L},\ldots, i_{H-2},s_{H-1})}\\
    &\quad \sum_{\tau':\tau'[0,k-1]=\tau[0,k-1] \text{ and }\tau'[k-1,k] \text{ hit level $L'+L$}}P_{I,\psi(\tau[0,k-1])}(\tau[k-1,k])\\
    \le&\frac{1}{|I(i_1,\ldots, i_{L'-1})|}\sum_{i_{L'+L},\ldots, i_{H-2},s_{H-1}}N\\
    =&\frac{d^{H-L'-L}}{d^{H-L'}}N\\
    =&\frac{1}{d^{L}}N.
\end{align*}



Moreover, summing over all $\tau[0,k-1] \in E_{k-1}^\complement(L')$, gives us $i_1,\ldots,i_{L'-1}$
\begin{align*}
    &P_{\cM,\psi}(E_k(L'+L)\cap E_{k-1}^\complement(L'))\\
    =&\frac{1}{|\cM|}\sum_{i_1,\ldots,i_{L'-1}}\sum_{I\in I(i_1,\ldots,i_{L'-1})}\sum_{\tau[0,k-1] \in E_{k-1}^\complement(L')}P_{I,\psi}(\tau[0,k-1])\\
    &\quad \sum_{\tau':\tau'[0,k-1]=\tau[0,k-1] \text{ and }\tau'[k-1,k] \text{ hit level $L'+L$}}P_{I,\psi(\tau[0,k-1])}(\tau'[k-1,k])\\
    =&\frac{1}{|\cM|}\sum_{\tau[0,k-1] \in E_{k-1}^\complement(L')}\sum_{ i_1,\ldots,i_{L'-1}}P_{i_1,\ldots,i_{L'-1},\psi}(\tau[0,k-1]) \\
    &\quad \sum_{I\in I(i_1,\ldots,i_{L'-1})}\sum_{\tau':\tau'[0,k-1]=\tau[0,k-1] \text{ and }\tau'[k-1,k] \text{ hit level $L'+L$}}P_{I,\psi(\tau[0,k-1])}(\tau'[k-1,k])\\
    \le&\frac{1}{|\cM|}\sum_{\tau[0,k-1] \in E_{k-1}^\complement(L')}\sum_{ i_1,\ldots,i_{L'-1}}P_{i_1,\ldots,i_{L'-1},\psi}(\tau[0,k-1])|I(i_1,\ldots, i_{L'-1})|\frac{N}{d^L} \\
    =&\frac{N}{d^L}\frac{|I(i_1,\ldots, i_{L'-1})|}{|\cM|}\sum_{\tau[0,k-1] \in E_{k-1}^\complement(L')}\sum_{ i_1,\ldots,i_{L'-1}}P_{i_1,\ldots,i_{L'-1},\psi}(\tau[0,k-1])\\
    =&\frac{N}{d^L}\sum_{\tau[0,k-1] \in E_{k-1}^\complement(L')}P_{\cM,\psi}(\tau[0,k-1]) \\
    =&\frac{N}{d^{L}} P_{\cM,\psi}(E_{k-1}^\complement(L')).
\end{align*}
In the second equality, we use $P_{i_1,\ldots,i_{L'-1,\psi}}$ because for any fixed $\tau[0,k-1]\in E_{k-1}^\complement(L')$ and all $I\in I(i_1,\ldots,i_{L'-1})$, $P_{i_1,\ldots,i_{L'-1,\psi}}(\tau[0,k-1])$ are the same.

Finally, we have
\begin{align*}
    P_{\cM,\psi}(E_k^\complement(L'+L))&\ge P_{\cM,\psi}(E_k^\complement(L'+L) \cap E^\complement_{k-1}(L'))\\
    &=P_{\cM,\psi}(E_{k-1}^\complement(L'))-P_{\cM,\psi}(E_k(L'+L) \cap E^\complement_{k-1}(L'))\\
    &\ge(1-\frac{N}{d^{L}})P(E^\complement_{k-1}(L'))\\
    &\ge(1-\frac{N}{d^{L-1}})P(E^\complement_{k-1}(L')).
\end{align*}

\paragraph{Proof of the Theorem}


If $KL\le H-2$, then applying Claim \ref{claim:claim_2} for $K-1$ times and applying Claim \ref{claim:claim_1} tells us
\begin{align*}
P_{\cM,\psi}(E_K^\complement(KL))=&P_{\cM,\psi}(E_K^\complement((K-1)L+L))\ge(1-\frac{N}{d^{L-1}})P(E^\complement_{K-1}((K-1)L))\\
\ge&\ldots\ge(1-\frac{N}{d^{L-1}})^{K-1}P(E^\complement_{1}(L))\ge(1-\frac{N}{d^{L-1}})^{K}.
\end{align*}

We can set $L=\lceil\log_d NH \rceil+1$ and $K\le \frac{H-2}{2 \lceil \log_d NH\rceil}$. Then for $H\ge 3$, we get $KL\le H-2$ and 
\begin{align*}
P_{\cM,\psi}(\text{does not hit any state at level $H-2$})&\ge (1-\frac{N}{d^{\lceil\log_d NH \rceil}})^{\frac{H-2}{2 \lceil \log_d NH\rceil}}\\
&\ge (1-\frac{N}{NH})^{\frac{H-2}{2 \lceil \log_d NH\rceil}}\\
&\ge (1-\frac{1}{H})^{H}\\
&\ge \frac{1}{e}.
\end{align*}

Let event $F$ denote the event (a set of length $KN$ episodes trajectories) that any state at level $H-2$ is not hit. Then we have $P_{\cM,\psi}(F)\ge 1-\frac{1}{e}$. We use $I(i_1,\ldots, i_{H-2})$ to denote the instance class that has fixed core states $(i_1,\ldots, i_{H-2})$ but different optimal states $s_{H-1}$. 

Consider any fixed $\tau\in F$. Similar as the proof in the prior claims, by the property of deterministic algorithm, we can define $p(i,j,h)$ for $h=H-2$, which represents the probability of
the output policy $\psi_\tau(K+1)$ under trajectory $\tau$ to choose action $j$ at node $i$ at level $H-2$. Then we have
\begin{align*}
&\sum_{I=(i_1,\ldots,i_{H-2},s_{H-1})\in I(i_1,\ldots, i_{H-2})}P_{I,\psi}(\psi_\tau(K+1)\text{ chooses optimal state)}\\
=&\sum_{I=(i_1,\ldots,i_{H-2},s_{H-1})\in I(i_1,\ldots, i_{H-2})}P_{I,\psi}((\psi_\tau(K+1))(i_{H-2})=s_{H-1})\\
=&\sum_{I=(i_1,\ldots,i_{H-2},s_{H-1})\in I(i_1,\ldots, i_{H-2})}p(i_{H-2},s_{H-1},H-2)\\
=&\sum_{s_{H-1}}p(i_{H-2},s_{H-1},H-2)\\
=&1.
\end{align*}

Summing over $\tau\in F$ gives us
\begin{align*}
&\sum_{I=(i_1,\ldots,i_{H-2},s_{H-1})\in I(i_1,\ldots, i_{H-2})}P_{I,\psi}(F\cap \text{ the output policy chooses optimal state})\\
=&\sum_{I=(i_1,\ldots,i_{H-2},s_{H-1})\in I(i_1,\ldots, i_{H-2})}\sum_{\tau\in F}P_{I,\psi}(\tau)P_{I,\psi}(\psi_\tau(K+1)\text{ chooses optimal state)}\\
=&\sum_{I=(i_1,\ldots,i_{H-2},s_{H-1})\in I(i_1,\ldots, i_{H-2})}\sum_{\tau\in F}P_{i_1,\ldots,i_{H-2},\psi}(\tau)P_{I,\psi}(\psi_\tau(K+1)\text{ chooses optimal state)}\\
=&\sum_{\tau\in F}P_{i_1,\ldots,i_{H-2},\psi}(\tau)\sum_{I=(i_1,\ldots,i_{H-2},s_{H-1})\in I(i_1,\ldots, i_{H-2})}P_{I,\psi}(\psi_\tau(K+1)\text{ chooses optimal state)}\\
=&\sum_{\tau\in F}P_{i_1,\ldots,i_{H-2},\psi}(\tau)\\
=&P_{i_1,\ldots,i_{H-2},\psi}(F).
\end{align*}
In the second equality, we notice that for all instance $I\in I=(i_1,\ldots,i_{H-2},s_{H-1})$, $P_{I,\psi}(\tau)$ are the same, so this probability distribution essentially depends on $i_1,\ldots,i_{H-2}$. In the third inequality, we change the order of the summation.

Finally, summing over $i_1,\ldots,i_{H-2}$ and taking average yields that 
\begin{align*}
&P_{\cM,\psi}(F\cap \text{ the output policy chooses optimal state})\\
=&\frac{1}{d^{H-1}}\sum_{i_1,\ldots,i_{H-2}}\sum_{I=(i_1,\ldots,i_{H-2},s_{H-1})\in I(i_1,\ldots, i_{H-2})}P_{I,\psi}(F\cap \text{ the output policy chooses optimal state})\\
=&\frac{1}{d^{H-1}}\sum_{i_1,\ldots,i_{H-2}}P_{i_1,\ldots,i_{H-2},\psi}(F)\\
=&\frac{1}{d^{H-1}}\sum_{i_1,\ldots,i_{H-2}}\frac{1}{d}\sum_{s_{H-1}}P_{I=(i_1,\ldots,i_{H-2},s_{H-1}),\psi}(F)\tag{$P_{I,\psi}(F)$ does not depend on the optimal state}\\
=&\frac{1}{d}P_{\cM,\psi}(F)
\end{align*}

Therefore, we get the probability of not choosing the optimal state is
\begin{align*}
&P_{\cM,\psi}(\text{ the output policy does not choose the optimal state})\\
\ge &P_{\cM,\psi}(F\cap \text{ the output policy does not choose the optimal state})\\
= &P_{\cM,\psi}(F)-P_{\cM,\psi}(F\cap \text{ the output policy chooses the optimal state})\\
=&\frac{d-1}{d}P_{\cM,\psi}(F)\\
\ge&\frac{1}{2}\cdot \frac{1}{e}.
\end{align*}

From the construction, we know that any policy that does not choose optimal state (thus also does not choose the optimal action associated with the optimal state) is $\epsilon$ sub-optimal. This implies that with probability at least $\frac{1}{2e}$, the output policy is at least $\epsilon$ sub-optimal.
\end{proof}

\section{Deployment-Efficient RL with Deterministic Policies and given Reward Function}\label{appx:RB_BE_with_Deter_Policies}

\subsection{Additional Notations}
In the appendix, we will frequently consider the MDP truncated at $\tilde{h}\leq H$, and we will use:
\begin{align*}
    V^\pi_h(s|\tilde{h})=\EE[\sum_{h'=h}^{\tilde{h}} r_{h'}(s_{h'},a_{h'})|s_h=s,\pi], \quad Q^\pi_h(s,a|\tilde{h})=\EE[\sum_{h'=h}^{\tilde{h}} r_{h'}(s_{h'},a_{h'})|s_h=s,a_h=a,\pi]
\end{align*}
to denote the value function in truncated MDP for arbitrary $h\leq \tilde{h}$, and also extend the definition in Section \ref{sec:preliminary} to $V^*_h(\cdot|\tilde{h}), Q^*_h(\cdot,\cdot|\tilde{h})$, $\pi^*_{|\tdh}$ for optimal policy setting and $V^*_h(\cdot,r|\tilde{h}), Q^*_h(\cdot,\cdot,r|\tilde{h})$, $\pi^*_{r|h}$ for reward-free setting.
\subsection{Auxiliary Lemma}
\begin{lemma}[Elliptical Potential Lemma; Lemma 26 of \citet{agarwal2020flambe}]\label{lem:elliptical_potential_lemma}
    Consider a sequence of $d \times d$ positive semi-definite matrices $X_1,..., X_T$ with $\max_t Tr(X_t) \leq 1$
     and define $M_0 = \lambda I,..., M_t = M_{t-1} + X_t$. 
     Then
    \begin{align*}
        \sum_{t=1}^T Tr(X_tM^{-1}_{t-1}) \leq (1+1/\lambda)d\log(1+T/d).
    \end{align*}
\end{lemma}

\begin{lemma}[\citet{NIPS2011_e1d5be1c}]\label{lem:norm_and_det}
    Suppose $\A,\B\in\mR^{d\times d}$ are two positive definite matrices satisfying $\A\succeq \B$, then for any $x\in\mR^d$, we have:
    \begin{align*}
        \|x\|^2_\A \leq \|x\|^2_\B \frac{\det(\A)}{\det(\B)}.
    \end{align*}
\end{lemma}
Next, we prove a lemma to bridge between trace and determinant, which is crucial to prove our key technique in Lemma \ref{lem:finite_sample_elliptical_potential_lemma}. 
\begin{restatable}{lemma}{TraceDetLogDetLemma}[Bridge between Trace and Determinant]\label{lem:TraceDetLogDetLemma}
    Consider a sequence of matrices $\A_0,\A_{N},...,\A_{(K-1)N}$ with $\A_0=I$ and $\A_{kN}=\A_{(k-1)N}+\Phi_{k-1}$, where $\Phi_{k-1}=\sum_{t=(k-1)N+1}^{kN}\phi_t\phi_t\trans$. We have
    \begin{align*}
        Tr(\A^{-1}_{(k-1)N}\Phi_{k-1})\leq \frac{\det(\A_{kN})}{\det(\A_{(k-1)N})}\log\frac{\det(\A_{kN})}{\det(\A_{(k-1)N})}.
    \end{align*}
\end{restatable}
\begin{proof}
    Consider a more general case, given matrix $\Y\succeq I$, we have the following inequality 
    \begin{align*}
        Tr(\I-\Y^{-1})\leq \log \det(\Y) \leq Tr(\Y-\I).
    \end{align*}
    By replacing $\Y$ with $\I+\A^{-1}\X$ in the above inequality, we have:
    \begin{align*}
        Tr((\A+\X)^{-1}\X)=&Tr((\I+\A^{-1}\X)^{-1}(\A^{-1}\X))=Tr((\I+\A^{-1}\X)^{-1}(\I+\A^{-1}\X-\I)\\
        =&Tr(\I-(\I+\A^{-1}\X)^{-1})\\
        \leq&\log \det(\I+\A^{-1}\X)= \log \frac{\det(\A+\X)}{\det(\A)}.
    \end{align*}
    By assigning $\A=\A_{(k-1)N}$ and $\X=\Phi_{k-1}$, and applying Lemma \ref{lem:norm_and_det}, we have:
    \begin{align*}
        Tr(\A_{(k-1)N}^{-1}\Phi_{k-1})=&\sum_{t=(k-1)N+1}^{kN}\|\phi_t\|^2_{\A_{(k-1)N}^{-1}}\\
        \leq&\sum_{t=(k-1)N+1}^{kN}\|\phi_t\|^2_{\A_{kN}^{-1}}\frac{\det\A_{kN}}{\det(\A_{(k-1)N})}\\
        =& Tr(\A_{kN}^{-1}\Phi_{k-1})\frac{\det\A_{kN}}{\det(\A_{(k-1)N})}\\
        \leq& \frac{\det\A_{kN}}{\det(\A_{(k-1)N})}\log\frac{\det\A_{kN}}{\det(\A_{(k-1)N})}
    \end{align*}
    which finished the proof.
\end{proof}
    
\ThmFiniteViolation*
\begin{proof}
    Suppose we have $Tr(\A^{-1}_{(k-1)N}\Phi_{k-1})\geq N\epsilon$, by applying Lemma \ref{lem:TraceDetLogDetLemma} we must have:
    \begin{align*}
        N\epsilon \leq& \frac{\det(\A_{kN})}{\det(\A_{(k-1)N})} \log \frac{\det(\A_{kN})}{\det(\A_{(k-1)N})} \leq \frac{\det(\A_{kN})}{\det(\A_{(k-1)N})}\log(\det(\A_{kN}))\\
        \leq& d\frac{\det(\A_{kN})}{\det(\A_{(k-1)N})} \log(1+KN/d) \tag{$\det(A) \leq (Tr(A)/d)^{d}$}
    \end{align*}
    which implies that,
    \begin{align*}
        \frac{N\epsilon}{d\log(1+KN/d)}\leq \frac{\det(\A_{kN})}{\det(\A_{(k-1)N})}
    \end{align*}
    Therefore,
    \begin{align*}
        |\mathcal{K}^+| \log \frac{N\epsilon}{d\log(1+KN/d)} \leq& \sum_{k\in\mathcal{K}} \log \frac{\det(\A_{kN})}{\det(\A_{(k-1)N})} \leq \sum_{k=1}^K \log \frac{\det(\A_{kN})}{\det(\A_{(k-1)N})}\\
         =&\log \frac{\det(\A_{KN})}{\det(\A_{0})} \leq d\log(1+KN/d)
    \end{align*}
    which implies that, conditioning on $N \geq \frac{d}{\epsilon}\log(1+KN/d)$, we have:
    $$
    |\mathcal{K}^+| \leq d\frac{\log(1+KN/d)}{\log(\frac{N\epsilon}{d\log(1+KN/d)})}
    $$
    Now, we are interested in find the mimimum $N$, under the constraint that $|\mathcal{K}^+| \leq c_K d$. To solve this problem, we first choose an arbitrary $p \leq c_K$, and find a $N$ such that,
    \begin{align*}
        \frac{\log(1+KN/d)}{\log(\frac{N\epsilon}{d\log(1+KN/d)})} \leq p
    \end{align*}
    In order to guarantee the above, we need:
    \begin{align*}
        N\epsilon \geq d\log(1+KN/d),\quad (\frac{N\epsilon}{d\log(1+KN/d)})^p \geq 1+KN/d
    \end{align*}
    The first constraint can be satisfied easily with $N \geq c_1 \frac{d}{\epsilon}\log\frac{dH}{\epsilon}$ for some constant $c_1$. Since usually $KN/d > 1$, the second constraint can be directly satisfied if:
    \begin{align*}
        (\frac{N\epsilon}{d\log(1+KN/d)})^p \geq 2KN/d
    \end{align*}
    Recall $K=c_KdH+1$, it can be satisfied by choosing
    \begin{align}
        N \geq c_{2}\Big(c_K\frac{Hd^{p}}{\epsilon^{p}}\log^{p}(\frac{Hd}{\epsilon})\Big)^{\frac{1}{p-1}} \label{eq:upper_bound_N}
    \end{align}
    where $c_2$ is an absolute constant.
    Therefore, we can find an absolute number $c$ such that,
    \begin{align*}
        N = c\Big(c_K\frac{Hd^{p}}{\epsilon^{p}}\log^{p}(\frac{Hd}{\epsilon})\Big)^{\frac{1}{p-1}}\geq \max\{c_1\frac{d}{\epsilon}\log(\frac{d}{\epsilon}), c_{2}\Big(c_K\frac{Hd^{p}}{\epsilon^{p}}\log^{p}(\frac{Hd}{\epsilon})\Big)^{\frac{1}{p-1}}\} 
    \end{align*}
    to make sure that
    \begin{align*}
        |\mathcal{K^+}| \leq pd
    \end{align*}
    Since in Eq.\eqref{eq:upper_bound_N}, it's required that $1/(p-1) < \infty$, we should constraint that $p > 1$ and therefore, $c_K \geq 2$. Because the dependence of $d,H,\frac{1}{\epsilon}, \log\frac{dH}{\epsilon}$ are decreasing as $p$ increases, by assigning $p=c_K$ and $1<p\leq c_K$, $N$ will be minimized when $p=c_K$. Then, we finished the proof.
\end{proof}
    
\subsection{Analysis for Algorithms}
Next, we will use the above lemma to bound the difference between $J(\pi_K)$ and $J(\pi^*)$.
We first prove a lemma similar to Lemma B.3 in \citep{jin2019provably} and Lemma A.1 in \citep{wang2020rewardfree}.
\begin{lemma}[Concentration Lemma]\label{lem:deployment_efficient_concentration}
    We use $\mathcal{E}_1$ to denote the event that, when running Algorithm \ref{alg:DERL_Deterministic_Policy_Algorithm}, the following inequality holds for all $k\in[K]$ and $h\in[h_k]$ and arbitrary $V^k_{h+1}$ occurs in Alg \ref{alg:DERL_Deterministic_Policy_Algorithm}.
    \begin{align*}
        \Big\|\sum_{\tau=1}^{k-1}\sum_{n=1}^N\phi_h^{\tau n}\Big(V^k_{h+1}(s_{h+1}^{\tau n})-\sum_{s'\in \cS} P_h(s'|s_{h}^{\tau n},a_{h}^{\tau n})V^k_{h+1}(s')\Big)\Big\|_{(\Lambda^k_h)^{-1}} \leq c\cdot dH\sqrt{\log(dKNH/\delta)}
    \end{align*}
    Under Assumption \ref{assump:linear_MDP}, there exists some absolute constant $c\geq 0$, such that $P(\mathcal{E}_1)\geq 1-\delta/2$.
\end{lemma}
\begin{proof}
    The proof is almost identical to Lemma B.3 in \citep{jin2019provably}, so we omit it here. The only difference is that we have an inner summation from $n=1$ to $N$ and we truncate the horizon at $h_k$ in iteration $k$.
\end{proof}

\begin{lemma}[Overestimation]\label{lem:overestimation_LBL_BatchLSVI_UCB}
    On the event $\mathcal{E}_1$ in Lemma \ref{lem:deployment_efficient_concentration}, which holds with probability $1-\delta / 2$, for all $k\in[K]$ and $n\in[N]$,
    \begin{align*}
        V_1^*(s_1^{kn}|h_k) \leq V_1^k(s_1^{kn})
    \end{align*}
    where recall that $V_1^k$ is the function computed at iteration $k$ in Alg.\ref{alg:DERL_Deterministic_Policy_Algorithm} and $V_1^*(\cdot|h_k)=\EE[\sum_{h=1}^{h_k}r_h(s_h,a_h)|\pi^*_{[1:h_k]}]$ denote the optimal value function in the MDP truncated at layer $h_k$ and $\pi^*_{[1:h_k]}$ is the optimal policy in the truncated MDP.

    Besides, we also have:
    \begin{align*}
        \EE_{s_1\sim d_1}[V_1^*(s_1|h_k)-V^{\pi_k}(s_1|h_k)] \leq 2\beta\EE_{s_1,a_1,...,s_{h_k},a_{h_k}\sim \pi_k}[\sum_{h=1}^{h_k} \|\phi(s_h,a_h)\|_{(\Lambda^k_h)^{-1}}]
    \end{align*}
\end{lemma}
\begin{proof}
    First of all, by applying Lemma \ref{lem:deployment_efficient_concentration} above, after a similar discussion to the proof of Lemma 3.1 in \citep{wang2020rewardfree}, we can show that 
    $$
    |\phi(s,a)\trans w^k_h -\sum_{s'\in \cS}P_h(s'|s,a)V^k_{h+1}(s')| \leq \beta \|\phi(s,a)\|_{(\Lambda^k_h)^{-1}},\quad\forall s\in\cS,a\in\cA, h\in[h_k]
    $$
    and the overestimation
    \begin{align*}
        V^*_h(s|h_k) \leq V^k_h(s),\quad\forall s\in\cS,h\in[h_k]
    \end{align*}
    As a result,
    \begin{align*}
        &\EE_{s_1\sim d_1}[V_1^*(s_1|h_k)-V^{\pi_k}(s_1|h_k)]\\
        \leq&\EE_{s_1\sim d_1}[V_1^k(s_1)-V^{\pi_k}(s_1|h_k)]\\
        =&\EE_{s_1\sim d_1,a_1\sim\pi_k}[Q_1^k(s_1,a_1)-Q^{\pi_k}(s_1,a_1|h_k)]\\
        =&\EE_{s_1\sim d_1,a_1\sim\pi_k}[\min\{(w_1^k)\trans\phi(s_1,a_1)+r_1(s_1,a_1)+u_1^k(s_1,a_1), H\} -r_1(s_1,a_1)-\sum_{s_2\in\cS}P_1(s_2|s_1,a_1)V^{\pi_k}_2(s_2|h_k))]\\
        = &\EE_{s_1\sim d_1,a_1\sim\pi_k}[\min\{(w_1^k)\trans\phi(s_1,a_1), H-r_1(s_1,a_1)-u_1^k(s_1,a_1)\}-\sum_{s_2\in\cS}P_1(s_2|s_1,a_1)V^k_2(s_2)]\\
        &+\EE_{s_1\sim d_1,a_1\sim \pi_k}[\sum_{s_2\in\cS}P_1(s_2|s_1,a_1)V^k_2(s_2)-\sum_{s_2\in\cS}P_1(s_2|s_1,a_1)V^{\pi_k}_2(s_2|h_k))]+\EE_{s\sim\mu}[u^k_1(s_1,a_1)]\\
        \leq& \EE_{s_1\sim d_1,a_1\sim\pi_k}[\sum_{s_2\in\cS}P_1(s_2|s_1,a_1)V_2^k(s_2)-\sum_{s_2\in\cS}P_1(s_2|s_1,a_1)V^{\pi_k}_2(s_2|h_k))]+2\EE_{s\sim\mu}[u^k_1(s_1,a_1)]\\
        =& \EE_{s_1\sim d_1,a_1,s_2,a_2\sim \pi_k}[V_2^k(s_2)-V^{\pi_k}_2(s_2|h_k)]+2\EE_{s\sim\mu}[u^k_1(s_1,a_1)]\\
        \leq & ...\\
        \leq& 2\EE_{s_1\sim d_1,a_1,...,s_{h_k},a_{h_k}\sim\pi_k}[\sum_{h=1}^{h_k} u^k_h(s_h,a_h)]\\
        \leq&2\beta\EE_{s_1\sim d_1,a_1,...,s_{h_k},a_{h_k}\sim\pi_k}[\sum_{h=1}^{h_k} \|\phi(s_h,a_h)\|_{(\Lambda^k_h)^{-1}}]
    \end{align*}
    where in the second inequality, we use the following fact
    \begin{align*}
        &\min\{(w_1^k)\trans\phi(s_1,a_1), H-r_1(s_1,a_1)-u_1^k(s_1,a_1)\}-\sum_{s_2\in\cS}P_1(s_2|s_1,a_1)V^k_2(s_2)\\
        =& \min\{(w_1^k)\trans\phi(s_1,a_1)-\sum_{s_2\in\cS}P_1(s_2|s_1,a_1)V^k_2(s_2), H-r_1^k(s_1,a_1)-u_1^k(s_1,a_1)-\sum_{s_2\in\cS}P_1(s_2|s_1,a_1)V^k_2(s_2)\}\\
        \leq& \min\{\beta\|\phi(s_1,a_1)\|_{(\Lambda_1^k)^{-1}}, H\}=u_1^k(s_1,a_1)
    \end{align*}
\end{proof}

Now we are ready to prove the following theorem restated from Theorem \ref{thm:deployment_complexity_given_reward} in a more detailed version, where we include the guarantees during the execution of the algorithm.
\begin{restatable}{theorem}{ThmKnownRewardConverge}[Deployment Complexity]\label{thm:deployment_complexity_given_reward_detailed_version}
    For arbitrary $\epsilon,\delta > 0$, and arbitrary $c_K \geq 2$, as long as $N \geq c \Big(c_K\frac{H^{4c_K+1}d^{3c_K}}{\epsilon^{2c_K}}\log^{2c_K}(\frac{Hd}{\delta\epsilon})\Big)^{\frac{1}{c_K-1}}$, 
    where $c$ is an absolute constant and independent with $c_K, d,H,\epsilon,\delta$, by choosing
    \begin{align}
        K=c_KdH+1.
    \end{align} 
    Algorithm \ref{alg:DERL_Deterministic_Policy_Algorithm} will terminate at iteration $k_H \leq K$ and return us a policy $\pi^{k_H}$, and with probability $1-\delta$,  (1) $\EE_{s_1\sim d_1}[V_1^*(s_1)-V_1^{\pi^{k_H}}(s_1)] \leq \epsilon$. 
    (2) for each $h\in[H-1]$, there exists an iteration $k_h$, such that $h_{k_h}=h$ but $h_{k_h+1}=h+1$, and $\pi_{k_h}$ is an $\epsilon$-optimal policy for the MDP truncated at step $h$;
\end{restatable}
\begin{proof}
    As stated in the theorem, we use $k_h$ to denote the number of deployment after which the algorithm switch the exploration from layer $h$ to layer $h+1$, i.e. $h_{k_h}=h$ and $h_{k_h+1}=h+1$. According to the definition and the algorithm, we must have $\Delta_{k_h}\leq \frac{\epsilon h_{k_h}}{2H}$, and for arbitrary $k_{h-1}+1\leq k \leq k_{h}-1$, $\Delta_k \geq \frac{\epsilon h_k}{2H}$ (if $k_{h-1}+1>k_{h}-1$, then it means $\Delta_{k_{h-1}+1}$ is small enough and the algorithm directly switch the exploration to the next layer, and we can skip the discussion below ).
    Therefore, for arbitrary $k_{h-1}+1\leq k \leq k_{h}-1$, during the $k$-th deployment, there exists $h\in[h_k]$, such that,
    \begin{align*}
        \frac{\epsilon}{2H} \leq \frac{\Delta_k}{h_k} \leq \frac{2\beta}{N}\sum_{n=1}^N \|\phi(s_{h}^{kn},a_{h}^{kn})\|_{(\Lambda^k_{h})^{-1}} \leq 2\beta\sqrt{\frac{1}{N}\sum_{n=1}^N \|\phi(s_{h}^{kn},a_{h}^{kn})\|^2_{(\Lambda^k_{h})^{-1}}}
    \end{align*}
    where the second inequality is because the average is less than the maximum. The above implies that
    \begin{align}
        \frac{1}{N}\sum_{n=1}^N \|\phi(s_{h}^{kn},a_{h}^{kn})\|^2_{(\Lambda^k_{h})^{-1}} = \frac{1}{N}Tr\Big((\Lambda^k_{h})^{-1}\Big(\sum_{n=1}^N \phi(s_{h}^{kn},a_{h}^{kn})\phi(s_{h}^{kn},a_{h}^{kn})\trans\Big)\Big) \geq \frac{\epsilon^2}{16H^2\beta^2} \label{eq:violation_with_reward_layer_by_layer}
    \end{align}
    According to Lemma \ref{lem:finite_sample_elliptical_potential_lemma}, there exists constant $c,c'$, such that by choosing $N$ according to Eq.\eqref{eq:N_constraint_1} below, the event in Eq.\eqref{eq:violation_with_reward_layer_by_layer} will not happen more than $dc_K$ times at each layer $h\in[h_k]$.
    \begin{align}
        N\geq c\Big(c_K\frac{H^{4c_K+1}d^{3c_K}}{\epsilon^{2c_K}}\log^{2c_K}(\frac{Hd}{\epsilon\delta})\Big)^{\frac{1}{c_K-1}} \geq c'\Big(c_K\frac{H^{2c_K+1}d^{c_K}\beta^{2c_K}}{\epsilon^{2c_K}}\log^{c_K}(\frac{Hd\beta}{\epsilon})\Big)^{\frac{1}{c_K-1}} \label{eq:N_constraint_1}
    \end{align} 
    Recall that $\epsilon < 1$ and the covariance matrices in each layer is initialized by $I_{d\times d}$. Therefore, at the first deployment, although the computation of $\pi^1$ does not consider the layers $h\geq 2$, Eq.\eqref{eq:violation_with_reward_layer_by_layer} happens in each layer $h\in[H]$. We use $\zeta(k,j)$ to denote the total number of times events in Eq.\eqref{eq:violation_with_reward_layer_by_layer} happens for layer $j$ previous to deployment $k$, as a result, 
    \begin{align*}
        k_h \leq \sum_{j=1}^h \zeta(k_h,j) -(h-1)+h \leq c_Kdh+1,\quad \forall h \in [H]
    \end{align*}
    where we minus $h-1$ because such event must happen at the first deployment for each $h\in[H]$ and we should remove the repeated computation; and we add another $h$ back is because there are $h$ times we waste the samples (i.e. for those $k$ such that $\Delta_k < \frac{\epsilon h_k}{2H}$).
    Therefore, we must have $k_H \leq c_KdH+1=K$.
    
    Moreover, because at iteration $k=k_h$, we have $\Delta_{k_h} \leq \epsilon/2$, according to Hoeffding inequality, with probability $1-\delta/2$, for each deployment $k$, we must have:
    \begin{align}
        \EE_{s_1,a_1,...,s_{h_k},a_{h_k}\sim \pi_{k}}[2\beta\sum_{h=1}^{h_k} \|\phi(s_h,a_h)\|_{(\Lambda^k_h)^{-1}}] \leq \Delta_{k} + 2\beta H\sqrt{\frac{1}{2N}\log(\frac{K}{\delta})}\label{eq:given_rew_upper_bound_bonus}
    \end{align}
    Therefore, by choosing 
    \begin{align}
        N \geq \frac{8\beta^2H^2}{\epsilon^2}\log(\frac{K}{\delta})=O(\frac{d^2H^4}{\epsilon^2}\log^2(\frac{K}{\delta}))\label{eq:N_constraint_2}
    \end{align}
    we must have,
    \begin{align*}
        \EE_{s_1,a_1,...,s_{h},a_{h}\sim \pi_{k_h}}[2\beta\sum_{h'=1}^{h} \|\phi(s_{h'},a_{h'})\|_{(\Lambda^k_{h'})^{-1}}] \leq \Delta_{k_h} + \frac{\epsilon}{2} = \epsilon,\quad\forall h \in [H]
    \end{align*}
    Therefore, after a combination of Eq.\eqref{eq:N_constraint_1} and Eq.\eqref{eq:N_constraint_2}, we can conclude that, for arbitrary $c_K \geq 2$ , there exists absolute constant $c$, such that by choosing
    $$
    N\geq c\Big(c_K\frac{H^{4c_K+1}d^{3c_K}}{\epsilon^{2c_K}}\log^{2c_K}(\frac{Hd}{\epsilon\delta})\Big)^{\frac{1}{c_K-1}}
    $$
    the algorithm will stop at $k_H\leq K$, and with probability $1-\delta$ (on the event of $\mathcal{E}_1$ in \ref{lem:deployment_efficient_concentration} and the Hoeffding inequality above), we must have:
    \begin{align*}
        \EE_{s_1\sim d_1}[V_1^*(s_1)-V^{\pi_k}(s_1)] \leq \EE_{s_1,a_1,...,s_{h_k},a_{h_k}\sim \pi_{k+1}}[2\beta\sum_{h=1}^{h_k} \|\phi(s_h,a_h)\|_{(\Lambda^k_h)^{-1}}] \leq \epsilon
    \end{align*}
    and an additional benefits that for each $h\in[H-1]$, $\pi_{k_h}$ is an $\epsilon$-optimal policy at the MDP truncated at $h$ step, or equivalently,
    \begin{align}
        \EE_{s_1\sim d_1}[V_1^*(s_1|h)-V_1^{\pi_{k_h}}(s_1|h)] \leq \epsilon.
    \end{align}\label{eq:given_reward_optimal_policy_for_each_step}
\end{proof}
\subsection{Additional Safety Guarantee Brought with Layer-by-Layer Strategy}\label{appx:additional_safety_guarantee}
The layer-by-layer strategy brings another advantage that, if we finish the exploration of the first $h$ layers, based on the samples collected so far, we can obtain a policy $\hat\pi_{|h}$, which is an $\epsilon$-optimal in the MDP truncated at step $h$, or equivalently:
\begin{align*}
    J(\pi^*) - J(\hat\pi_{|h}) \geq H - h + O(\epsilon),\quad\forall h \in [H]
\end{align*}

We  formally state these guarantees in Theorem \ref{thm:deployment_complexity_given_reward_detailed_version} (a detailed version of Theorem \ref{thm:deployment_complexity_given_reward}), Theorem \ref{thm:deployment_complexity_reward_free_detailed_version} and Theorem \ref{thm:arbitary_policy_formal} (the formal version of Theorem \ref{thm:arbitary_policy_informal}). 
Such a property may be valuable in certain application scenarios. 
For example, in ``Safe DE-RL'', which we will discuss  in Appendix \ref{appx:discussion_on_Extended_DERL}, $\hat\pi_{|h}$ can be used as the pessimistic policy in Algorithm \ref{alg:Mixture_Policy_Strategy} and guarantee the monotonic policy improvement criterion.
Besides, in some real-world settings, we may hope to maintain a sub-optimal but gradually improving policy before we complete the execution of the entire algorithm.

If we replace Line 7-8 in LSVI-UCB (Algorithm 1) in \citet{jin2019provably} with Line 13-18 in our Algorithm \ref{alg:DERL_Deterministic_Policy_Algorithm}, the similar analysis can be done based on Lemma \ref{lem:finite_sample_elliptical_potential_lemma}, and the same $\Theta(dH)$ deployment complexity can be derived. 
However, the direct extension based on LSVI-UCB does not have the above safety guarantee. It is only guaranteed to return a near-optimal policy after $K=\Theta(dH)$ deployments, but if we interrupt the algorithm after some $k < K$ deployments, there is no guarantee about what the best possible policy would be based on the data collected so far.

\section{Reward-Free Deployment-Efficient RL with Deterministic Policies}\label{appx:RF_BE_with_Deter_Policies}
\subsection{Algorithm}
Similar to other algorithms in reward-free setting \citep{wang2020rewardfree, jin2020rewardfree}, our algorithm includes an ``Exploration Phase'' to uniformly explore the entire MDP, and a ``Planning Phase'' to return near-optimal policy given an arbitrary reward function. The crucial part is to collect a well-covered dataset in the online ``exploration phase'', which is sufficient for the batch RL algorithm \citep{antos2008learning, munos2008finite, chen2019information} in the offline ``planning phase'' to work.

Our algorithm in Alg.\ref{alg:DE_rl_layer_by_layer_reward_free_exploration} and Alg.\ref{alg:DE_rl_layer_by_layer_reward_free_planning} is based on \citep{wang2020rewardfree} and the layer-by-layer strategy. 
The main difference with Algorithm \ref{alg:DERL_Deterministic_Policy_Algorithm} is in two-folds. First, similar to \citep{wang2020rewardfree}, we replace the reward function with $1/H$ of the bonus term. Secondly, we use a smaller threshold for $\Delta_k$ comparing with Algorithm \ref{alg:DERL_Deterministic_Policy_Algorithm}.
\begin{algorithm}[ht]
    \textbf{Input}: Failure probability $\delta>0$, and target accuracy $\epsilon > 0$, $\beta\gets c_\beta \cdot dH\sqrt{\log(dH\delta^{-1}\epsilon^{-1})}$ for some $c_\beta>0$, total number of deployments $K$, batch size $N$\\
    Initialize $h_1=1$\\
    $D_1=\{\},D_2=\{\},...,D_H=\{\}$\\
    \For{$k=1,2,...,K$}{
        $Q^k_{h_k+1}(\cdot,\cdot)\gets 0$ and $V^k_{h_k+1}(\cdot)=0$\\
        \For{$h=h_k,h_k-1,...,1$}{
            $\Lambda^k_h \gets I+\sum_{\tau=1}^k\sum_{n=1}^N \phi_h^{\tau n}(\phi_h^{\tau n})\trans$\\
            $u_h^k(\cdot,\cdot)\gets \min\{\beta\cdot \sqrt{\phi(\cdot,\cdot)\trans(\Lambda^k_h)^{-1}\phi(\cdot,\cdot)}, H\}$\\
            Define the exploration-driven reward function $r^k_h(\cdot,\cdot)\gets u^k_h(\cdot,\cdot)/H$\\
            $w^k_h\gets (\Lambda^k_h)^{-1}\sum_{\tau=1}^{k-1}\sum_{n=1}^N\phi_h^{\tau n}\cdot V^k_{h+1}(s^{\tau n}_{h+1})$\\
            $Q^k_h(\cdot,\cdot)\gets \min\{(w^k_h)\trans\phi(\cdot,\cdot)+r^k_h(\cdot,\cdot)+u^k_h(\cdot,\cdot), H\}$ and $V^k_h(\cdot)=\max_{a\in\cA}Q^k_h(\cdot,a)$\\
            $\pi^k_h(\cdot)\gets\arg\max_{a\in\cA}Q^k_h(\cdot,a)$
        }
        Define $\pi^k = \pi^k_1 \circ \pi^k_2 \circ ... \pi^k_{h_k}\circ \mathrm{unif}_{[h_k+1:H]}$\\
        \For{$n=1,...,N$}{
            Receive initial state $s_1^{kn}\sim d_1$\\
            \For{$h=1,2,...,H$}{
                Take action $a^{kn}_h\gets \pi^k(s_h^{kn})$ and observe $s_{h+1}^{kn}\sim P_h(s^k_h, a_h^k)$\\
                $D_h = D_h \bigcup \{(s_h^{kn},a_h^{kn})\}$
            }
        }
        Compute $\Delta_k \gets \frac{2\beta}{N}\sum_{n=1}^N\sum_{h=1}^{h_k} \sqrt{\phi(s_h^{kn},a_h^{kn})\trans (\Lambda_h^k)^{-1} \phi(s_h^{kn},a_h^{kn})}$.\\
        \If{$\Delta_k < \frac{\epsilon h_k}{(4H+2)H}$}{
            \lIf{$h_k = H$}{
                \Return $D=\{D_1,D_2,...,D_H\}$
            }
            \lElse{
                $h_k \gets h_k + 1$
            }
        }
    }
    \caption{Reward-Free DE-RL with Deterministic Policies in Linear MDPs: Exploration Phase}\label{alg:DE_rl_layer_by_layer_reward_free_exploration}
\end{algorithm}

\begin{algorithm}[ht]
    \textbf{Input}: Horizon length $\tdh$; Dataset $\mathcal{D}=\{(s_h^{kn},a_h^{kn})_{k,n,h\in[K]\times[N]\times[\tdh]}\}$, reward function $r=\{r_h\}_{h\in[\tdh]}$\\
    $Q_{\tdh+1}(\cdot,\cdot)\gets 0$ and $V_{\tdh+1}(\cdot)\gets 0$\\
    \For{$h=\tdh,\tdh-1,...,1$}{
        $\Lambda_h\gets I+\sum_{\tau=1}^K\sum_{n=1}^N\phi(s_h^{\tau n}, a_h^{\tau n})\phi(s_h^{\tau n}, a_h^{\tau n})\trans$\\
        Let $u_h^{plan}(\cdot,\cdot)\gets \min\{\beta \sqrt{\phi(\cdot,\cdot)\trans(\Lambda_h)^{-1}\phi(\cdot,\cdot)}, \tdh\}$\\
        $w_h\gets (\Lambda_h)^{-1}\sum_{\tau=1}^K\phi(s_h^{\tau n},a_h^{\tau n})\cdot V_{h+1}(s_{h+1}^{\tau n}, a)$\\
        $Q_h(\cdot,\cdot)\gets \min\{w_h\trans \phi(\cdot,\cdot)+r_h(\cdot,\cdot)+u_h^{plan}(\cdot,\cdot), \tdh\}$ and $V_h(\cdot)=\max_{a\in\cA}Q_h(\cdot,a)$\\
        $\pi_h(\cdot)\gets \arg\max_{a\in\cA}Q_h(\cdot,a)$
    }
    \Return $\pi_{r|\tdh}=\{\pi_h\}_{h\in[\tdh]}$, $\hat V_1(\cdot,r|\tdh):=V_1(\cdot)$
    \caption{Reward-Free DE-RL with Deterministic Policies in Linear MDPs: Planning Phase}\label{alg:DE_rl_layer_by_layer_reward_free_planning}
\end{algorithm}

\subsection{Analysis for Alg \ref{alg:DE_rl_layer_by_layer_reward_free_exploration} and Alg \ref{alg:DE_rl_layer_by_layer_reward_free_planning}}
We first show a lemma adapted from Lemma \ref{lem:deployment_efficient_concentration} for Alg \ref{alg:DE_rl_layer_by_layer_reward_free_exploration}. Since the proof is similar, we omit it here.
\begin{lemma}[Concentration for DE-RL in Reward-Free Setting]\label{lem:deployment_efficient_concentration_reward_free}
    We use $\mathcal{E}_2$ to denote the event that, when running Algorithm \ref{alg:DE_rl_layer_by_layer_reward_free_exploration}, the following inequality holds for all $k\in[K]$ and $h\in[h_k]$ and all $V=V^k_{h+1}$ occurs in Alg \ref{alg:DE_rl_layer_by_layer_reward_free_exploration} or $V=V_h$ occurs in Alg \ref{alg:DE_rl_layer_by_layer_reward_free_planning}:
    \begin{align*}
        \Big\|\sum_{\tau=1}^{k-1}\sum_{n=1}^N\phi_h^{\tau n}\Big(V(s_{h+1}^{\tau n})-\sum_{s'\in \cS} P_h(s'|s_{h}^{\tau n},a_{h}^{\tau n})V(s')\Big)\Big\|_{(\Lambda^k_h)^{-1}} \leq c\cdot dH\sqrt{\log(dKNH/\delta)}
    \end{align*}
    Under Assumption \ref{assump:linear_MDP}, there exists some absolute constant $c\geq 0$, such that $P(\mathcal{E}_2)\geq 1-\delta/2$.
\end{lemma}
\begin{proof}
    The proof is almost identical to Lemma 3.1 in \citep{wang2020rewardfree}, so we omit it here. The only difference is that we have an inner summation from $n=1$ to $N$ and we truncate the horizon at $h_k$ in iteration $k$.
\end{proof}
Next, we prove a lemma simlar to Lemma \ref{lem:overestimation_LBL_BatchLSVI_UCB} based on Lemma \ref{lem:deployment_efficient_concentration_reward_free}.
\begin{lemma}[Overestimation]\label{lem:overestimation}
    On the event $\mathcal{E}_2$ in Lemma \ref{lem:deployment_efficient_concentration_reward_free}, which holds with probability $1-\delta/2$, in Algorithm \ref{alg:DE_rl_layer_by_layer_reward_free_exploration}, for all $k\in[K]$ and $n\in[N]$, 
    \begin{align*}
        V^*_1(s_1^{kn},r^k|h_k) \leq V_1^k(s_1^{kn})
    \end{align*}
    and
    \begin{align*}
        \EE_{s_1\sim d_1}[V_1^*(s, r^k|h_k)] \leq \EE_{s_1\sim d_1}[V_1^k(s)] \leq (2H+1)\EE_{s_1\sim d_1}[V^{\pi_k}(s, r^k|h_k)]
    \end{align*}
\end{lemma}
\begin{proof}
    We first prove the overestimation inequality.
    \paragraph{Overestimation}
    First of all, similar to the proof of Lemma 3.1 in \citep{wang2020rewardfree}, on the event of $\mathcal{E}_2$ defined in Lemma \ref{lem:deployment_efficient_concentration}, which holds with probability $1-\delta/2$, we have:
    \begin{align}
        |\phi(s,a)\trans w_h^k - \sum_{s'\in\cS}P_h(s'|s,a)V^k_{h+1}(s')|\leq \beta\cdot\|\phi(s,a)\|_{(\Lambda^k_h)^{-1}},\quad \forall s,a\in\cS\times\cA,~k\in[K],~h\in[h_k]\label{eq:estimation_bias}
    \end{align}
    Then, we can use induction to show the overestimation. For $h=h_k+1$, we have:
    \begin{align*}
        0=V^*_{h_k+1}(s,r^k|h_k)\leq V^k_{h_k+1}(s)=0,\quad\forall s\in\cS
    \end{align*}
    Suppose for some $h\in[h_k]$, we have
    \begin{align*}
        V^*_{h+1}(s,r^k|h_k)\leq V^k_{h+1}(s),\quad\forall s\in\cS
    \end{align*}
    Then, $\forall s\in\cS$, we have
    \begin{align*}
        V_h^*(s,r^k|h_k)=&\max_a(r_h^k(s,a)+\sum_{s'\in\cS}P_h(s'|s,a)V^*_{h+1}(s',r^k|h_k))\\
        \leq& \max_a (r_h^k(s,a)+\sum_{s'\in\cS}P_h(s'|s,a)V^k_{h+1}(s')),H\\
        \leq& \min\{\max_a (r_h^k(s,a)+\phi(s,a)\trans w^k_h+\beta\|\phi(s,a)\|_{(\Lambda^k_h)^{-1}}), H\}\\
        =& \max_a\min\{r_h^k(s,a)+\phi(s,a)\trans w^k_h+\beta\|\phi(s,a)\|_{(\Lambda^k_h)^{-1}}, H\}\\
        =&V^k_h(s)
    \end{align*}
    where in the last inequality, we apply Eq.\eqref{eq:estimation_bias}.
    \paragraph{Relationship between $V_1^k(\cdot)$ and $V^{\pi_k}(\cdot,r^k)$}
    \begin{align*}
        &\EE_{s_1\sim d_1}[V_1^k(s_1)-V^{\pi_k}(s_1, r^k|h_k)]\\
        =&\EE_{s_1\sim d_1,a_1\sim\pi_k}[Q_1^k(s_1,a_1)-Q^{\pi_k}(s_1,a_1,r^k|h_k)]\\
        =&\EE_{s_1\sim d_1,a_1\sim\pi_k}[\min\{(w_1^k)\trans\phi(s_1,a_1)+r^k_1(s_1,a_1)+u_1^k(s_1,a_1), H\}\\
        \quad & -r^k_1(s_1,a_1)-\sum_{s_2\in\cS}P_1(s_2|s_1,a_1)V^{\pi_k}_2(s_2,r^k|h_k))]\\
        \leq&\EE_{s_1\sim d_1,a_1\sim\pi_k}[\min\{(w_1^k)\trans\phi(s_1,a_1), H-r_1^k(s_1,a_1)-u_1^k(s_1,a_1)\}-\sum_{s_2\in\cS}P_1(s_2|s_1,a_1)V^k_2(s_2)]\\
        &+\EE_{s_1\sim d_1,a_1\sim \pi_k}[\sum_{s_2\in\cS}P_1(s_2|s_1,a_1)V^k_2(s_2)-\sum_{s_2\in\cS}P_1(s_2|s_1,a_1)V^{\pi_k}_2(s_2,r^k|h_k))]+\EE_{s\sim d_1}[u^k_1(s_1,a_1)]\\
        \leq& \EE_{s_1\sim d_1,a_1\sim\pi_k}[\sum_{s'\in\cS}P_1(s_2|s_1,a_1)V_2^k(s_2)-\sum_{s_2\in\cS}P_1(s_2|s_1,a_1)V^{\pi_k}_2(s_2,r^k|h_k))]+2\EE_{s\sim d_1}[u^k_1(s_1,a_1)]\\
        =& \EE_{s_1\sim d_1,a_1,s_2,a_2\sim \pi_k}[V_2^k(s_2)-V^{\pi_k}_2(s_2,r^k|h_k)]+2\EE_{s\sim d_1}[u^k_1(s_1,a_1)]\\
        \leq & ...\\
        \leq& 2\EE_{s_1\sim d_1,a_1,...,s_{h_k},a_{h_k}\sim\pi_k}[\sum_{h=1}^H u_h^k(s_h,a_h)]\\
        =& 2H\EE_{s_1\sim d_1}[V^{\pi_k}(s,r^k|h_k)]
    \end{align*}
    where in the first inequality, we add and subtract $\sum_{s_2\in\cS}P_1(s_2|s_1,a_1)V_2^k(s_2)$, and in the second inequality, we use the following fact
    \begin{align*}
        &\min\{(w_1^k)\trans\phi(s_1,a_1), H-r_1^k(s_1,a_1)-u_1^k(s_1,a_1)\}-\sum_{s_2\in\cS}P_1(s_2|s_1,a_1)V^k_2(s_2)\\
        =& \min\{(w_1^k)\trans\phi(s_1,a_1)-\sum_{s_2\in\cS}P_1(s_2|s_1,a_1)V^k_2(s_2), H-r_1^k(s_1,a_1)-u_1^k(s_1,a_1)-\sum_{s_2\in\cS}P_1(s_2|s_1,a_1)V^k_2(s_2)\}\\
        \leq& \min\{\beta\|\phi(s_1,a_1)\|_{(\Lambda_1^k)^{-1}}, H\}=u_1^k(s_1,a_1)
    \end{align*}
\end{proof}
Next, we provide some analysis for Algorithm \ref{alg:DE_rl_layer_by_layer_reward_free_planning}, which will help us to understand what we want to do in Algorithm \ref{alg:DE_rl_layer_by_layer_reward_free_exploration}
\begin{lemma}\label{lem:analysis_of_reward_free_planning_alg}
    On the event $\mathcal{E}_2$ in Lemma \ref{lem:deployment_efficient_concentration_reward_free}, which holds with probability $1-\delta/2$, if we assign $\tdh=h_k$ in Algorithm \ref{alg:DE_rl_layer_by_layer_reward_free_planning} and assign $\mathcal{D}$ to be the samples collected till deployment $k$, i.e. $\mathcal{D}=\{(s_h^{kn},a_h^{kn})_{k,n,h\in[K]\times[N]\times[h_k]}\}$, then for arbitrary reward function $r$ satisfying the linear Assumption \ref{assump:linear_MDP}, the policy $\pi_{r|\tdh}$ returned by Alg \ref{alg:DE_rl_layer_by_layer_reward_free_planning} would satisfy:

    \begin{align}
        \EE_{s_1\sim d_1}[V_1^{\pi_{r|\tilde h}^*}(s_1,r|\tilde h)-V_1^{\pi_{r|\tilde h}}(s_1,r|\tilde h)]\leq 2H\EE_{s_1\sim d_1}[V_1^{\pi_{r^{plan}|\tdh}^*}(s_1,r^{plan}|\tilde h)]\label{eq:reward_free_subopt_gap}
    \end{align}
    where $r^{plan}:=u^{plan}/\tdh$.
\end{lemma}
\begin{proof}
By applying the similar technique in the analysis of $\EE_{s_1\sim d_1}[V_1^k(s_1)-V^{\pi_k}(s_1,r^k|\tilde h)]$ in Lemma \ref{lem:overestimation} after replacing $r^k$ with $r$, we have:
\begin{align*}
    \EE_{s_1\sim d_1}[V_1^{\pi^*_{r|\tilde h}}(s_1,r|\tilde h)-V_1^{\pi_{r|\tilde h}}(s_1,r|\tilde h)]\leq& \EE_{s_1\sim d_1}[\hat V_1(s_1,r|\tilde h)-V_1^{\pi_{r|\tilde h}}(s_1,r|\tilde h)]\leq 2\EE_{s_1\sim d_1}[V_1^{\pi_{r|\tilde h}}(s_1,u^{plan})]
\end{align*}
where $\hat V_1$ denotes the value function returned by Alg \ref{alg:DE_rl_layer_by_layer_reward_free_planning}
Besides,
\begin{align*}
    2\EE_{s_1\sim d_1}[V_1^{\pi_{r|\tdh}}(s_1,u^{plan})]=& 2\tdh\EE_{s_1\sim d_1}[V_1^{\pi_{r|\tdh}}(s_1,r^{plan}|\tdh)]\leq 2H\EE_{s_1\sim d_1}[V_1^{\pi_{r^{plan}|\tdh}^*}(s_1,r^{plan}|\tdh)]
\end{align*}

then, we finish the proof.
\end{proof}

From Eq.\eqref{eq:reward_free_subopt_gap} in Lemma \ref{lem:analysis_of_reward_free_planning_alg}, we can see that, after exploring with Algorithm \ref{alg:DE_rl_layer_by_layer_reward_free_exploration}, the sub-optimality gap between $\pi^*$ and $\pi$ returned by Alg.\ref{alg:DE_rl_layer_by_layer_reward_free_planning} can be bounded by the value of the optimal policy w.r.t. $r^K$, which we will further bound in the next theorem.

Now we are ready to prove the main theorem.
\begin{restatable}{theorem}{ThmRewardFreeConvergence}\label{thm:deployment_complexity_reward_free_detailed_version}
    For arbitrary $\epsilon,\delta > 0$, by assigning $K=c_KdH+1$ for some $c_K \geq 2$, 
    as long as
    \begin{align}
        N \geq c \Big(c_K\frac{H^{6c_K+1}d^{3c_K}}{\epsilon^{2c_K}}\log^{2c_K}(\frac{Hd}{\delta\epsilon})\Big)^{\frac{1}{c_K-1}}\Big\}\label{eq:lower_bound_on_N_reward_free}
    \end{align}
    where $c$ is an absolute constant and independent with $c_K, d,H,\epsilon,\delta$, then, Alg \ref{alg:DE_rl_layer_by_layer_reward_free_exploration} will terminate at iteration $k_H \leq K$ and return us a dataset $D=\{D_1,D_2,...,D_H\}$, such that given arbitrary reward function $r$ satisfying Assumption \ref{assump:linear_MDP}, by running Alg \ref{alg:DE_rl_layer_by_layer_reward_free_planning} with $D$ and $r$, with probability $1-\delta$, we can obtain a policy $\pi_r$ satisfying $\EE_{s\sim d_1}[V_1^*(s,r)-V^{\pi}(s,r)] \leq \epsilon$. 
    
    Moreover, for each $h\in[H-1]$, there exists iteration $k_h$, such that $h_{k_h}=h$ but $h_{k_h+1}=h+1$, and if we run Alg \ref{alg:DE_rl_layer_by_layer_reward_free_planning} with the reward function $r$ and the dataset Alg \ref{alg:DE_rl_layer_by_layer_reward_free_exploration} has collected till $k=k_h$, we can obtain a policy $\pi_{r|h}$, which is an $\epsilon$-optimal policy for MDP truncated at step $h$.
\end{restatable}
\begin{proof}    
    The proof is similar to Theorem \ref{thm:deployment_complexity_given_reward}.
    As stated in theorem, we use $k_h$ to denote the number of deployment when the algorithm switch the exploration from layer $h$ to layer $h+1$, i.e. $h_{k_h}=h$ and $h_{k_h+1}=h+1$. According to the definition and the algorithm, we must have $\Delta_{k_h}\leq \frac{\epsilon h_{k_h}}{(4H+2)H}$, and for arbitrary $k_{h-1}+1\leq k \leq k_{h}-1$, we must have $\Delta_k \geq \frac{\epsilon h_k}{(4H+2)H}$ (if $k_{h-1}+1>k_{h}-1$, then it means $\Delta_{k_{h-1}+1}$ is small enough and the algorithm directly switch the exploration to the next layer, and we can skip the discussion below). 
    Therefore, for arbitrary $k_{h-1}+1\leq k \leq k_{h}-1$, during the $k$-th deployment, there exists $h\in[h_k]$, such that,
    \begin{align*}
        \frac{\epsilon}{(4H+2)H} \leq \frac{\Delta_k}{h_k} \leq \frac{2\beta}{N}\sum_{n=1}^N \|\phi(s_{h}^{kn},a_{h}^{kn})\|_{(\Lambda^k_{h})^{-1}} \leq 2\beta\sqrt{\frac{1}{N}\sum_{n=1}^N \|\phi(s_{h}^{kn},a_{h}^{kn})\|^2_{(\Lambda^k_{h})^{-1}}}
    \end{align*}
    which implies that
    \begin{align}
        \frac{1}{N}\sum_{n=1}^N \|\phi(s_{h}^{kn},a_{h}^{kn})\|^2_{(\Lambda^k_{h})^{-1}} \geq \frac{\epsilon^2}{16H^2(2H+1)^2\beta^2} \label{eq:violation_reward_free}
    \end{align}
    According to Lemma \ref{lem:finite_sample_elliptical_potential_lemma}, there exists an absolute constant $c$, for arbitrary $\epsilon<1$, by choosing $N$ according to Eq.\eqref{eq:reward_free_N_constraint_1} below, the events in Eq.\eqref{eq:violation_reward_free} will not happen more than $c_Kd$ times at each layer $h\in[H]$.
    \begin{align}
        N\geq c\Big(c_K\frac{H^{6c_K+1}d^{3c_K}}{\epsilon^{2c_K}}\log^{2c_K}(\frac{Hd}{\epsilon\delta})\Big)^{\frac{1}{c_K-1}} \label{eq:reward_free_N_constraint_1}
    \end{align}
    We use $\zeta(k,j)$ to denote the total number of times Eq.\eqref{eq:violation_with_reward_layer_by_layer} happens for layer $j$ till deployment $k_h$. With a similar discussion as Theorem \ref{thm:deployment_complexity_given_reward}, we have:
    \begin{align*}
        k_h \leq \sum_{j=1}^{h}\zeta(k_h,j) -(h-1)+h \leq c_Kdh+1,\quad \forall h \in [H]
    \end{align*}
    Moreover, we must have $\Delta_{k_h} \leq \frac{\epsilon}{4H+2}$ for each $h\in[H]$, and according to Hoeffding inequality, with probability $1-\delta/2$, for each step $k$, we must have
    \begin{align*}
        \EE_{s_1,a_1,...,s_h,a_h\sim \pi_{k_h}}[2\beta\sum_{h'=1}^h \|\phi(s_{h'},a_{h'})\|_{(\Lambda^k_{h'})^{-1}}] \leq \Delta_{k_h} + 2\beta H\sqrt{\frac{1}{2N}\log(\frac{K}{\delta})}
    \end{align*}
    Therefore, by choosing 
    \begin{align}
        N \geq \frac{8\beta^2H^2(2H+1)^2}{\epsilon^2}\log(\frac{K}{\delta})=O(\frac{d^2H^6}{\epsilon^2}\log^2(\frac{K}{\delta}))\label{eq:reward_free_N_constraint_2}
    \end{align}
    we have,
    \begin{align*}
        \EE_{s_1,a_1,...,s_h,a_h\sim \pi_{k_h}}[2\beta\sum_{h'=1}^h \|\phi(s_{h'},a_{h'})\|_{(\Lambda^k_{h'})^{-1}}] \leq \Delta_{k_h} + \frac{\epsilon}{4H+2} = \frac{\epsilon}{2H+1}
    \end{align*}
    For arbitrary $h\in[H]$, in Algorithm \ref{alg:DE_rl_layer_by_layer_reward_free_planning}, if we assign $\tdh=h$ and $\mathcal{D}=\{(s_h^{kn},a_h^{kn})_{k,n,h\in[k_h]\times[N]\times[h]}\}$, note that $r^{plan}=r^{k_h}$, by applying Lemma \ref{lem:overestimation} and Lemma \ref{lem:analysis_of_reward_free_planning_alg} we have:
    \begin{align*}
        &\EE_{s_1\sim d_1}[V_1^{\pi_{r|\tilde h}^*}(s_1,r|\tilde h)-V_1^{\pi_{r|\tilde h}}(s_1,r|\tilde h)]\leq 2H\EE_{s_1\sim d_1}[V_1^{\pi_{r^{plan}|\tdh}^*}(s_1,r^{plan}|\tilde h)]\\
        =&2H\EE_{s_1\sim d_1}[V_1^*(s_1,r^{k_h}|h)]\leq 2H(2H+1)\EE_{s_1\sim d_1}[V^{\pi_{k_h}}(s_1, r^{k_h}|h)]\\
        =& (2H+1) \EE_{s_1,a_1,...,s_h,a_h\sim \pi_{k_h}}[2\beta\sum_{h'=1}^h \|\phi(s_{h'},a_{h'})\|_{(\Lambda^k_{h'})^{-1}}] \leq \epsilon
    \end{align*}
    Therefore, after a combination of Eq.\eqref{eq:reward_free_N_constraint_1} and Eq.\eqref{eq:reward_free_N_constraint_2}, we can conclude that, for arbitrary $c_K \geq 2$ , there exists absolute constant $c$, such that by choosing
    $$
    N \geq c \Big(c_K\frac{H^{6c_K+1}d^{3c_K}}{\epsilon^{2c_K}}\log^{2c_K}(\frac{Hd}{\delta\epsilon})\Big)^{\frac{1}{c_K-1}}
    $$
    Alg \ref{alg:DE_rl_layer_by_layer_reward_free_exploration} will terminate at $k_H \leq K$, and with probability $1-\delta$ (on the event in Lemma \ref{lem:deployment_efficient_concentration_reward_free} and Hoeffding inequality above), for each $h\in[H]$, if we feed Alg \ref{alg:DE_rl_layer_by_layer_reward_free_planning} with $\tdh=h$, $\mathcal{D}=\{(s_h^{kn},a_h^{kn})_{k,n,h\in[k_h]\times[N]\times[h]}\}$ and arbitrary linear reward function $r$, the policy $\pi_{r|h}$ returned by Alg \ref{alg:DE_rl_layer_by_layer_reward_free_planning} should satisfy:
    \begin{align*}
        \EE_{s_1\sim d_1}[V_1^{\pi_{r|h}}(s_1,r|h)] \geq \EE_{s_1\sim d_1}[V_1^{\pi_{r|h}^*}(s_1,r|h)] - \epsilon
    \end{align*}

\end{proof}

\section{DE-RL with Arbitrary Deployed Policies}\label{appx:DERL_with_arbitrary_policies}
In the proof for this section, without loss of generality, we assume the initial state is fixed, which will makes the notation and derivation simpler without trivialize the results. For the case where initial state is sampled from some fixed distribution, our algorithms and results can be extended simply by considering the concentration error related to the initial state distribution.

\subsection{Algorithms}
\begin{algorithm}[ht]
    \textbf{Input}: Time step $\h$; Dataset in previous steps $\{D_1,...,D_{\h-1}\}$; Unregularized Covariance Matrices $\{\Sigma_1, ...\Sigma_{\h-1}\}$; Bonus factor $\betap$; Matrix to construct reward function $\Sigma_R$; Discretize resolution $\resolution \leq \frac{1}{2d(N+1)}$\\

    $R(\cdot,\cdot)\gets \sqrt{\phi(\cdot,\cdot)\trans (2I+\Sigma_R)^{-1}\phi(\cdot,\cdot)}$\\
    $Z_\h \gets {\rm Discretization}((2I+\Sigma_R)^{-1}, \frac{\resolution^2}{4d}),~\bar R(\cdot,\cdot)\gets \sqrt{\phi(\cdot,\cdot)\trans Z_\h\phi(\cdot,\cdot)}$\\

    $Q_\h(\cdot,\cdot)=R(\cdot,\cdot),\quad V_\h(\cdot)=\max_a Q_\h(\cdot,a)$, $\bar Q_\h(\cdot,\cdot)=\bar R(\cdot,\cdot),\quad \bar V_\h(\cdot)=\max_a \bar Q_\h(\cdot,a)$\\
    $\bar\pi_\h(\cdot)=\arg\max_a \bar{Q}_\h(\cdot,a)$\\
    \For{$\tdh=\h-1,...,1$}{
        $w_\tdh \gets \Sigma_\tdh^{-1}\sum_{(s_\tdh,a_\tdh,s_{\tdh+1}) \in D_\tdh} \phi(s_\tdh, a_\tdh)\cdot V_{\tdh+1}(s_{\tdh+1})$\\
        $u_\tdh := \betap [\phi(\cdot,\cdot)\trans \Sigma_\tdh^{-1} \phi(\cdot,\cdot)]^{1/2}$\\
        $Q_\tdh(\cdot,\cdot) \gets \min\{w_\tdh\trans \phi(\cdot,\cdot)+u_\tdh, 1\}$, $V_\tdh(\cdot) \gets \max_a Q_\tdh(\cdot,a)$\\
        $\bar{w}_\tdh \gets {\rm Discretization}(w_\tdh, \frac{\resolution}{2d})$, $Z_\tdh \gets {\rm Discretization}(\betap^2\Sigma_\tdh^{-1}, \frac{\resolution^2}{4d})$\\
        $\bar u_\tdh := [\phi(\cdot,\cdot)\trans Z_\tdh \phi(\cdot,\cdot)]^{1/2}$\\
        $\bar{Q}_\tdh(\cdot,\cdot) \gets \min\{\bar{w}_\tdh\trans \phi(\cdot,\cdot)+\bar u_\tdh,1\}$, $\bar V_\tdh(\cdot) \gets \max_a \bar Q_\tdh(\cdot,a)$\\
        $\bar\pi_\tdh(\cdot)\gets \arg\max_a \bar{Q}_\tdh(\cdot,a)$
    }
    \Return $V_1(s_1), \bar\pi:=\bar\pi_1\circ\bar\pi_2\circ...\bar\pi_\h$
    \caption{SolveOptQ}\label{alg:optQ_solution}
\end{algorithm}
\begin{algorithm}[ht]
    \textbf{Input}: Time step $\h$; Dataset in previous steps $\{D_1,...,D_{\h-1}\}$; Covariance Matrices $\{\Sigma_1, ...\Sigma_{\h-1}\}$; Deterministic Policy to evaluate $\bpi=\{\bpi_1,\bpi_2,...,\bpi_h\}$\\
    Initialize a zero matrix $\tilde\hatcov^\pi_\h=O$\\
    \For{$i=1,2,...,d$}{
        \For{$j=i,i+1,...,d$}{
            Define $\tilde R^{ij}$, such that, $\tilde R^{ij}_\h(\cdot,\cdot)=\frac{1+\phi_i(\cdot,\cdot)\phi_j(\cdot,\cdot)}{2}$ and $\tilde R^{ij}_{\tdh}=0$ for all $\tdh\in[h-1]$;\\
            $\hat Q^\bpi_\h(\cdot,\cdot)=\tilde R^{ij}_\h(\cdot,\cdot)$, $\hat V^\bpi_\h(\cdot)=\hat Q^\bpi_\h(\cdot,\bpi_\h(\cdot))$\\
            \For{$\tdh=\h-1,...,1$}{
                $\hat w^\bpi_\tdh \gets \Sigma_\tdh^{-1}\sum_{(s_\tdh,a_\tdh,s_{\tdh+1}) \in D_\tdh} \phi(s_{t}, a_{t})\cdot \hat V^\bpi_{\tdh+1}(s_{\tdh+1})$\\
                $\hat Q^\bpi_\tdh(\cdot,\cdot) \gets \min\{\tilde R^{ij}_{\tdh}(\cdot,\cdot) + (\hat w^\bpi_\tdh)\trans \phi(\cdot,\cdot), 1\}$ \\ 
                $\hat V^\bpi_\tdh(\cdot) = \hat Q^\bpi_\tdh(\cdot,\bpi_\tdh(\cdot))$
            } 
            $(\tilde\hatcov^\bpi_\tdh)_{ij}\gets \hat V^\bpi_1(s_1);\quad (\tilde\hatcov^\bpi_\tdh)_{ji} \gets \hat V^\bpi_1(s_1)$
        }
    }
    $\hat \hatcov_\tdh^{\bpi} = 2(\tilde\hatcov^{\bpi}_\tdh) - \textbf{1}$\\
    \Return $\hat \hatcov_\tdh^{\bpi}$
    \caption{EstimateCovMatrix}\label{alg:approx_random_feature_matrix}
\end{algorithm}
We first introduce the definition for $\mathrm{Discretization}$ function:

\begin{definition}[$\mathrm{Discretization}$ function]\label{def:discretize_function}
Given vector $w=(w_1,w_2,...,w_d)\trans\in\mR^d$ or matrix $\Sigma=(\Sigma_{ij})_{i,j\in[d]}\in\mR^{d\times d}$ as input, we have:
\begin{align*}
    {\rm Discretization}(w,\resolution) = (\resolution\lceil\frac{w_1}{\resolution}\rceil, \resolution\lceil\frac{w_2}{\resolution}\rceil,...,\resolution\lceil\frac{w_d}{\resolution}\rceil), \quad {\rm Discretization}(\Sigma,\resolution) = (\resolution\lceil\frac{\Sigma_{ij}}{\resolution}\rceil)_{i,j\in[d]}
\end{align*}
where $\lceil \cdot \rceil$ is the ceiling function.
\end{definition}

In Algorithm \ref{alg:approx_random_feature_matrix}, we are trying to estimate the expected covariance matrix under policy $\bpi$ by policy evaluation. The basic idea is that, the expected covariance matrix can be represented by:
\begin{align*}
    \EE_{s_h,a_h\sim \bpi}[\phi(s_h,a_h)\phi(s_h,a_h)\trans] = \big(\EE_{\bpi}[\phi_i(s_h,a_h)\phi_j(s_h,a_h)]\big)_{ij} = \big(V^\bpi(s_1, R^{ij})\big)_{ij}
\end{align*}
where we use $(a_{ij})_{ij}$ to denote a matrix whose element indexed by $i$ in row and $j$ in column is $a_{ij}$.
In another word, the element in the covariance matrix indexed by $ij$ is equal to the value function of policy $\bpi$ with $R_{\h}^{ij}(s_h,a_h):=\phi_i(s_h,a_h)\phi_j(s_h,a_h)$ as reward function at the last layer (and use zero reward in previous layers), where $\phi_i$ denotes the $i$-th elements of vector $\phi$. 
Because the techniques rely on the reward is non-negative and bounded in $[0,1]$, by leveraging the fact that $|\phi_i(\cdot,\cdot)| \leq \|\phi(\cdot,\cdot)\| \leq 1$, we shift and scale $R^{ij}$ to obtain $\tilde R^{ij}$ and use it for policy evaluation.

In Alg \ref{alg:optQ_solution}, we maintance two $Q$ functions $Q_h$ and $\bQ_h$. The learning of $Q_h$ is based on LSVI-UCB, while $\bQ_h$ is a ``discretized version'' for $Q_h$ computed by discretizing $w_h, \betap^2\Sigma_h^{-1}$ (or $\Sigma_R^{-1}$ at layer $h$) elementwisely with resolution $\resolution$, and $\bQ_h$ will be used to compute $\bpi_h$ for deployment. The main reason why we discretize $Q_h$ is to make sure the number of greedy policies $\bpi$ is bounded, so that we can use union bound and upper bound the error when using Alg \ref{alg:approx_random_feature_matrix} to estimate the covariance matrix. In Section \ref{sec:discretize_error_analysis}, we will analyze the error resulting from discretization, and we will upper bound the estimation error Alg. \ref{alg:approx_random_feature_matrix}.
\subsection{Function Classes and $\resolution$-Cover}
We first introduce some useful function classes and their $\resolution$-cover.
\paragraph{Notation for Value Function Classes and Policy Classes}
We first introduce some new notations for value and policy classes.
Similar to Eq.(6) in \citep{jin2019provably}, we define the greedy value function class
\begin{align*}
    \cV_{L,B}^*=\{V(\cdot)|V(\cdot)=\max_a\min\{\phi(\cdot,a)\trans w+\sqrt{\phi(\cdot,a)\trans\Sigma\phi(\cdot,a)}, 1\}, \|w\|\leq L, \|\Sigma\|\leq B\}
\end{align*}
and the Q function class:
\begin{align*}
    \cQ_{L,B}=\{Q(\cdot,\cdot)|Q(\cdot,\cdot)=\min\{\phi(\cdot,\cdot)\trans w+\sqrt{\phi(\cdot,\cdot)\trans\Sigma\phi(\cdot,\cdot)}, 1\}, \|w\|\leq L, \|\Sigma\|\leq B\}
\end{align*}
Besides, suppose we have a deterministic policy class $\Pi$ with finite candidates (i.e. $|\Pi|\leq \infty$), we use $\cV_{L,B}\times\Pi$ to denote:
\begin{align*}
    \cV_{L,B}\times\Pi=\{V(\cdot)|V(\cdot)=\min\{\phi(\cdot,\pi(\cdot))\trans w+\sqrt{\phi(\cdot,\pi(\cdot))\trans\Sigma\phi(\cdot,\pi(\cdot))}, 1\}, \|w\|\leq L, \|\Sigma\|\leq B,\pi\in\Pi\}
\end{align*}
Recall that in Alg.\ref{alg:approx_random_feature_matrix}, we will use a special reward function, and we need to consider it in the union bound. We denote:
\begin{align*}
    \cV_\phi\times\Pi = \{V|V(\cdot)=\frac{1+\phi_i(\cdot,\pi(\cdot))\phi_j(\cdot,\pi(\cdot))}{2},i,j\in[d], \pi\in\Pi\}
\end{align*}
and easy to check $|\cV_\phi\times\Pi|=d^2|\Pi|$. 

Moreover, if we have a Q function class $\cQ$, we will use $\Pi_\cQ$ to denote the class of greedy policies induced from $\cQ$, i.e.
\begin{align*}
    \Pi_\cQ := \{\pi(\cdot)=\arg\max Q(\cdot,a)|Q\in\cQ\}.
\end{align*}

\paragraph{Discretization with Resolution $\resolution$}
In the following, we will use $\mathcal{C}_{w,L,\resolution}$ to denote the $\resolution$-cover for $w\in\mR^d$ with $\|w\|\leq L$, concretely,
\begin{align*}
    \mathcal{C}_{w,L,\resolution} = \{w| |\frac{w_i}{\resolution}| \in [\lceil \frac{L}{\resolution}\rceil], \forall i \in [d]\}
\end{align*}
where $\lceil \cdot \rceil$ is the ceiling function.

Similarly, we will use $\mathcal{C}_{\Sigma,B,\resolution}$ to denote the $\resolution$-cover for matrix $\Sigma\in\mR^{d\times d}$ with $\max_{i,j}|\Sigma_{ij}| \leq B$
\begin{align*}
    \mathcal{C}_{\Sigma,B,\resolution} = \{\Sigma| |\frac{\Sigma_{ij}}{\resolution}| \in [\lceil \frac{B}{\resolution}\rceil], \forall i,j \in [d]\}.
\end{align*}
Easy to check that:
\begin{align*}
    \log |\mathcal{C}_{w,L,\resolution}| \leq d \log \frac{2L}{\resolution},\quad \log |\mathcal{C}_{\Sigma,B,\resolution}| \leq d^2 \log \frac{2B}{\resolution}
\end{align*}
Recall the definition of $\mathrm{Discretize}$ function in Def. \ref{def:discretize_function}, easy to check that:
\begin{align*}
    \|{\rm Discretize}(w,\resolution) - w\| \leq d \resolution, \|{\rm Discretize}(\Sigma,\resolution) - \Sigma\|\leq \|{\rm Discretize}(\Sigma,\resolution) - \Sigma\|_F \leq d \resolution
\end{align*}

\paragraph{$\resolution$-cover}
Before we introduce our notations for $\resolution$-net, we first show a useful lemma:
\begin{lemma}\label{lem:function_class_distance}
    For arbitrary $w, \Sigma$, denote $\bar w = \mathrm{Discretize}(w, \frac{\resolution}{2d})$ and $\bar \Sigma = \mathrm{Discretize}(\Sigma, \frac{\resolution^2}{4d})$. Consider the following two functions and their greedy policies, where $\|\phi(\cdot,\cdot)\|\leq 1$
    \begin{align*}
        Q(s,a)=\min\{w\trans\phi(\cdot,a) + \sqrt{\phi(\cdot,a)\trans\Sigma\phi(\cdot,a)},1\},\quad \pi = \arg\max_a Q(s,a)\\
        \bar Q(s,a)=\min\{\bar w\trans\phi(\cdot,a) + \sqrt{\phi(\cdot,a)\trans\bar \Sigma\phi(\cdot,a)},1\},\quad \bpi = \arg\max_a \bar Q(s,a)
    \end{align*}
    then we have:
    \begin{align*}
        |Q(s,\pi(s)) - Q(s,\bar\pi(s))|\leq 2\resolution,\quad\forall s\in\cS,\quad\\
        \|Q-\bar Q\|_\infty \leq \resolution,\quad \sup_s|\max_a Q(\cdot, a)-\max_a \bQ(\cdot,a)|\leq \resolution
    \end{align*}
\end{lemma}
\begin{proof}
    After similar derivation as Eq.(28) in \citep{jin2019provably}, we can show that
    \begin{align*}
        \sup_s|\max_a Q(\cdot, a)-\max_a \bQ(\cdot,a)|\leq&\|Q-\bQ\|_\infty\\
        \leq& \sup_{s,a}\Big| w\trans\phi(\cdot,\cdot)+\sqrt{\phi(\cdot,\cdot)\trans \Sigma\phi(\cdot,\cdot)}-\bar w\trans\phi(\cdot,\cdot)-\sqrt{\phi(\cdot,\cdot)\trans \bar \Sigma\phi(\cdot,\cdot)}\Big|\\
        \leq &\|w-\bar w\| + \sqrt{\|\Sigma-\bar \Sigma\|_F}\leq d\frac{\resolution}{2d}+\sqrt{d\frac{\resolution^2}{4d}}\leq \resolution
    \end{align*}
    Because $\pi$ and $\bar\pi$ are greedy policies, we have:
    \begin{align*}
        0\leq& Q(s,\pi(s)) - Q(s,\bar\pi(s)) \leq Q(s,\pi(s)) - \bar Q(s,\pi(s))+\bar Q(s,\bar\pi(s)) - Q(s,\bar\pi(s))\\
        \leq&2\|Q-\bar Q\|_\infty \leq 2\resolution. \tag*{\qedhere}
    \end{align*}
\end{proof}

Now, we consider the following Q function class and V function class,
\begin{align*}
    \bar\cQ_{L, B, \resolution} := \{Q|Q(\cdot,\cdot)=\min\{w\trans\phi(\cdot,\cdot) + \sqrt{\phi(\cdot,\cdot)\trans\Sigma\phi(\cdot,\cdot)}, 1\}, w\in\mathcal{C}_{w,L,\frac{\resolution}{2d}}, \Sigma\in \mathcal{C}_{\Sigma, dB, \frac{\resolution^2}{4d}}\}\\
    \bar\cV^*_{L, B, \resolution} := \{V|V(\cdot)=\max_a\min\{w\trans\phi(\cdot,a) + \sqrt{\phi(\cdot,a)\trans\Sigma\phi(\cdot,a)}, 1\}, w\in\mathcal{C}_{w,L,\frac{\resolution}{2d}}, \Sigma\in \mathcal{C}_{\Sigma, dB, \frac{\resolution^2}{4d}}\}
\end{align*}
based on Lemma \ref{lem:function_class_distance}, and another important fact that $\max_{i,j}|a_{ij}|\leq \|A\|_F \leq d\|A\|$, we know that $\bar\cQ_{L,B,\resolution}$ is an $\resolution$-cover of $\cQ_{L,B}$, i.e. for arbitrary $Q\in\cQ_{L,B}$, there exists $\bQ\in\bar\cQ_{L,B,\resolution}$, such that $\|Q-\bQ\|\leq \resolution$. Similarly, $\bar\cV^*_{L,B,\resolution}$ is also an $\resolution$-cover of $\cV^*_{L,B}$. 

Besides, we will use $\Pi_{\bQ_{L,B,\resolution}}$ to denote the collection of greedy policy induced from elements in $\bar\cQ_{L, B, \resolution}$.

We also define $\bar\cV_{L, B, \resolution}\times\Pi$, which is an $\resolution$ cover for $\cV_{L, B}\times\Pi$.
\begin{align*}
\bar\cV_{L, B, \resolution}\times\Pi:=\{V|V(\cdot)=\min\{\phi(\cdot,\pi(\cdot))\trans w+\sqrt{\phi(\cdot,\pi(\cdot))\trans\Sigma\phi(\cdot,\pi(\cdot))}, 1\}, w\in\mathcal{C}_{w,L,\frac{\resolution}{2d}}, \Sigma\in \mathcal{C}_{\Sigma, dB, \frac{\resolution^2}{4d}}, \pi \in \Pi\}
\end{align*}
Obviously, $|\bar\cV_{L, B, \resolution}\times\Pi|=|\Pi|\cdot|\mathcal{C}_{w,L,\frac{\resolution}{2d}}|\cdot|\mathcal{C}_{\Sigma, dB, \frac{\resolution^2}{4d}}|$.

Besides, because $\cV_\phi\times\Pi$ is already a finite set, it is an $\resolution$-cover of itself.


\subsection{Constraints in Advance}
\paragraph{Induction Condition Related to Accumulative Error}
Recall the induction condition in \ref{cond:induction_condition}, and we restate it here.
\CondInductionError*
\paragraph{Constraints for the Validity of the Algorithm}
Besides, in order to make sure the algorithm can run successfully, we add the following constaints: 
\begin{align}
\Sigma_R \succeq& -\frac{1}{2}I \\
Z_\tdh\succeq& 0, \quad\forall \tdh \in [\h-1] \label{eq:constraint_Z_h_1} \\
I \succeq & Z_\h \succeq 0 \label{eq:constraint_Z_h_2}
\end{align}
where constraint \eqref{eq:constraint_Z_h_2} for $Z_\h$ is to make sure the reward $\bar R$ locates in $[0,1]$ interval.

Recall the definition of $\Sigma_\tdh=I+\sum_{n=1}^N \phi(s_\tdh,a_\tdh)\phi(s_\tdh,a_\tdh)\trans$, therefore,
\begin{align*}
    \sigma_{\min}(\betap^2\Sigma_\tdh^{-1}) = \betap^2/\sigma_{\max}(\Sigma_\tdh) \geq \frac{\betap^2}{1+N}
\end{align*}
According to Lemma \ref{lem:matrix_perturbation}, to make sure $Z_\tdh\succeq 0$, we need the following constraint on $\resolution$
\begin{align}
    d\frac{\resolution^2}{4d} \leq \frac{\betap^2}{(N+1)} \label{eq:constaint_resolution_0}
\end{align}
which is equivalent to $\resolution \leq \betap / \sqrt{N+1}$.

As for $Z_h$, the constraint is equivalent to:
\begin{align*}
    2I+\Sigma_R \succeq (1+\frac{\resolution^2}{4})I,\quad (2I+\Sigma_R)^{-1} \succeq \frac{\resolution^2}{4}I
\end{align*} 
and can be rewritten to
\begin{align}
    I+\Sigma_R \succeq \frac{\resolution^2}{4}I,\quad \frac{4}{\resolution^2}I \succeq 2I+\Sigma_R \label{eq:constraint_resolution_1}
\end{align} 





\subsection{Concentration Bound}
Based on the notations above, we are already to claim that:
\begin{claim}\label{claim:value_function_classes}
    By choosing $L=\sqrt{dN}, B=\betap^2 + d\resolution$ for some $\resolution\leq 1/d $, during the running of Algorithm \ref{alg:DERL_with_arbitrary_policies}
    \begin{itemize}
        \item In Alg \ref{alg:optQ_solution}, for each $\h\in[H]$ and $\tdh\in[\h]$, and the $\bQ_\tdh$ and $\bpi_\tdh$ generated while running the algorithm, we must have $\bQ_\tdh \in\bQ_{L,B,\resolution}$ and $\bpi_\tdh \in \Pi_{\bQ_{L,B,\resolution}}$. Besides, $Q_\tdh \in \cQ_{L,B}$ and $V_\tdh\in\cV_{L,B}^*$
        \item In Alg \ref{alg:approx_random_feature_matrix}, for all $\tdh\in[h]$, we have $\bpi_\tdh \in \Pi_{\bQ_{L,B,\resolution}}$, and for arbitrary value function $\hat V^\bpi_\tdh$ generated in it, we have $\hat V^\bpi_\tdh \in (\cV_{L,B}\times\Pi_{\bQ_{L,B,\resolution}}) \cup (\cV_\phi\times\Pi_{\bQ_{L,B,\resolution}})$ and therefore, there exists $V\in (\bar\cV_{L,B,\resolution}\times\Pi_{\bQ_{L,B,\resolution}}) \cup (\cV_\phi\times\Pi_{\bQ_{L,B,\resolution}})$ such that $\|V-\hat V^\bpi_\tdh\| \leq \resolution$
    \end{itemize}
\end{claim} 
\begin{proof}
    \textbf{Algorithm \ref{alg:optQ_solution}}: First we bound the norm of the weights $w_h$ in Algorithm \ref{alg:optQ_solution}. For arbitrary $v\in\mR^d$ and $\|v\|=1$, we have:
    \begin{align*}
        |v\trans w_\tdh| =& |v\Sigma_\tdh^{-1}\sum_{(s_\tdh,a_\tdh,s_{\tdh+1})\in D_\tdh}\phi(s_\tdh,a_\tdh)V_{\tdh+1}(s_{\tdh+1})|\leq |v\Sigma_\tdh^{-1}\sum_{(s_\tdh,a_\tdh,s_{\tdh+1})\in D_\tdh}\phi(s_\tdh,a_\tdh)|\\
        \leq&\sqrt{|\sum_{(s_\tdh,a_\tdh,s_{\tdh+1})\in D_\tdh} v\trans \Sigma_\tdh^{-1}v||\sum_{(s_\tdh,a_\tdh,s_{\tdh+1})\in D_\tdh}\phi(s_\tdh,a_\tdh)\trans\Sigma_\tdh^{-1}\phi(s_\tdh,a_\tdh)|}\\
        \leq& \|v\|\sqrt{d|D_\tdh|} = \|v\|\sqrt{dN}
    \end{align*}
    therefore, $\|w_\tdh\|\leq \sqrt{dN}$. Besides, according to Lemma \ref{lem:matrix_perturbation} and constraint \eqref{eq:constraint_resolution_1}, we have:
    \begin{align*}
        \|\betap^2\Sigma_\tdh^{-1}\|\leq \betap^2, \quad \|Z_\tdh\| \leq& \|\betap^2\Sigma_\tdh^{-1}\| + d\resolution\leq \betap^2 + d\resolution\quad \forall h\in[\h-1]\\
        \|(2I+\Sigma_R)^{-1}\|\leq 1, \quad \|Z_\h\| \leq& \|(2I+\Sigma_R)^{-1}\|+d\resolution \leq 1+d\resolution
    \end{align*}
    Recall $B=\betap^2+d\resolution$ and $\betap > 1$, the claim about Alg \ref{alg:optQ_solution} is true.

    \textbf{Algorithm \ref{alg:approx_random_feature_matrix}}: 
    The discussion about the value range of $\hat w^\bpi_\tdh$ is similar to above. Therefore, all the value functions occurred in the previous $\h-1$ layers would belong to $\cV_{L,B}\times\Pi$, except that the last layer should belong to $\cV_\phi\times\Pi$. Besides, since Alg \ref{alg:approx_random_feature_matrix} is only used to estimate the policies returned by Alg \ref{alg:optQ_solution}, we should have $\Pi=\Pi_{\bar\cQ_{L,B,\resolution}}$. As a result, the claim for Alg \ref{alg:approx_random_feature_matrix} is correct.
\end{proof}
Recall Lemma D.4 from \citep{jin2019provably}, which holds for arbitrary $\cV$ with covering number $\cN_\resolution$ and $\sup_s|V(s)|\leq H$. Next, we state a slightly generalized version by replacing $\sup_s|V(s)|\leq H$ with $\sup_s|V(s)|\leq 1$, since in our Alg. \ref{alg:optQ_solution} and Alg. \ref{alg:approx_random_feature_matrix} $V_{\max}=1$. This is also the main reason why we only need the coefficient of the bonus term $\betap=\tilde{O}(d)$ instead of $\tilde{O}(dH)$ in Alg. \ref{alg:optQ_solution} so that we can achieve better dependence on $H$.
\begin{lemma}[(Revised) Lemma D.4 in \citep{jin2019provably}]\label{lem:concentration_lemma}
    Let $\{s_\tau\}_{\tau=1}^\infty$, be a stochastic process on state space $\cS$ with corresponding filtration $\{\cF_\tau\}_{\tau=1}^\infty$. Let $\{\phi_\tau\}_{\tau=0}^\infty$ be an $\mR^d$-valued stochastic process where $\phi_\tau\in\cF_{\tau-1}$ and $\|\phi_\tau\|\leq 1$. Let $\Lambda_t = \lambda I +\sum_{\tau=1}^t \phi_\tau\phi_\tau\trans$. Then for any $\delta>0$, with probability at least $1-\delta$, for all $t \geq 0$, and any $V\in\cV$ so that $\sup_s |V (s)| \leq 1$, we have:
    \begin{align*}
        \Big\|\sum_{\tau=1}^t\phi_\tau\{V(s_\tau)-\EE[V(x_\tau)|\cF_{\tau-1}]\}\Big\|^2_{\Lambda_t^{-1}}\leq 4[\frac{d}{2}\log\frac{t+\lambda}{\lambda}+\log \frac{\cN_\resolution}{\delta}]+\frac{8t^2\resolution^2}{\lambda}
    \end{align*}
    where $\cN_\resolution$ is the $\resolution$-covering number of $\cV$ w.r.t. the distance ${\rm dist}(V,V')=\sup_s|V(s)-V(s)'|$
\end{lemma}


Now we are ready to prove the main concentration result for Algorithm \ref{alg:DERL_with_arbitrary_policies}:
\begin{theorem}\label{thm:concentration_lemma_for_discretization}
    Consider value function class $\cV:=\cV^*_{L,B}\cup(\cV_{L,B}\times\Pi_{\bQ_{L,B,\resolution}}) \cup (\cV_\phi\times\Pi_{\bQ_{L,B,\resolution}})$ with $L=1, B=\betap^2+d\resolution$.
    According to Claim \ref{claim:value_function_classes}, $\cV$ covers all possible value functions occurs when running Alg \ref{alg:DERL_with_arbitrary_policies}.
    We use $\mathcal{E}$ to denote the even that the following inequality holds for arbitrary $V\in\cV$ and arbitrary $k\in[K]=[H], h\in[H]$:
    \begin{align*}
        \Big\|\sum_{\tau=1}^{k-1}\sum_{n=1}^N\phi_h^{\tau n}\Big(V(s_{\tdh+1}^{\tau n})-\sum_{s'\in \cS} P_h(s'|s_{h}^{\tau n},a_{h}^{\tau n})V(s')\Big)\Big\|_{(\Sigma_h)^{-1}} \leq c\cdot d\sqrt{\log(dN\betap H/\resolution\delta)}
    \end{align*}
    As long as
    \begin{align}
        \resolution \leq 1/N \label{eq:constraint_resolution_2}
    \end{align}
    there exists some constant $c$, such that $P(\cE)\geq 1-\delta/2$.
\end{theorem}
\begin{proof}
    We consider the value function class:
    \begin{align*}
        \cV := \cV^*_{L,B} \cup (\cV_{L,B}\times\Pi_{\bQ_{L,B,\resolution}})\cup(\cV_\phi\times\Pi_{\bQ_{L,B,\resolution}})
    \end{align*}
    and we have an $\resolution$-cover for it, which we denote as:
    \begin{align*}
        \cV_\resolution := \bar\cV^*_{L,B,\resolution}\cup (\bar\cV_{L,B,\resolution}\times\Pi_{\bQ_{L,B,\resolution}})\cup(\cV_\phi\times\Pi_{\bQ_{L,B,\resolution}})
    \end{align*}
    Besides, there exists $c' > 0$, s.t.
    \begin{align*}
        \log |\cV_\resolution| \leq& \log |\cV_\phi\times\Pi_{\bQ_{L,B,\resolution}}| + \log |\bar\cV^*_{L,B,\resolution}| + \log |\bar\cV_{L,B,\resolution}\times\Pi_{\bQ_{L,B,\resolution}}|\\
        \leq& \log |\cV_\phi\times\bQ_{L,B,\resolution}| + \log |\bar\cV^*_{L,B,\resolution}| + \log |\bar\cV_{L,B,\resolution}\times\bQ_{L,B,\resolution}|\\
        \leq& c'd^2\log\frac{dHN\betap}{\resolution}
    \end{align*}
    By plugging into Lemma \ref{lem:concentration_lemma} and considering the union bound over $k\in[K]$ and $h\in[H]$ (note that $K=H$), we have with probability $1-\delta/2$,
    \begin{align*}
        \Big\|\sum_{\tau=1}^{k-1}\sum_{n=1}^N\phi_h^{\tau n}\Big(V(s_{\tdh+1}^{\tau n})-\sum_{s'\in \cS} P_h(s'|s_{h}^{\tau n},a_{h}^{\tau n})V(s')\Big)\Big\|^2_{(\Lambda^k_h)^{-1}} \leq c'' d^2\log\frac{dHN\betap}{\resolution\delta}+8N^2\resolution^2
    \end{align*}
    When $\resolution\leq \frac{1}{N}$, the first term will dominate, and there must exists $c$ such that
    \begin{align*}
        \Big\|\sum_{\tau=1}^{k-1}\sum_{n=1}^N\phi_h^{\tau n}\Big(V(s_{\tdh+1}^{\tau n})-\sum_{s'\in \cS} P_h(s'|s_{h}^{\tau n},a_{h}^{\tau n})V(s')\Big)\Big\|_{(\Lambda^k_h)^{-1}} \leq c\cdot d\sqrt{\log(dN\betap H/\resolution\delta)}
    \end{align*}
\end{proof}

\subsection{Bias Analysis}\label{sec:discretize_error_analysis}
\begin{lemma}[Overestimation in Alg \ref{alg:optQ_solution}]\label{lem:optQ_overestimation}
    Suppose we choose 
    \begin{align}
        \betap=c_\betap' d\sqrt{\log \frac{dHN}{\resolution\delta}}\label{eq:choice_of_beta}
    \end{align}
    for some $c_\betap > 0$. During the running of Alg \ref{alg:DERL_with_arbitrary_policies}, on the condition \ref{cond:induction_condition} and on the event of $\cE$ in Theorem \ref{thm:concentration_lemma_for_discretization}, which holds with probability $1-\delta/2$, for arbitrary $\h\in[H]$ and $\tdh\leq \h-1$, the parameter $w_\tdh$ and value function $V_{\tdh+1}$ occurs in Algorithm \ref{alg:optQ_solution} should satisfy:
    \begin{align*}
        |\phi(s,a)\trans w_\tdh-\sum_{s'\in\cS}P_\tdh(s'|s,a)V_{\tdh+1}(s')| \leq \betap \|\phi(s,a)\|_{\Sigma_\tdh^{-1}}
    \end{align*}
    and
    \begin{align*}
        V^*_\tdh(s) \leq V_\tdh(s) \leq V^*_\tdh(s) + \EE_{\pi}[\sum_{h'=\tdh}^\h \betap\|\phi(s_{h'},a_{h'})\|_{\Sigma_{h'}^{-1}}] \leq V^*_\tdh(s) + \betap\xi
    \end{align*}
\end{lemma}
\begin{proof}
    The proof is mainly based on Theorem \ref{thm:concentration_lemma_for_discretization}, and the steps are similar to Lemma B.3 in \citep{jin2019provably} and Lemma 3.1 in \citep{wang2020rewardfree} and we omit here.
\end{proof}
\begin{lemma}[Bias Accumulation in Alg \ref{alg:optQ_solution}]\label{lem:bais_accumulation}
    On the induction condition \ref{cond:induction_condition} and on the events in Theorem \ref{thm:concentration_lemma_for_discretization} which holds with probability $1-\delta/2$, if 
    \begin{align}
        \resolution \leq \frac{\betap\xi}{2H} \label{eq:constraint_resolution_3}
    \end{align} 
    in Algorithm \ref{alg:optQ_solution}, for arbitrary $\bar R$ generated, we have:
    \begin{align*}
        V_1^*(s_1;\bar R) - V_1^{\bpi}(s_1;\bar R) \leq 3\betap\xi
    \end{align*}
    where recall that we use $V^\pi(s;\bar R)$ to denote the value function with $\bar R$ as reward function.
\end{lemma}
\begin{proof}
    We will use $\pi_h(\cdot):=\arg\max_a Q_h(\cdot,a)$ to denote the optimal policy w.r.t. the $Q$ function without discretization, although we do not deploy it in practice.
    According to Lemma \ref{lem:function_class_distance}, we should have $\max_{s,h} |Q_h(s,\pi(s))-Q_h(s,\bpi(s))|\leq 2\resolution$ for arbitrary $h \in [\h]$.

    Recall that $\bar\pi=\bar\pi_1\circ\bar\pi_2...\circ\bar\pi_h$
    \begin{align*}
        &V_1^*(s_1) - V_1^{\bpi}(s_1) \leq V_1(s_1) - V_1^{\bpi}(s_1) \tag{Lemma \ref{lem:optQ_overestimation}; Overestimation}\\
        =& Q_1(s_1, \pi_1(s_1)) - Q_1^{\bar\pi}(s_1, \bpi_1(s_1)) \leq Q_1(s_1, \bpi_1(s_1)) - Q_1^{\bar\pi}(s_1, \bpi_1(s_1)) + 2\resolution \tag{Lemma \ref{lem:function_class_distance}}\\
        = & \EE_{s_1\sim d_1,a_1\sim\bpi_1}[\min\{\phi(s_1,a_1)w_1\trans+u_1(s_1,a_1),H\}-P_2V_2(s_1,a_1)+P_2V_2(s_1,a_1)-P_2V_2^{\bpi}(s_1,a_1)] + 2\resolution\\
        \leq& \EE_{s_1\sim d_1,a_1\sim\bpi_1,s_2\sim P_2(\cdot|s_1,a_1)}[V_2(s_2)-V_2^{\bpi}(s_2)]+2\betap\EE_{s_1\sim d_1,a_1\sim\bpi_1}[\|\phi(s_1,a_1)\|_{\Sigma_1^{-1}}]+2\resolution\\
        \leq&...\\
        \leq& 2\betap\EE_{s_1\sim d_1,a_1,s_2,...,s_{\h-1},a_{\h-1} \sim \bpi}[\sum_{\tdh=1}^{\h-1} \|\phi(s_\tdh,a_\tdh)\|_{\Sigma_\tdh^{-1}}] + 2(\h-1)\resolution + \EE_{s_h,a_h\sim \bpi}[V_h(s_h)-V_h^{\bpi}(s_h)] \\
        \leq&2\betap \xi + 2\h\resolution\leq 3\betap\xi \tag{Condition \ref{cond:induction_condition}}
    \end{align*}
\end{proof}

\begin{lemma}[Bias of Linear Regression in Alg \ref{alg:approx_random_feature_matrix}]\label{lem:linear_regression_bias}
    During the running of Alg \ref{alg:DERL_with_arbitrary_policies}, on the event of $\cE$ in Theorem \ref{thm:concentration_lemma_for_discretization}, which holds with probability $1-\delta/2$, for arbitrary $h\in[H], \tdh\in [h-1]$, and arbitrary $\pi$, $\hat w^\pi_\tdh$ and $\hat V^\pi_{\tdh+1}$ occurs in Alg \ref{alg:approx_random_feature_matrix}, we have:
    \begin{align*}
        |\phi(s,a)\trans \hat w^\pi_h-\sum_{s'\in\cS}P_h(s'|s,a)\hat V^\pi_{\tdh+1}(s')| \leq \betap \|\phi(s,a)\|_{\Sigma_h^{-1}}
    \end{align*}
    where $\betap$ is the same as Lemma \ref{lem:optQ_overestimation}.
\end{lemma}
The proofs for the above Lemma is based on Theorem \ref{thm:concentration_lemma_for_discretization} and Claim \ref{claim:value_function_classes} and is similar to Lemma 3.1 in \citep{wang2020rewardfree}, so we omit it here.

\begin{lemma}[Policy Evaluation Error in Alg \ref{alg:approx_random_feature_matrix}]\label{lem:policy_evaluation_error}
    During the running of Algorithm \ref{alg:DERL_with_arbitrary_policies}, on the events of $\mathcal{E}$ in Theorem \ref{thm:concentration_lemma_for_discretization}, which holds with probability $1-\delta/2$, and on the induction condition in \ref{cond:induction_condition}, for arbitrary $\h\in[H]$ and $i,j\in[d]$, and arbitrary policy $\pi$ and their evaluation results $\hat V^\pi$ Algorithm \ref{alg:approx_random_feature_matrix}, we have:
    \begin{align*}
        |V^\pi(s_1;\tilde{R}^{ij}) - \hat V_1^\pi(s_1)| \leq \betap\xi
    \end{align*}
    where we use $\tilde{R}^{ij}$ to denote the reward function used in Algorithm \ref{alg:approx_random_feature_matrix}.
\end{lemma}
\begin{proof}
    As a result of Lemma.\ref{lem:linear_regression_bias} , for arbitrary $\tilde{R}^{ij}$, we have:
    \begin{align*}
        |\hat V_1^\pi(s_1) - V_1^\pi(s_1;\tilde{R}^{ij})|=&|\hat Q_1^\pi(s_1,a_1) - Q_1^\pi(s_1,a_1;\tilde{R}^{ij})|\\
        =& |\phi(s_1,a_1)\trans \hat w^\pi_1 - \sum_{s_2}P_h(s_2|s_1,a_1)V^\pi(s_2;\tilde{R}^{ij})|\\
         \leq& |\phi(s_1,a_1)\trans \hat w^\pi_1 - \sum_{s_2}P_h(s_2|s_1,a_1)\hat V_2^\pi(s_2)|+ \EE_{\pi}|\hat V_2^\pi(s_2)-Q^\pi_2(s_2,\pi(s_2);\tilde{R}^{ij})|\\
         \leq& \betap \|\phi(s_1,a_1)\|_{\Sigma_1^{-1}} + \EE_{s_2}|V^\pi(s_2)-Q^\pi_2(s_2,\pi(s_2);\tilde{R}^{ij})| \\
         &...\\
         \leq& \betap\EE_{s_1,a_1,...,s_{\h-1},a_{\h-1}\sim \pi}[\sum_{t=1}^{\h-1}\|\phi(s_t,a_t)\|_{\Sigma_t^{-1}}]\\
         \leq& \betap\frac{\h-1}{H}\xi \leq \betap\xi
    \end{align*}
\end{proof}

\subsection{Main Theorem and Proof}
Now, we restate Theorem \ref{thm:arbitary_policy_informal} in a formal version below:

\begin{theorem}[Formal Version of Theorem \ref{thm:arbitary_policy_informal}]\label{thm:arbitary_policy_formal}
    For arbitrary $0<\epsilon,\delta < 1$, there exists absolute constants $c_i,c_\betap, c_\beta$ and $c_N$, such that by choosing \fix{($\betap$ used in Alg. \ref{alg:optQ_solution} and \ref{alg:approx_random_feature_matrix}, and $\beta$ used in Alg. \ref{alg:DE_rl_layer_by_layer_reward_free_planning})}\footnote{\fix{The ICLR 2022 version (and also the second version in arxiv) omitted that the bonus coefficients $\beta$ in Alg. \ref{alg:DE_rl_layer_by_layer_reward_free_planning} and $\betap$ in Alg. \ref{alg:optQ_solution} and \ref{alg:approx_random_feature_matrix} should have different dependence on $H$, which resulted in incorrect dependence on $H$ in the sample complexity upper bound. We thank Dan Qiao for calling this to our attention.}}
    \begin{align*}
        i_{\max}=c_i\frac{d}{\nu_{\min}^4}&\log\frac{d}{\nu_{\min}}, \quad\betap=c_\betap d\sqrt{\log\frac{dH}{\epsilon\delta\nu_{\min}}},\quad \beta=c_\beta dH\sqrt{\log\frac{dH}{\epsilon\delta\nu_{\min}}}\\
         N=&c \Big(\frac{H^4d^3}{\epsilon^2 \nu^2_{\min}}+\frac{H^2d^7}{\nu^{14}_{\min}}\Big)\log^2 \frac{dH}{\epsilon\delta\nu_{\min}}, \quad \epsilon_0=\frac{1}{N}.
    \end{align*}
    with probability $1-\delta$, 
    after $K=H$ deployments, by running Alg \ref{alg:DE_rl_layer_by_layer_reward_free_planning} with the collected dataset $D=\{D_1,...,D_H\}$ and arbitrary $r$ satisfying the linear assumption in \ref{assump:linear_MDP} (in Line of Algorithm \ref{alg:DERL_with_arbitrary_policies}), we will obtain a policy $\hat\pi$ such that $V_1^{\pi}(s_1;r) \geq V_1^{\pi^*}(s_1;r) - \epsilon$.

    As additional guarantees, after $h$ deployments, by running Alg \ref{alg:DE_rl_layer_by_layer_reward_free_planning} with the collected dataset $\{D_1,D_2...,D_h\}$ and reward function $r$, we will obtain a policy $\pi_{|h}$ which is $\epsilon$-optimal in the MDP truncated at step $h$.
\end{theorem}
\begin{proof}
    We will use $\Sigma_{h,i}:=2I+\sum_{j=1}^{i-1} \EE_{\pi_j}[\phi(s_h,a_h)\phi(s_h,a_h)\trans]$ to denote the matrix which $\tilde \Sigma_{h,i}$ approximates in Alg. \ref{alg:DERL_with_arbitrary_policies}, and use $R_{h,i}:=\sqrt{\phi(\cdot,\cdot)\trans\Sigma_{h,i}^{-1}\phi(\cdot,\cdot)}$ to denote the reward function used in Alg \ref{alg:optQ_solution} if the covariance matrix estimation is perfect (i.e. $\tilde{\Sigma}_{h,i}=\Sigma_{h,i}$).
    
    The proof consists of three steps. In step 1, we try to show that the inner loop of Alg \ref{alg:DERL_with_arbitrary_policies} will terminate and $\Pi_h$ will contain a set of exploratory policies. In step 2, we will analyze the samples generated by a mixture of policies in $\Pi_h$. In the last step, we determine the choice of hyper-parameters and fill the gaps of pre-assumed constraints and induction conditions.

    \paragraph{Step 1: Exploration Ability for Policies in $\Pi_h$}
    In the inner loop (line 5 - 12) in Algorithm \ref{alg:DERL_with_arbitrary_policies}, our goal is to find a set of policies $\Pi_h$, such that if the algorithm stops at iteration $i$, the following uncertainty measure is as small as possible
    \begin{align}
        V^*_{h,i+1}(s_1;R_{h,i}):=\max_\pi \EE_\pi[\|\phi(s_h,a_h)\|_{\Sigma_{h,i}^{-1}}]\label{eq:uncertainty_definition}
    \end{align}
    To achieve this goal, we repeatedly use Alg \ref{alg:approx_random_feature_matrix} to estimate the covariance matrix of the policy and append it to $\tSigma_{h,i}$ as an approximation of $\Sigma_{h,i}$, and use Alg \ref{alg:optQ_solution} to find a near-optimal policy to maximizing the uncertainty-based reward function $\tR$, by sampling trajectories with which we can reduce the uncertainty $Q^*_{h,i}$ in Eq.\eqref{eq:uncertainty_definition}.


    First, we take a look at the estimation error of the accumulative covariance matrix when running Algorithm \ref{alg:DERL_with_arbitrary_policies}. On the conditions in Lemma \ref{lem:policy_evaluation_error}, we can bound the elementwise estimation error of $\Sigma_{h,i}$:
    \begin{align*}
        |(\tSigma_{h,i})_{jk} - (\Sigma_{h,i})_{jk}| \leq i\cdot\betap\xi,\quad\forall j,k\in[d]
    \end{align*}
    As a result of Lemma \ref{lem:matrix_perturbation}, we have:
    \begin{align*}
        |\tR_{h,i}(s_h,a_h) - R_{h,i}(s_h,a_h)| =& |\sqrt{\phi\trans\tSigma_{h,i}^{-1}\phi}-\sqrt{\phi\trans\Sigma_{h,i}^{-1}\phi}| \\
        \leq & \sqrt{\frac{i\cdot d\betap\xi}{1-i\cdot d\betap\xi}} \leq \sqrt{\frac{i_{\max} d\betap\xi}{1-i_{\max} d\betap\xi}}\\
        \leq& \frac{\nu^2_{\min}}{8} \numberthis\label{eq:reward_estimation_error}
    \end{align*}
    where the last but two step is because we at most repeat it $i_{\max}$ iterations at each layer $h$, and we introduce the following constraint for $\xi$ during the derivation, to make sure the bias is small and all the iterms occurs in the derivation is well defined:
    \begin{align}
        \xi \leq \frac{\nu^4_{\min}}{32i_{\max}d\betap} \label{eq:constraints_on_xi_2}
    \end{align}

    Next, we want to find a good choice of $i_{\max}$ to make sure $V_{h,i+1}$ will not always be large and the for-loop will break for some $i\leq i_{\max}$. We first provide an upper bound for $V_{h,i+1}$:
    \begin{align*}
        V_{h,i+1}(s_1) \leq& V^{\pi_{h,i+1}}(s_1;\tR_{h,i}) + \betap\xi\tag{Lemma \ref{lem:optQ_overestimation}}\\
        \leq&V^{\bpi_{h,i+1}}(s_1;\tR_{h,i}) + 4\betap\xi\tag{$V^{\pi} - V^{\bar\pi} \leq V^* - V^{\bar\pi}$; Lemma \ref{lem:bais_accumulation}}\\
        \leq&V^{\bpi_{h,i+1}}(s_1;R_{h,i}) + 4\betap\xi+\frac{\nu^2_{\min}}{8} \tag{bias of reward} \\
        \leq& V^{\bpi_{h,i+1}}(s_1;R_{h,i}) +\frac{\nu^2_{\min}}{4} \tag{Constraints on $\xi$ in Eq.\eqref{eq:constraints_on_xi_2}}
    \end{align*}
    Next, we try to show that $V^{\bpi_{h,i+1}}(s_1;R_{h,i})$ can not always be large.
    According to Elliptical Potential Lemma in Lemma \ref{lem:elliptical_potential_lemma}, we have:
    \begin{align*}
        \sum_{i=1}^{i_{\max}}V^{\bpi_{h,i+1}}(s_1;R_{h,i}) =& \sum_{i=1}^{i_{\max}}\EE_{\bpi_{h,i+1}}[\|\phi(s_h,a_h)\|_{\Sigma_{h,i}^{-1}}] \\
        \leq& \sum_{i=1}^{i_{\max}}\sqrt{\EE_{\bpi_{h,i+1}}[\|\phi(s_h,a_h)\|^2_{\Sigma_{h,i}^{-1}}]} \\
        \leq& \sqrt{i_{\max}\sum_{i=1}^{i_{\max}}\EE_{\bpi_{h,i+1}}[\|\phi(s_h,a_h)\|^2_{\Sigma_{h,i}^{-1}}]} \\
        =& \sqrt{i_{\max}\sum_{i=1}^{i_{\max}}Tr(\EE_{\bpi_{h,i+1}}[\phi(s_h,a_h)\phi(s_h,a_h)\trans]\Sigma_{h,i}^{-1})} \\
        \leq& \sqrt{2i_{\max}d\log(1+i_{\max}/d)} 
    \end{align*}
    where in the last step, we use the definition of $\Sigma_{h,i}$.
    Therefore,
    \begin{align*}
        \min_i V^{\bpi_{h,i+1}}(s_1;R_{h,i}) \leq \frac{1}{i_{\max}} \sum_{i=1}^{i_{\max}}V^{\bpi_{h,i+1}}(s_1;R_{h,i}) \leq \sqrt{2\frac{d\log(1+i_{\max}/d)}{i_{\max}}} 
    \end{align*}
    In order to guarantee $\min_i V^{\bpi_{h,i+1}}(s_1;R_{h,i})  \leq \nu^2_{\min}/8$, we require:
    \begin{align*}
        \sqrt{2\frac{d\log(1+i_{\max}/d)}{i_{\max}}} \leq \nu^2_{\min}/8
    \end{align*}
    which can be satisfied by:
    \begin{align}
        i_{\max} = c_i \frac{d}{\nu^4_{\min}}\log \frac{d}{\nu_{\min}} \label{eq:choice_of_imax}
    \end{align}
    for some absolute constant $c_i$. 

    Combining the above results, we can conclude that the inner loop in Alg \ref{alg:DERL_with_arbitrary_policies} will break at some $i < i_{\max}$, such that $V_{h,i+1}\leq 3\nu^2_{\min}/8$, and guarantee that:
    \begin{align*}
        \max_\pi\EE_{\pi}[\|\phi(s_h,a_h)\|_{\Sigma^{-1}_{h,i}}]:=&V^*_{h,i+1}(s_1;R_{h,i})\\
        \leq& V^*_{h,i+1}(s_1;\tR_{h,i}) + \frac{\nu^2_{\min}}{8} \tag{reward estimation error Eq.\eqref{eq:reward_estimation_error}} \\
        \leq& V_{h,i+1}(s_1) + \frac{\nu^2_{\min}}{8}\tag{Overestimation in Lemma \ref{lem:optQ_overestimation}} \\
        \leq& V^{\bpi_{h,i+1}}(s_1;R_{h,i}) +\frac{\nu^2_{\min}}{4}+ \frac{\nu^2_{\min}}{8} \\
        \leq& \frac{\nu^2_{\min}}{2}
    \end{align*}

    \paragraph{Step 2: Policy Deployment and Concentration Error}
    For uniform mixture policy $\pi_{h,mix}:=Unif(\Pi_h)$, by applying Lemma \ref{lem:matrix_dominate}, Lemma \ref{lem:blow_up_with_large_data} and the results above, we must have:
    \begin{align*}
        &\max_\pi \EE_\pi[\sqrt{\phi(s_h,a_h)\trans(I+N\EE_{\pi_{h,mix}}[\phi\phi\trans])^{-1}\phi(s_h,a_h)}]\\
        =&\max_\pi \EE_\pi[\sqrt{\phi(s_h,a_h)\trans(I+\frac{N}{|\Pi_h|}|\Pi_h|\EE_{\pi_{h,mix}}[\phi\phi\trans])^{-1}\phi(s_h,a_h)}]\\
        \leq& \sqrt{\frac{2}{1+N/|\Pi_h|}\max_\pi \EE_\pi[\sqrt{\phi(s_h,a_h)\trans(I+|\Pi_h|\EE_{\pi_{h,mix}}[\phi\phi\trans])^{-1}\phi(s_h,a_h)}] }\\
        \leq& \sqrt{\frac{1}{1+N/i_{\max}}} \nu_{\min} \leq \sqrt{\frac{i_{\max}}{N}} \nu_{\min}
    \end{align*}
    and this is the motivation of breaking criterion in Line \ref{line:break_criterion} in Alg \ref{alg:DERL_with_arbitrary_policies}.

   In the following, we will use $\Sigma^-_h :=\Sigma_h - I= \sum_{n=1}^N \phi(s_{h,n}a_{h,n})\phi(s_{h,n},a_{h,n})\trans$ to denote the matrix of sampled feature without regularization terms, according to Lemma \ref{lem:matrix_concentration}, with probability $1-\delta / 2$, we have:
    \begin{align*}
        \|\frac{1}{N}\sigma_{\max}(N\EE_{\pi_{h,mix}}[\phi\phi\trans]-\Sigma^-_h)\| \leq \frac{4}{\sqrt{N}} \log\frac{8dH}{\delta},\quad\forall h\in[H]
    \end{align*}
    Follow the same steps in the proof of Lemma \ref{lem:blow_up_with_large_data}, we know that 
    $$
    \sigma_{\min}(N\EE_{\pi_{h,mix}}[\phi\phi\trans])=\frac{N}{|\Pi_h|}\sigma_{\min}(|\Pi_h|\EE_{\pi_{h,mix}}[\phi\phi\trans])\geq \frac{N}{|\Pi_h|} \geq \frac{N}{i_{\max}}.
    $$
    As a result,
    \begin{align*}
        \min_{x:\|x\|=1} x^\top\Sigma_h x =& \min_{x:\|x\|=1} x^\top(I+N\EE_{\pi_{h,mix}}[\phi\phi\trans]) x + x^\top(N\EE_{\pi_{h,mix}}[\phi\phi\trans]-\Sigma^-_h)x \\
        \geq& \sigma_{\min}(I+N\EE_{\pi_{h,mix}}[\phi\phi\trans]) - \sigma_{\max}(N\EE_{\pi_{h,mix}}[\phi\phi\trans]-\Sigma^-_h)\\
        \geq& 1+(\frac{N}{2i_{\max}} - 4\sqrt{N}\log \frac{8dH}{\delta})
    \end{align*}
    which implies that, as long as 
    \begin{align}
        N \geq 16 i^2_{\max}\log^2\frac{8dH}{\delta}\label{eq:constraint_N_concentration_error_0}
    \end{align}
    we have
    \begin{align*}
        \sigma_{\max}(\Sigma_h^{-1}) \leq \frac{1}{1+N/2i_{\max} - \sqrt{N}\log \frac{8dH}{\delta}} \leq \frac{4i_{\max}}{N}
    \end{align*}
    Therefore, for arbitrary $\pi$, we have:
    \begin{align*}
        &|\EE_\pi[\sqrt{\phi(s_h,a_h)\trans(I+N\EE_{\pi_{h,mix}}[\phi\phi\trans])^{-1}\phi(s_h,a_h)}] - \EE_\pi[\sqrt{\phi(s_h,a_h)\trans(\Sigma_h)^{-1}\phi(s_h,a_h)}]|\\
        \leq& \EE_\pi[\sqrt{|\phi(s_h,a_h)\trans(I+N\EE_{\pi_{h,mix}}[\phi\phi\trans])^{-1}\phi(s_h,a_h)-\phi(s_h,a_h)\trans(\Sigma_h)^{-1}\phi(s_h,a_h)|}] \\
        \leq& \EE_\pi[\sqrt{|\phi(s_h,a_h)\trans\Big((I+N\EE_{\pi_{h,mix}}[\phi\phi\trans])^{-1}-(\Sigma_h)^{-1}\Big)\phi(s_h,a_h)|}]\\
        =& \EE_\pi[\sqrt{|\phi(s_h,a_h)\trans(I+N\EE_{\pi_{h,mix}}[\phi\phi\trans])^{-1}\big(N\EE_{\pi_{h,mix}}[\phi\phi\trans]-\Sigma^{-}_h\big)(\Sigma_h)^{-1}\phi(s_h,a_h)|}]\\
        \leq& \sqrt{\sigma_{\max}((I+N\EE_{\pi_{h,mix}}[\phi\phi\trans])^{-1})\frac{4i_{\max}}{N}\sigma_{\max}(N\EE_{\pi_{h,mix}}[\phi\phi\trans]-\Sigma^{-}_h)} \\
        \leq& \sqrt{\frac{1}{1+N/i_{\max}}\frac{16i_{\max}}{\sqrt{N}}\log\frac{8d}{\delta}}\tag{$i_{\max}\EE_{\pi_{h,mix}}[\phi\phi\trans]  \succeq |\Pi_h|\EE_{\pi_{h,mix}}[\phi\phi\trans] \succeq 1$}\\
        \leq& \frac{4i_{\max}}{N^{3/4}}\sqrt{\log\frac{8dH}{\delta}}
    \end{align*} 
    As a result,
    \begin{align*}
        &\max_\pi \EE_\pi[\sqrt{\phi(s_h,a_h)\trans(\Sigma_h)^{-1}\phi(s_h,a_h)}]\\
        \leq & \max_\pi \EE_\pi[\sqrt{\phi(s_h,a_h)\trans(I+N\EE_{\pi_{h,mix}}[\phi\phi\trans])^{-1}\phi(s_h,a_h)}] + \frac{4i_{\max}}{N^{3/4}}\sqrt{\log\frac{8dH}{\delta}} \\
        \leq& \sqrt{\frac{i_{\max}}{N}} \nu_{\min} + \frac{4i_{\max}}{N^{3/4}}\sqrt{\log\frac{8dH}{\delta}}
    \end{align*}
    In order to make sure the induction conditions holds, we need 
    \begin{align*}
        \sqrt{\frac{i_{\max}}{N}} \nu_{\min} + \frac{4i_{\max}}{N^{3/4}}\sqrt{\log\frac{8dH}{\delta}} \leq \xi / H
    \end{align*}
    As long as we tighten the constraint in \ref{eq:constraint_N_concentration_error_0} to:
    \begin{align}
        N \geq 256 \frac{i^2_{\max}}{\nu^4_{\min}}\log^2\frac{8dH}{\delta}=\tilde{O}(\frac{d^2}{\nu^{12}_{\min}})\label{eq:constraint_N_concentration_error}
    \end{align}
    the induction conditions can be satisfied when 
    \begin{align*}
        2\sqrt{\frac{i_{\max}}{N}} \nu_{\min}\leq \xi / 2H 
    \end{align*}
    or equivalently,
    \begin{align*}
        N \geq \frac{16H^2\nu^2_{\min}i_{\max}}{\xi^2} = O(\frac{H^2d}{\xi^2 \nu^2_{\min}} )
    \end{align*}
    \paragraph{Step 3: Determine Hyper-parameters}
    \paragraph{(1) Resolution $\resolution$}
    Recall that we still have a constraint for $Z_h$ in \eqref{eq:constraint_resolution_1}
    \begin{align*}
        I+\Sigma_R \succeq \frac{\resolution^2}{4}I,\quad \frac{4}{\resolution^2}I \succeq 2I+\Sigma_R 
    \end{align*} 
    Since we already determined $i_{\max}$ in Eq.\eqref{eq:choice_of_imax}, also recall our constraints on $\xi$ in \eqref{eq:constraints_on_xi_2} the above constraints for $\resolution$ can be satisfied as long as:
    \begin{align}
        \resolution \leq \sqrt{\frac{1}{i_{\max}}} \label{eq:constaint_resolution_1_final}
    \end{align}
    Combining all the constraints of $\resolution$, including \eqref{eq:constaint_resolution_0}, \eqref{eq:constraint_resolution_2}, \eqref{eq:constraint_resolution_3} and \eqref{eq:constaint_resolution_1_final}, we conclude that:
    \begin{align*}
        \resolution \leq \min\{\frac{1}{N}, \frac{\betap}{\sqrt{N+1}}, \frac{1}{\sqrt{i_{\max}}}, \frac{\betap\xi}{H}\} = \frac{1}{N}
    \end{align*}
    \paragraph{(2) Induction error $\xi$}
    Besides the constraint in \eqref{eq:constraints_on_xi_2}, we need another one to make sure the quality of the final output policy. By applying Lemma \ref{lem:analysis_of_reward_free_planning_alg} for planning algorithm Alg. \ref{alg:DERL_with_arbitrary_policies}, if the induction condition \eqref{cond:induction_condition} holds till $h\in[H]$, Alg. \ref{alg:DE_rl_layer_by_layer_reward_free_planning} will return us a policy $\hat\pi$ such that:
    \begin{align*}
        V^* - V^{\pi} \leq 2\beta\max_\pi \EE_\pi[\sum_{\tdh=1}^H \|\phi(s_\tdh,a_\tdh)\|_{\Sigma_\tdh^{-1}}]\leq 2\beta\xi
    \end{align*}
    To make sure $V^* - V^{\hat\pi} \leq \epsilon$, we require $\xi \leq \frac{\epsilon}{2\beta}$, which implies that
    \begin{align*}
        \xi \leq \min\{\frac{\nu^4_{\min}}{32i_{\max}d\betap},\frac{\epsilon}{2\beta}\}
    \end{align*}
    \paragraph{Choice of $N$, $\beta'$ and $\beta$}
    Since $\epsilon_0 = \frac{1}{N}$, by plugging it into Eq.\eqref{eq:choice_of_beta}, we may choose $\betap$ to be:
    \begin{align*}
        \betap = c_\betap'' d \sqrt{\log\frac{dHN}{\delta}}
    \end{align*}
    The discussion about the choice of $\beta$ is similar to what we did for Alg. \ref{alg:DE_rl_layer_by_layer_reward_free_exploration} and \ref{alg:DE_rl_layer_by_layer_reward_free_planning}, and we can choose:
    \begin{align*}
        \beta = c_\beta'' dH \sqrt{\log\frac{dHN}{\delta}}
    \end{align*}
    Now, we are ready to compute $N$. When $\xi=\frac{\epsilon}{2\beta} \leq  \frac{\nu^4_{\min}}{32i_{\max}d\betap}$, we have:
    \begin{align*}
        N = O(\frac{H^2d}{\xi^2 \nu^2_{\min}})=O(\frac{H^4d^3}{\epsilon^2 \nu^2_{\min}})\log \frac{dH}{\epsilon\delta\nu_{\min}}
    \end{align*}
    and otherwise, we have:
    \begin{align*}
        N = O(\frac{H^2d}{\xi^2 \nu^2_{\min}})=O(\frac{H^2d^7}{\nu^{14}_{\min}})\log \frac{dH}{\epsilon\delta\nu_{\min}}
    \end{align*}
    Combining the additional constraint to control the concentration error in Eq.\eqref{eq:constraint_N_concentration_error}, the total number of complexity would at the level:
    \begin{align*}
        N \geq c \Big(\frac{H^4d^3}{\epsilon^2 \nu^2_{\min}}+\frac{H^2d^7}{\nu^{14}_{\min}}\Big)\log^2 \frac{dH}{\epsilon\delta\nu_{\min}}
    \end{align*}
    and therefore, 
    \begin{align*}
        \beta = c_\beta dH \sqrt{\log\frac{dH}{\delta\epsilon\nu_{\min}}},\quad \betap = c_\betap d \sqrt{\log\frac{dH}{\delta\epsilon\nu_{\min}}}
    \end{align*}
    for some $c_\beta$ and $c_\betap$.
    \paragraph{Near-Optimal Guarantee}
    Under the events in Theorem \ref{thm:concentration_lemma_for_discretization}, considering the failure rate of concentration inequality in Step 2, we can conclude that the induction condition holds for $h\in[H]$ with probability $1-\delta$. Combining our discussion about choice of $\xi$ above, the probability that Alg \ref{alg:DERL_with_arbitrary_policies} will return us an $\epsilon$-optimal policy would be $1-\delta$.

    The additional guarantee in Theorem \ref{thm:arbitary_policy_formal} can be directly obtained by considering the induction condition at layer $h\in[H]$.
\end{proof}

\subsection{Technical Lemma}
\begin{lemma}[Matrix Bernstein Theorem (Theorem 6.1.1 in \citep{tropp2015introduction})]\label{lem:matrix_concentration}
    Consider a finite sequence $\{\BoldS_k\}$ of independent, random matrices with common dimension $d_1 \times d_2$. Assume that
    \begin{align*}
        \EE \BoldS_k = 0~~and~~\|\BoldS_k\| \leq L~~for~each~index~k
    \end{align*}
    Introduce the random matrix
    $$
    \Z = \sum_k \BoldS_k
    $$
    Let $v(\Z)$ be the matrix variance statistic of the sum:
    \begin{align*}
        v(\Z) =& \max\{\|\EE(\Z\Z^*)\|, \|\EE(\Z^*\Z)\|\}\\
        =&\max\{\|\sum_k \EE(\BoldS_k\BoldS_k^*)\|, \|\sum_k \EE(\BoldS_k^*\BoldS_k)\|\}
    \end{align*}
    Then,
    \begin{align*}
        \EE\|\Z\|\leq \sqrt{2v(\Z)\log(d_1 + d_2)} + \frac{1}{3}L\log(d_1+d_2)
    \end{align*}
    Furthermore, for all $t\geq 0$
    \begin{align*}
        \mathbb{P}\{\|\Z\|\geq t\}\leq (d_1 + d_2)\exp\Big(\frac{-t^2/2}{v(\Z)+Lt/3}\Big)
    \end{align*}
\end{lemma}

\begin{lemma}[Matrix Perturbation]\label{lem:matrix_perturbation}
    Given a positive definite matrix $A \succ I$ and $\Delta$ satisfying $|\Delta_{ij}| \leq \epsilon < 1/d$, define matrix $A_+=A+\Delta$, then for arbitrary $\phi\in \mR^d$ with $\|\phi\|\leq 1$, we have:
    \begin{align*}
        A_+ \succ 0,\quad |\phi\trans(A_+-A)\phi| \leq d\epsilon,\quad |\phi\trans(A_+^{-1}-A^{-1})\phi|\leq \frac{d\epsilon}{1-d\epsilon}
    \end{align*}
    which implies that
    \begin{align*}
        \|A_+\| \leq \|A\| + \|\Delta\| \leq \|A\|+\|\Delta\|_F \leq \|A\| + d\epsilon,\quad \|A_+^{-1}\|\geq \|A^{-1}\| - \frac{d\epsilon}{1-d\epsilon}
    \end{align*}
    Moreover,
    \begin{align*}
        |\|\phi\|_{A^{-1}_+}-\|\phi\|_{A^{-1}}| \leq \sqrt{|\|\phi\|^2_{A^{-1}_+}-\|\phi\|^2_{A^{-1}}|} \leq \sqrt{\frac{d\epsilon}{1-d\epsilon}}
    \end{align*}
\end{lemma}
\begin{proof}
    First of all, easy to see that
    \begin{align*}\sigma_{\max} (\Delta) \leq \|\Delta\|_F \leq d\epsilon 
    \end{align*}
    and therefore we have 
    $$
    \sigma_{\min}(A_+)=\min_{x:\|x\|=1} x\trans Ax + x\trans \Delta x > 1-d\epsilon > 0.
    $$
    where we use  $\sigma_{\min}$ and $\sigma_{\max}$ to denote the smallest and the largest sigular value, respectively, and use $\|\cdot\|_F$ to refer to the Frobenius norm.
    Therefore,
    \begin{align*}
        |\phi\trans(A_+-A)\phi| = |\phi\trans\Delta\phi| \leq d\epsilon
    \end{align*}
    and
    \begin{align*}
        |\phi\trans(A_+^{-1}-A^{-1})\phi| = |\phi\trans A_+^{-1}(A_+-A)A^{-1}\phi|\leq \sigma_{\max}(A_+^{-1})\sigma_{\max}(A_+-A)\sigma_{\max}(A^{-1}) \leq \frac{d\epsilon}{1-d\epsilon}
    \end{align*}
    Moreover,
    \begin{align*}
        |\|\phi\|_{A^{-1}_+}-\|\phi\|_{A^{-1}}| \leq \sqrt{|\|\phi\|_{A^{-1}_+}-\|\phi\|_{A^{-1}}|\cdot|\|\phi\|_{A^{-1}_+}+\|\phi\|_{A^{-1}}|} = \sqrt{|\phi\trans(A_+^{-1}-A^{-1})\phi|}\leq \sqrt{\frac{d\epsilon}{1-d\epsilon}}
    \end{align*}
\end{proof}

Next, we will try to prove that, with a proper choice of $N$, Algorithm \ref{alg:DERL_with_arbitrary_policies} will explore layer $h$ to satisfy the recursive induction condition.
\begin{lemma}[Random Matrix Estimation Error]\label{lem:random_matrix_est_error}
    Denote $\hatcov^\pi_h=\EE_\pi[\phi\phi\trans]$. Based on the same induction condition \ref{eq:induction_condition}, we have:
    \begin{align*}
        |\|\phi\|_{(\hatcov^\pi)^{-1}}-\|\phi\|_{(\hat\hatcov^\pi)^{-1}}|\leq \sqrt{\frac{d\epsilon}{1-d\epsilon}},\quad \forall \|\phi\| \leq 1
    \end{align*}
\end{lemma}
\begin{proof}
    Based on Lemma \ref{lem:policy_evaluation_error}, we have:
    \begin{align*}
        |\hatcov^\pi_{ij} - \hat\hatcov^\pi_{ij}| \leq \frac{h-1}{H}\epsilon \leq \epsilon
    \end{align*}
    and as a result of Lemma \ref{lem:matrix_perturbation}, we finish the proof.
\end{proof}

\begin{lemma}\label{lem:matrix_dominate}
    Given a matrix $A \succeq \lambda I$ with $\lambda > 0$, and $\phi$ satisfies $\|\phi\|\leq 1$, then we have:
    \begin{align*}
        \phi\trans (cI+nA)^{-1} \phi \leq \frac{\lambda+c}{\lambda n+c}\phi\trans(cI+A)^{-1}\phi,\quad \forall n > 1, c > 0
    \end{align*}
\end{lemma}
\begin{proof}
    Because $A \succeq \lambda I$, we have 
    \begin{align*}
    cI+nA =& cI + \frac{c(n-1)}{\lambda+c}A+(n-\frac{c(n-1)}{\lambda+c})A \succeq \Big(1 + \frac{\lambda (n-1)}{\lambda+c}\Big)cI+(n-\frac{c(n-1)}{\lambda+c})A\\
    =&\frac{\lambda n+c}{\lambda+c}\Big(cI+A\Big)
    \end{align*}
    Therefore,
    \begin{align*}
        \phi\trans (cI+nA)^{-1} \phi \leq \frac{\lambda+c}{\lambda n+c}\phi\trans(cI+A)\phi
    \end{align*}
\end{proof}

\begin{lemma}\label{lem:blow_up_with_large_data}
    Given a matrix $A \succeq 0$, suppose $\max_\pi \EE_\pi[\|\phi\|_{(cI+A)^{-1}}] \leq \tilde\epsilon \leq \nu^2_{\min}/(2\sqrt{c})$, where $c>0$ is a constant, where $\nu_{\min}$ is defined in Definition \ref{def:reachability_coefficient}, we have:
    \begin{align*}
        \max_\pi \EE_\pi[\|\phi\|_{(cI+nA)^{-1}}] \leq \sqrt{\frac{c+1}{\sqrt{c}(c+n)}\tilde\epsilon}
    \end{align*}
\end{lemma}
\begin{proof}
    Because $\|\phi\|_{(cI+A)^{-1}} \leq \|\phi\|_{(cI)^{-1}}\leq 1/\sqrt{c}$, we must have:
    \begin{align*}
        \max_\pi Tr((cI+A)^{-1}\EE_\pi[\phi\phi\trans])=\max_\pi \EE_\pi[\|\phi\|^2_{(cI+A)^{-1}}] \leq \frac{1}{\sqrt c}\max_\pi \EE_\pi[\|\phi\|_{(cI+A)^{-1}}] \leq \frac{\tilde\epsilon}{\sqrt c}
    \end{align*}
    Consider the SVD of $A=U\trans\Sigma U$ with $\Sigma=(\sigma_{ii})_{i=1,...,d}$ and $U=[u_1,u_2...,u_d]$, then we have:
    \begin{align*}
        \forall \pi,\quad \frac{\tilde\epsilon}{\sqrt c} \geq Tr((cI+A)^{-1}\EE_\pi[\phi\phi\trans]) = Tr((cI+\Sigma)^{-1}U\trans\EE_\pi[\phi\phi\trans]U) = \sum_{i=1}^d \frac{\EE_\pi[|\phi\trans u_i|^2]}{c+\sigma_{ii}}.
    \end{align*}
    According to the Definition \ref{def:reachability_coefficient}, we have:
    \begin{align*}
        \frac{\tilde\epsilon}{\sqrt c} \geq \frac{\max_\pi\EE_\pi[|\phi\trans u_i|^2]}{c+\sigma_{ii}} \geq \frac{\nu^2_{\min}}{c+\sigma_{ii}},\quad \forall i \in [d]
    \end{align*}
    which implies that
    \begin{align*}
        \sigma_{ii} \geq \frac{\sqrt c\nu^2_{\min}}{\tilde\epsilon} - c \geq c,\quad \forall i \in [d]
    \end{align*}
    By applying Lemma \ref{lem:matrix_dominate} and assign $\lambda=c$, we have:
    \begin{align*}
        \max_\pi \EE_\pi[\|\phi\|_{(cI+nA)^{-1}}] \leq \max_\pi \sqrt{\EE_\pi[\|\phi\|^2_{(cI+nA)^{-1}}]} \leq \max_\pi \sqrt{\frac{2c}{c+cn}\EE_\pi[\|\phi\|^2_{(cI+A)^{-1}}]} \leq \sqrt{\frac{2}{\sqrt{c}(1+n)}\tilde\epsilon}
    \end{align*}
\end{proof}

\subsection{More about our reachability coefficient}\label{sec:reachability_coefficient}
Recall the definition of reachability coefficient in \citep{zanette2020provably} is:
$$
\min_{h\in[H]}\min_{\|\theta\|=1}\max_\pi |\EE_{\pi}[\phi_{h}]\trans\theta|
$$
Easy to see that, for arbitrary $\theta$ with $\|\theta\|=1$, we have
\begin{align*}
    \max_\pi \sqrt{\EE_\pi[(\phi\trans_h\theta)^2]} \geq \max_\pi \EE_\pi[|\phi\trans_h\theta|]\geq \max_\pi |\EE_\pi[\phi\trans_h\theta]|
\end{align*}
Therefore,
\begin{align*}
    \nu_{\min}=\min_{h\in[H]}\nu_h = \min_{h\in[H]} \min_{\|\theta\|=1}\max_\pi \sqrt{\EE_\pi[(\phi\trans_h\theta)^2]} \geq \min_{h\in[H]} \min_{\|\theta\|=1} \max_\pi |\EE_\pi[\phi\trans_h\theta]|
\end{align*}
Besides, according to the min-max theorem, $\nu_h$ is also lower bounded by $\sqrt{\max_\pi\sigma_{\min}(\EE_\pi[\phi_h\phi_h\trans])}$, to see this,
\begin{align*}
    \max_\pi\sigma_{\min}(\EE_\pi[\phi_h\phi_h\trans]) = \max_\pi\min_{\|\theta\|=1}\theta\trans\EE_\pi[\phi_h\phi_h\trans]\theta = \max_\pi\min_{\|\theta\|=1}\EE_\pi[(\phi\trans\theta)^2]\leq \min_{\|\theta\|=1}\max_\pi\EE_\pi[(\phi\trans\theta)^2]
\end{align*}
In fact, the value of $\max_\pi\sigma_{\min}(\EE_\pi[\phi_h\phi_h\trans])$ is also related to the "Well-Explored Dataset" assumption in many previous literature in offline setting \citep{jin2021pessimism}, where it is assumed that there exists a behavior policy such that the minimum singular value of the covariance matrix is lower bounded. Therefore, we can conclude that our reachability coefficient $\nu_{\min}$ is also lower bounded by, e.g. $\underline{c}$ in Corollary 4.6 in \citep{jin2021pessimism}.

\section{Extended Deployment-Efficient RL Setting}\label{appx:discussion_on_Extended_DERL}

\subsection{Sample-Efficient DE-RL}
In applications such as recommendation systems, the value of $N$ cannot exceed the number of users our system serves during a period of time. Therefore, as an interesting extension to our framework, we can  revise the constraint (b) in  Definition~\ref{def:deployment_efficient_linear_mdp} and explicitly assign an upper bound for $N$.  Concretely, we may consider the following alternatives: (b') The sample size $N \leq d^{c_1}H^{c_2}\epsilon^{-c_3}\log^{c_4}\frac{dH}{\epsilon\delta}$, where $c_1,c_2,c_3,c_4>0$ are some constant fixed according to the real situation.
Under these revised constraints, the lower bound for $K$ may be different. 

In fact, given constraints in the form of $N \leq N_0$, our results in Section \ref{sec:towards_optimal_DERL} already implies an upper bound for $K$, since we can emulate 1 deployment of our algorithm that uses a large $N > N_0$ by deploying the same policy for $\lceil N/N_0 \rceil$ times. 
However, this may result in sub-optimal deployment complexity since we are not adaptively updating our policy within those $N/N_0$ deployments. 
It would be an interesting open problem to identify the fine-grained trade-off between $K$ and $N_0$ in such a setting.

\subsection{Safe DE-RL}
\paragraph{Monotonic Policy Improvement Constraint}
In many applications, improvement of service quality after policy update is highly desired, and can be incorporated in our formulation by adding an additional constraint into Def \ref{def:deployment_efficient_linear_mdp}:
\begin{align}
    (c) \quad J(\pi_{i+1}) \geq J(\pi_i) - \epsilon.\label{eq:safety_constraint}
\end{align}
Because we require the deployed policy has substantial probability to visit unfamiliar states so that the agent can identify the near-optimal policy as shown in Def~\ref{def:deployment_efficient_linear_mdp}-(a), we relax the strict policy improvement with an small budget term $\epsilon > 0$ for exploration.

\paragraph{Trade-off between Pessimism and Optimism}
\begin{algorithm}[ht]
    \For{k=1,2,...,K}{
        $D=\{D_1,D_2,...,D_k\}$\\
        $\pi_{k,\mathrm{pessim}}\gets {\rm PessimismBased\_OfflineAlgorithm}(D)$\\
        $\pi_{k,\mathrm{optim}}\gets {\rm OptimismBased\_BatchExplorationStrategy}(D)$\\ 
        // Mix policy in trajectory level, i.e. w.p. $1-\alpha$, $\tau\sim \pi_{\mathrm{pessim}}$; w.p. $\alpha$, $\tau\sim\pi_{\mathrm{optim}}$\\
        $\pi_k \gets (1-\alpha)\pi_{k,\mathrm{pessim}}+\alpha \pi_{k,\mathrm{optim}}$ \\
        $D_k = \{\tau_n \sim \pi_k,\forall n \in[N]\}$ \\
    }
\caption{Mixture Policy Strategy}\label{alg:Mixture_Policy_Strategy}
\end{algorithm}
The balance between satisfying two contradictory constraints: (a) and (c), implies that a proper algorithm should leverage both pessimism and optimism in face of uncertainty. In Algorithm \ref{alg:Mixture_Policy_Strategy}, we propose a simple but effective mixture policy srategy, where we treat pessimistic and optimistic algorithm as black boxes and mix the learned policies with a coefficient $\alpha$. One key property of the mixed policy is that:
\begin{property}[Policy improvement for mixture policy]
    \begin{align}
        J(\pi_k) - J(\pi_{k-1}) \geq J(\pi_{k,\mathrm{pessim}}) - J(\pi_{k-1}) - O(\alpha)\label{eq:mixture_policy_improvement}
    \end{align}
\end{property}
As a result, as long as the offline algorithm (which we treat as a black box here) has some policy improvement guarantee, such as 
\citep{kumar2020conservative, liu2020provably, laroche2019safe}, then Eq.\eqref{eq:mixture_policy_improvement} implies a substantial policy improvement if $\alpha$ is chosen appropraitely. Besides, if we use Algorithms in Section \ref{sec:DERL_deterministic_policy} or \ref{sec:DERL_arbitrary_policy}, and collecting $\tilde N = \Theta(N / \alpha)$ samples, the guarantees in Theorem \ref{thm:deployment_complexity_given_reward} and Theorem \ref{thm:arbitary_policy_informal} can be extended correspondingly. 
Therefore, Alg \ref{alg:Mixture_Policy_Strategy} will return us a near-optimal policy after $K$ deployments while satisfying the safety constraint (c) in \eqref{eq:safety_constraint}.



\end{document}